\documentclass[preprint,10pt]{elsarticle}

\usepackage{epsfig}
\usepackage{palatino}
\RequirePackage{ifthen}
\usepackage{latexsym}
\RequirePackage{amsmath}
\RequirePackage{amsthm}
\RequirePackage{amssymb}
\RequirePackage{xspace}
\RequirePackage{graphics}
\RequirePackage{textcomp}
\usepackage{keyval}
\usepackage{xspace}
\usepackage{mathrsfs}
\usepackage{paralist}
\usepackage{url}
\usepackage{mathrsfs}
\usepackage{framed}
\usepackage{lipsum}
\usepackage{listings}

\def\titlename{Safe and Stabilizing Distributed Multi-Path Cellular Flows}

\usepackage{hyperref}
\hypersetup{
  pdftitle={\titlename},
  colorlinks=true,
  citecolor={blue},
  linkcolor = {blue},
  bookmarksopen=true,
  bookmarksnumbered=true
}

\newcommand{\sayan}[1]{}
\newcommand{\taylor}[1]{}

\newcommand{\num}[1]{\relax\ifmmode \mathbb #1\else $\mathbb #1$\fi}
\newcommand{\nnnum}[1]{\relax\ifmmode 
  {\mathbb #1}_{\geq 0} \else ${\mathbb #1}_{\geq 0}$
  \fi}
\newcommand{\npnum}[1]{\relax\ifmmode 
  {\mathbb #1}_{\leq 0} \else ${\mathbb #1}_{\leq 0}$
  \fi}
\newcommand{\pnum}[1]{\relax\ifmmode 
  {\mathbb #1}_{> 0} \else ${\mathbb #1}_{> 0}$
  \fi}
\newcommand{\nnum}[1]{\relax\ifmmode 
  {\mathbb #1}_{< 0} \else ${\mathbb #1}_{< 0}$
  \fi}
\newcommand{\plnum}[1]{\relax\ifmmode 
  {\mathbb #1}_{+} \else ${\mathbb #1}_{+}$
  \fi}
\newcommand{\nenum}[1]{\relax\ifmmode 
  {\mathbb #1}_{-} \else ${\mathbb #1}_{-}$
  \fi}

\newcommand{\reals}{{\num R}}                    %reals
\newcommand{\booleans}{{\num B}}                    %reals
\newcommand{\naturals}{{\num N}}                      %natural numbers
\newcommand{\extb}[1]{\relax\ifmmode {\sf ExtBeh}_{#1} \else ${\sf ExtBeh}_{#1}$\fi} 
\newcommand{\tdists}[1]{\relax\ifmmode {\sf Tdists}_{#1} \else ${\sf Tdists}_{#1}$\fi} 

\newcommand{\exec}[1]{\relax\ifmmode {\sf Execs}_{#1} \else ${\sf Exec}_{#1}$\fi} 
\newcommand{\execf}[1]{\relax\ifmmode {\sf Execs}^*_{#1} \else ${\sf Exec}^*_{#1}$\fi} 
\newcommand{\execi}[1]{\relax\ifmmode {\sf Execs}^\omega_{#1} \else ${\sf Exec}^\omega_{#1}$\fi} 

\newcommand{\ctrace}[1]{\relax\ifmmode {\sf Ctraces}_{#1} \else ${\sf Ctraces}_{#1}$\fi} 

\newcommand{\trace}[1]{\relax\ifmmode {\sf Traces}_{#1} \else ${\sf Traces}_{#1}$\fi} 
\newcommand{\tracef}[1]{\relax\ifmmode {\sf Traces}^*_{#1} \else ${\sf Traces}^*_{#1}$\fi} 
\newcommand{\tracei}[1]{\relax\ifmmode {\sf Traces}^\omega_{#1} \else ${\sf Traces}^\omega_{#1}$\fi} 

\newcommand{\frag}[1]{\relax\ifmmode {\sf Frags}_{#1} \else ${\sf Frags}_{#1}$\fi} 
\newcommand{\fragf}[1]{\relax\ifmmode {\sf Frags}^*_{#1} \else ${\sf Frags}^*_{#1}$\fi} 
\newcommand{\fragi}[1]{\relax\ifmmode {\sf Frags}^\omega_{#1} \else ${\sf Frags}^\omega_{#1}$\fi} 

\newcommand{\reach}[1]{\relax\ifmmode {\sf Reach}_{#1} \else ${\sf Reach}_{#1}$\fi}

\newcounter{theorem}
\newtheorem{rtheorem}{Theorem}
\newtheorem{lemma}[theorem]{Lemma}
\newtheorem{corollary}[theorem]{Corollary}

\newtheorem{inv}[theorem]{Invariant}
\newcounter{assump}
\newtheorem{assumption}[assump]{Assumption}
\def\A{{\cal A}} % HA
\def\E{{\cal E}} % HA
\def\I{{\cal I}} % environment sequence
\def\R{{\cal R}} % relation
\def\T{{\cal T}} % set of trajectories
\def\U{{\cal U}} % set of trajectories
\newcommand{\col}[1]{\relax\ifmmode \mathscr #1\else $\mathscr #1$\fi}

\newcommand{\SC}[2]{\relax\ifmmode {\tt Scount}(#1,#2) \else ${\tt Scount}(#1,#2)$\fi} 
\newcommand{\SCM}[2]{\relax\ifmmode {\tt Smin}(#1,#2) \else ${\tt Smin}(#1,#2)$\fi} 
\newcommand{\Aut}[1]{\relax\ifmmode {\tt Aut}(#1) \else ${\tt Aut}(#1)$\fi} 

\newcommand{\auto}[1]{{\operatorname{\mathsf{#1}}}}
\newcommand{\act}[1]{{\operatorname{\mathsf{#1}}}}

\newcommand{\deq}{\overset{\scriptscriptstyle\triangle}{=}}

\newcommand{\seclabel}[1]{\label{sec:#1}}
\newcommand{\lemlabel}[1]{\label{lemma:#1}}
\newcommand{\corlabel}[1]{\label{cor:#1}}
\newcommand{\secref}[1]{Section~\ref{sec:#1}}
\newcommand{\sslabel}[1]{\label{ss:#1}}
\newcommand{\ssref}[1]{Subsection~\ref{ss:#1}}

\newcommand{\secrefthree}[3]{Sections~\ref{sec:#1},~\ref{sec:#2},~and~\ref{sec:#3}}
\newcommand{\figlabel}[1]{\label{fig:#1}}
\newcommand{\figref}[1]{Figure~\ref{fig:#1}}
\newcommand{\figreftwo}[2]{Figures~\ref{fig:#1}~and~\ref{fig:#2}}

\newcommand{\lnlabel}[1]{\label{line:#1}}
\newcommand{\thmlabel}[1]{\label{thm:#1}}
\newcommand{\invlabel}[1]{\label{inv:#1}}

\newcommand{\lnsref}[2]{lines~\ref{line:#1}--\ref{line:#2}\xspace}
\newcommand{\lnref}[1]{line~\ref{line:#1}\xspace}
\newcommand{\lnreftwo}[2]{lines~\ref{line:#1}~and~~\ref{line:#2}\xspace}
\newcommand{\thmref}[1]{Theorem~\ref{thm:#1}\xspace}
\newcommand{\lmref}[1]{Lemma~\ref{lemma:#1}\xspace}
\newcommand{\lemref}[1]{Lemma~\ref{lemma:#1}\xspace}
\newcommand{\invref}[1]{Invariant~\ref{inv:#1}\xspace}

\newcommand{\assref}[1]{Assumption~\ref{ass:#1}\xspace}
\newcommand{\asslabel}[1]{\label{ass:#1}}
\newcommand{\assreftwo}[2]{Assumptions~\ref{ass:#1}~and~\ref{ass:#2}\xspace}

\newcommand{\eqlabel}[1]{\label{eq:#1}}

\newcommand{\remove}[1]{}
\newcommand{\salg}[1]{\relax\ifmmode {\mathcal F}_{#1}\else ${\mathcal F}_{#1}$\fi} 
\newcommand{\msp}[1]{\relax\ifmmode (#1, \salg{#1}) \else $(#1, \salg{#1})$\fi} 
\newcommand{\msprod}[2]{\relax\ifmmode ( #1 \times #2, \salg{#1} \otimes \salg{#2}) \else $(#1 \times #2, \salg{#1} \otimes \salg{#2})$\fi} 
\newcommand{\dist}[1]{\relax\ifmmode {\mathcal P}\msp{#1}
  \else ${\mathcal P}\msp{#1}$\fi} 
\newcommand{\subdist}[1]{\relax\ifmmode {\mathcal S}{\mathcal P}\msp{#1} 
  \else ${\mathcal S}{\mathcal P}\msp{#1}$\fi} 
\newcommand{\disc}[1]{\relax\ifmmode {\sf Disc}(#1)
  \else ${\sf Disc}(#1)$\fi} 

\newcommand{\Trajeq}{\relax\ifmmode {\mathcal R}_\T \else ${\mathcal R}_\T$\fi} 
\newcommand{\Acteq}{\relax\ifmmode {\mathcal R}_A \else ${\mathcal R}_A$\fi} 
\newcommand{\noop}{\relax\ifmmode \lambda \else $\lambda$\fi} 
\newcommand{\close}[1]{\relax\ifmmode \overline{#1} \else $\overline{#1}$\fi}

\newcommand{\pc}{{\operatorname{\mathsf {counter}}}}

\newcommand{\mybox}[3]{
  \framebox[#1][l]
  {
    \parbox{#2}
    {
      #3
    }
  }
}

\newcommand{\twosep}[4]{
  \parbox{\columnwidth}{\vspace{1pt} \vfill
    \parbox[t]{#1\columnwidth}{#3}%
   	\vrule width 0.2pt
    \parbox[t]{#2\columnwidth}{#4}%
  }}

\newcommand{\tup}[1]
           {
             \relax\ifmmode
             \langle #1 \rangle
             \else $\langle$ #1 $\rangle$ \fi
           }

\newcommand{\lit}[1]{ \relax\ifmmode
                \mathord{\mathcode`\-="702D\sf #1\mathcode`\-="2200}
                \else {\it #1} \fi }

\newcommand{\figuresize}{\scriptsize}

\lstdefinelanguage{ioa}{
  basicstyle=\figuresize,
  keywordstyle=\bf \figuresize,
  identifierstyle=\it \figuresize,
  emphstyle=\tt \figuresize,
  mathescape=true,
  tabsize=20,
  sensitive=false,
  columns=fullflexible,
  keepspaces=false,
  flexiblecolumns=true,
  basewidth=0.05em,
  escapeinside={(*@}{@*)},
  moredelim=[il][\rm]{//},
  moredelim=[is][\sf \figuresize]{!}{!},
  moredelim=[is][\bf \figuresize]{*}{*},
  keywords={automaton,and, 
  	 choose,const,continue, components,
  	 discrete, do,
  	 each,eff, external,else, elseif, evolve, end,
  	 fi,for, forward, from,
  	 hidden,
  	 in,input,internal,if,invariant, init,initially, imports,
     let,
     or, output, operators, od, of,
     pre,
     return,
     such,satisfies, stop, signature, simulation, 
     trajectories,trajdef, transitions, that,then, type, types, to, tasks,
     variables, vocabulary, 
     when,where, with,while},
  emph={set, seq, tuple, map, array, enumeration},   
   literate=
        {(}{{$($}}1
        {)}{{$)$}}1
        {\\in}{{$\in\ $}}1
        {\\preceq}{{$\preceq\ $}}1
        {\\subset}{{$\subset\ $}}1
        {\\subseteq}{{$\subseteq\ $}}1
        {\\supset}{{$\supset\ $}}1
        {\\supseteq}{{$\supseteq\ $}}1
        {\\forall}{{$\forall$}}1
        {\\le}{{$\le\ $}}1
        {\\ge}{{$\ge\ $}}1
        {\\gets}{{$\gets\ $}}1
        {\\cup}{{$\cup\ $}}1
        {\\cap}{{$\cap\ $}}1
        {\\langle}{{$\langle$}}1
        {\\rangle}{{$\rangle$}}1
        {\\exists}{{$\exists\ $}}1
        {\\bot}{{$\bot$}}1
        {\\rip}{{$\rip$}}1
        {\\emptyset}{{$\emptyset$}}1
        {\\notin}{{$\notin\ $}}1
        {\\not\\exists}{{$\not\exists\ $}}1
        {\\ne}{{$\ne\ $}}1
        {\\to}{{$\to\ $}}1
        {\\implies}{{$\implies\ $}}1
        {<}{{$<\ $}}1
        {>}{{$>\ $}}1
        {=}{{$=\ $}}1
        {~}{{$\neg\ $}}1
        {|}{{$\mid$}}1
        {'}{{$^\prime$}}1
        {\\A}{{$\forall\ $}}1
        {\\E}{{$\exists\ $}}1
        {\\nE}{{$\nexists\ $}}1
        {\\/}{{$\vee\,$}}1
        {\\vee}{{$\vee\,$}}1
        {/\\}{{$\wedge\,$}}1
        {\\wedge}{{$\wedge\,$}}1
        {=>}{{$\Rightarrow\ $}}1
        {->}{{$\rightarrow\ $}}1
        {<=}{{$\Leftarrow\ $}}1
        {<-}{{$\leftarrow\ $}}1
        {~=}{{$\neq\ $}}1
        {\\U}{{$\cup\ $}}1
        {\\I}{{$\cap\ $}}1
        {|-}{{$\vdash\ $}}1
        {-|}{{$\dashv\ $}}1
        {<<}{{$\ll\ $}}2
        {>>}{{$\gg\ $}}2
        {||}{{$\|$}}1
        {[}{{$[$}}1
        {]}{{$\,]$}}1
        {[[}{{$\langle$}}1
        {]]]}{{$]\rangle$}}1
        {]]}{{$\rangle$}}1
        {<=>}{{$\Leftrightarrow\ $}}2
        {<->}{{$\leftrightarrow\ $}}2
        {(+)}{{$\oplus\ $}}1
        {(-)}{{$\ominus\ $}}1
        {_i}{{$_{i}$}}1
        {_j}{{$_{j}$}}1
        {_{i,j}}{{$_{i,j}$}}3
        {_{j,i}}{{$_{j,i}$}}3
        {_0}{{$_0$}}1
        {_1}{{$_1$}}1
        {_2}{{$_2$}}1
        {_n}{{$_n$}}1
        {_p}{{$_p$}}1
        {_k}{{$_n$}}1
        {-}{{$\ms{-}$}}1
        {@}{{}}0
        {\\delta}{{$\delta$}}1
        {\\R}{{$\R$}}1
        {\\Rplus}{{$\Rplus$}}1
        {\\N}{{$\N$}}1
        {\\times}{{$\times\ $}}1
        {\\tau}{{$\tau$}}1
        {\\alpha}{{$\alpha$}}1
        {\\beta}{{$\beta$}}1
        {\\gamma}{{$\gamma$}}1
        {\\ell}{{$\ell\ $}}1
        {--}{{$-\ $}}1
        {\\TT}{{\hspace{1.5em}}}3        
      }

\lstdefinelanguage{ioaNums}[]{ioa}
{
  numbers=left,
  numberstyle=\tiny,
  stepnumber=2,
  numbersep=4pt
}

\lstdefinelanguage{ioaNumsRight}[]{ioa}
{
  numbers=right,
  numberstyle=\tiny,
  stepnumber=2,
  numbersep=4pt
}

\lstnewenvironment{IOA}%
  {\lstset{language=IOA}}
  {}

\lstnewenvironment{IOANums}%
  {
  \if@firstcolumn
    \lstset{language=IOA, numbers=left, firstnumber=auto}
  \else
    \lstset{language=IOA, numbers=right, firstnumber=auto}
  \fi
  }
  {}

\lstnewenvironment{IOANumsRight}%
  {
    \lstset{language=IOA, numbers=right, firstnumber=auto}
  }
  {}

\newcommand{\linefigioa}[9]{

}

\lstdefinelanguage{ioaLang}{%
  basicstyle=\ttfamily\small,
  keywordstyle=\rmfamily\bfseries\small,
  identifierstyle=\small,
  keywords={assumes,automaton,axioms,backward,bounds,by,case,choose,components,const,d,det,discrete,do,eff,else,elseif,ensuring,enumeration,evolve,fi,fire,follow,for,forward,from,hidden,if,in,%
    input,init,initially,internal,invariant,let, local,od,of,output,pre,schedule,signature,so,%
    simulation,states,variables, tasks, stop,tasks,that,then,to,trajdef,trajectory,trajectories,transitions,tuple,type,union,urgent,uses,when,where,while,yield},
  literate=
        {\\in}{{$\in$}}1
        {\\preceq}{{$\preceq$}}1
        {\\subset}{{$\subset$}}1
        {\\subseteq}{{$\subseteq$}}1
        {\\supset}{{$\supset$}}1
        {\\supseteq}{{$\supseteq$}}1
        {\\rho}{{$\rho$}}1
        {\\infty}{{$\infty$}}1
        {<}{{$<$}}1
        {>}{{$>$}}1
        {=}{{$=$}}1
        {~}{{$\neg$}}1 
        {|}{{$\mid$}}1
        {'}{{$^\prime$}}1
        {\\A}{{$\forall$}}1 {\\E}{{$\exists$}}1
        {\\/}{{$\vee$}}1 {/\\}{{$\wedge$}}1 
        {=>}{{$\Rightarrow$}}1 
        {->}{{$\rightarrow$}}1 
        {<=}{{$\leq$}}1 {>=}{{$\geq$}}1 {~=}{{$\neq$}}1
        {\\U}{{$\cup$}}1 {\\I}{{$\cap$}}1
        {|-}{{$\vdash$}}1 {-|}{{$\dashv$}}1
        {<<}{{$\ll$}}2 {>>}{{$\gg$}}2
        {||}{{$\|$}}1
        {<=>}{{$\Leftrightarrow$}}2 
        {<->}{{$\leftrightarrow$}}2
        {(+)}{{$\oplus$}}1
        {(-)}{{$\ominus$}}1
}

\lstdefinelanguage{bigIOALang}{%
  basicstyle=\ttfamily,
  keywordstyle=\rmfamily\bfseries,
  identifierstyle=,
  keywords={assumes,automaton,axioms,backward,by,case,choose,components,const,%
    d,det,discrete,do,eff,else,elseif,ensuring,enumeration,evolve,fi,for,forward,from,hidden,if,in%
    input,initially,internal,invariant,local,od,of,output,pre,schedule,signature,so,%
    tasks, simulation,states,stop,tasks,that,then,to,trajdef,trajectories,transitions,tuple,type,union,urgent,uses,when,where,yield},
  literate=
        {\\in}{{$\in$}}1
        {\\preceq}{{$\preceq$}}1
        {\\subset}{{$\subset$}}1
        {\\subseteq}{{$\subseteq$}}1
        {\\supset}{{$\supset$}}1
        {\\supseteq}{{$\supseteq$}}1
        {<}{{$<$}}1
        {>}{{$>$}}1
        {=}{{$=$}}1
        {~}{{$\neg$}}1 
        {|}{{$\mid$}}1
        {'}{{$^\prime$}}1
        {\\A}{{$\forall$}}1 {\\E}{{$\exists$}}1
        {\\/}{{$\vee$}}1 {/\\}{{$\wedge$}}1 
        {=>}{{$\Rightarrow$}}1 
        {->}{{$\rightarrow$}}1 
        {<=}{{$\leq$}}1 {>=}{{$\geq$}}1 {~=}{{$\neq$}}1
        {\\U}{{$\cup$}}1 {\\I}{{$\cap$}}1
        {|-}{{$\vdash$}}1 {-|}{{$\dashv$}}1
        {<<}{{$\ll$}}2 {>>}{{$\gg$}}2
        {||}{{$\|$}}1
        {<=>}{{$\Leftrightarrow$}}2 
        {<->}{{$\leftrightarrow$}}2
        {(+)}{{$\oplus$}}1
        {(-)}{{$\ominus$}}1
}

\lstnewenvironment{BigIOA}%
  {\lstset{language=bigIOALang,basicstyle=\ttfamily}
   \csname lst@SetFirstLabel\endcsname}
  {\csname lst@SaveFirstLabel\endcsname\vspace{-4pt}\noindent}

\lstnewenvironment{SmallIOA}%
  {\lstset{language=ioaLang,basicstyle=\ttfamily\scriptsize}
   \csname lst@SetFirstLabel\endcsname}
  {\csname lst@SaveFirstLabel\endcsname\noindent}

\newcommand{\true}{\relax\ifmmode \mathit{true} \else \emph{true} \/\fi}
\newcommand{\false}{\relax\ifmmode \mathit{false} \else \emph{false} \/\fi}

\newlength{\bracklen}

\newcommand{\tri}[3]{\ensuremath{\mathit{#1}^\mathit{#2}_\mathit{#3}}}

\newcommand{\sugLocalVars}[2]{\ifthenelse{\equal{}{#2}}%
                             {\tri{localVars}{#1}{desug}}%
                             {\tri{localVars}{#1}{#2,desug}}}
\newcommand{\sugVars}[2]{\ifthenelse{\equal{}{#2}}%
                        {\tri{vars}{#1}{desug}}%
                        {\tri{vars}{#1}{#2,desug}}}

\newenvironment{subSyntax}{\begin{array}{l}}{\end{array}}

\newcommand{\ms}[1]{\ifmmode%
\mathord{\mathcode`-="702D\it #1\mathcode`\-="2200}\else%
$\mathord{\mathcode`-="702D\it #1\mathcode`\-="2200}$\fi}

\newcommand{\tcon}[1]{{\tt #1}}

\def\A{{\cal A}} % TA
\def\T{{\cal T}} % set of trajectories

\newcommand{\vx}{{\bf x}}
\newcommand{\vy}{{\bf y}}

\newcommand{\arrow}[1]{\mathrel{\stackrel{#1}{\rightarrow}}}

\lstdefinelanguage{pvs}{
  basicstyle=\tt \figuresize,
  keywordstyle=\sc \figuresize,
  identifierstyle=\it \figuresize,
  emphstyle=\tt \figuresize,
  mathescape=true,
  tabsize=20,
  sensitive=false,
  columns=fullflexible,
  keepspaces=false,
  flexiblecolumns=true,
  basewidth=0.05em,
  moredelim=[il][\rm]{//},
  moredelim=[is][\sf \figuresize]{!}{!},
  moredelim=[is][\bf \figuresize]{*}{*},
  keywords={and, 
  	 begin,
  	 cases, const,
  	 do,
  	 external, else, exists, end, endcases, endif,
  	 fi,for, forall, from,
  	 hidden,
  	 in, if, importing,
     let, lambda, lemma,
     measure, 
     not,
     or, of,
     return, recursive,
     stop, 
     theory, that,then, type, types, type+, to, theorem,
     var,
     with,while},
  emph={nat, setof, sequence, eq, tuple, map, array, enumeration, bool, real, exp, nnreal, posreal},   
   literate=
        {(}{{$($}}1
        {)}{{$)$}}1
        {\\in}{{$\in\ $}}1
        {\\mapsto}{{$\rightarrow\ $}}1
        {\\preceq}{{$\preceq\ $}}1
        {\\subset}{{$\subset\ $}}1
        {\\subseteq}{{$\subseteq\ $}}1
        {\\supset}{{$\supset\ $}}1
        {\\supseteq}{{$\supseteq\ $}}1
        {\\forall}{{$\forall$}}1
        {\\le}{{$\le\ $}}1
        {\\ge}{{$\ge\ $}}1
        {\\gets}{{$\gets\ $}}1
        {\\cup}{{$\cup\ $}}1
        {\\cap}{{$\cap\ $}}1
        {\\langle}{{$\langle$}}1
        {\\rangle}{{$\rangle$}}1
        {\\exists}{{$\exists\ $}}1
        {\\bot}{{$\bot$}}1
        {\\rip}{{$\rip$}}1
        {\\emptyset}{{$\emptyset$}}1
        {\\notin}{{$\notin\ $}}1
        {\\not\\exists}{{$\not\exists\ $}}1
        {\\ne}{{$\ne\ $}}1
        {\\to}{{$\to\ $}}1
        {\\implies}{{$\implies\ $}}1
        {<}{{$<\ $}}1
        {>}{{$>\ $}}1
        {=}{{$=\ $}}1
        {~}{{$\neg\ $}}1
        {|}{{$\mid$}}1
        {'}{{$^\prime$}}1
        {\\A}{{$\forall\ $}}1
        {\\E}{{$\exists\ $}}1
        {\\/}{{$\vee\,$}}1
        {\\vee}{{$\vee\,$}}1
        {/\\}{{$\wedge\,$}}1
        {\\wedge}{{$\wedge\,$}}1
        {->}{{$\rightarrow\ $}}1
        {=>}{{$\Rightarrow\ $}}1
        {->}{{$\rightarrow\ $}}1
        {<=}{{$\Leftarrow\ $}}1
        {<-}{{$\leftarrow\ $}}1
        {~=}{{$\neq\ $}}1
        {\\U}{{$\cup\ $}}1
        {\\I}{{$\cap\ $}}1
        {|-}{{$\vdash\ $}}1
        {-|}{{$\dashv\ $}}1
        {<<}{{$\ll\ $}}2
        {>>}{{$\gg\ $}}2
        {||}{{$\|$}}1
        {[}{{$[$}}1
        {]}{{$\,]$}}1
        {[[}{{$\langle$}}1
        {]]]}{{$]\rangle$}}1
        {]]}{{$\rangle$}}1
        {<=>}{{$\Leftrightarrow\ $}}2
        {<->}{{$\leftrightarrow\ $}}2
        {(+)}{{$\oplus\ $}}1
        {(-)}{{$\ominus\ $}}1
        {_i}{{$_{i}$}}1
        {_j}{{$_{j}$}}1
        {_{i,j}}{{$_{i,j}$}}3
        {_{j,i}}{{$_{j,i}$}}3
        {_0}{{$_0$}}1
        {_1}{{$_1$}}1
        {_2}{{$_2$}}1
        {_n}{{$_n$}}1
        {_p}{{$_p$}}1
        {_k}{{$_n$}}1
        {-}{{$\ms{-}$}}1
        {@}{{}}0
        {\\delta}{{$\delta$}}1
        {\\R}{{$\R$}}1
        {\\Rplus}{{$\Rplus$}}1
        {\\N}{{$\N$}}1
        {\\times}{{$\times\ $}}1
        {\\tau}{{$\tau$}}1
        {\\alpha}{{$\alpha$}}1
        {\\beta}{{$\beta$}}1
        {\\gamma}{{$\gamma$}}1
        {\\ell}{{$\ell\ $}}1
        {--}{{$-\ $}}1
        {\\TT}{{\hspace{1.5em}}}3        
      }

\lstdefinelanguage{BigPVS}{
  basicstyle=\tt,
  keywordstyle=\sc,
  identifierstyle=\it,
  emphstyle=\tt ,
  mathescape=true,
  tabsize=20,
  sensitive=false,
  columns=fullflexible,
  keepspaces=false,
  flexiblecolumns=true,
  basewidth=0.05em,
  moredelim=[il][\rm]{//},
  moredelim=[is][\sf \figuresize]{!}{!},
  moredelim=[is][\bf \figuresize]{*}{*},
  keywords={and, 
  	 begin,
  	 cases, const,
  	 do, datatype,
  	 external, else, exists, end, endif, endcases,
  	 fi,for, forall, from,
  	 hidden,
  	 in, if, importing,
     let, lambda, lemma,
     measure,
     not,
     or, of,
     return, recursive,
     stop, 
     theory, that,then, type, types, type+, to, theorem,
     var,
     with,while},
  emph={nat, setof, sequence, eq, tuple, map, array, first, rest, add, enumeration, bool, real, posreal, nnreal},   
   literate=
        {(}{{$($}}1
        {)}{{$)$}}1
        {\\in}{{$\in\ $}}1
        {\\mapsto}{{$\rightarrow\ $}}1
        {\\preceq}{{$\preceq\ $}}1
        {\\subset}{{$\subset\ $}}1
        {\\subseteq}{{$\subseteq\ $}}1
        {\\supset}{{$\supset\ $}}1
        {\\supseteq}{{$\supseteq\ $}}1
        {\\forall}{{$\forall$}}1
        {\\le}{{$\le\ $}}1
        {\\ge}{{$\ge\ $}}1
        {\\gets}{{$\gets\ $}}1
        {\\cup}{{$\cup\ $}}1
        {\\cap}{{$\cap\ $}}1
        {\\langle}{{$\langle$}}1
        {\\rangle}{{$\rangle$}}1
        {\\exists}{{$\exists\ $}}1
        {\\bot}{{$\bot$}}1
        {\\rip}{{$\rip$}}1
        {\\emptyset}{{$\emptyset$}}1
        {\\notin}{{$\notin\ $}}1
        {\\not\\exists}{{$\not\exists\ $}}1
        {\\ne}{{$\ne\ $}}1
        {\\to}{{$\to\ $}}1
        {\\implies}{{$\implies\ $}}1
        {<}{{$<\ $}}1
        {>}{{$>\ $}}1
        {=}{{$=\ $}}1
        {~}{{$\neg\ $}}1
        {|}{{$\mid$}}1
        {'}{{$^\prime$}}1
        {\\A}{{$\forall\ $}}1
        {\\E}{{$\exists\ $}}1
        {\\/}{{$\vee\,$}}1
        {\\vee}{{$\vee\,$}}1
        {/\\}{{$\wedge\,$}}1
        {\\wedge}{{$\wedge\,$}}1
        {->}{{$\rightarrow\ $}}1
        {=>}{{$\Rightarrow\ $}}1
        {->}{{$\rightarrow\ $}}1
        {<=}{{$\Leftarrow\ $}}1
        {<-}{{$\leftarrow\ $}}1
        {~=}{{$\neq\ $}}1
        {\\U}{{$\cup\ $}}1
        {\\I}{{$\cap\ $}}1
        {|-}{{$\vdash\ $}}1
        {-|}{{$\dashv\ $}}1
        {<<}{{$\ll\ $}}2
        {>>}{{$\gg\ $}}2
        {||}{{$\|$}}1
        {[}{{$[$}}1
        {]}{{$\,]$}}1
        {[[}{{$\langle$}}1
        {]]]}{{$]\rangle$}}1
        {]]}{{$\rangle$}}1
        {<=>}{{$\Leftrightarrow\ $}}2
        {<->}{{$\leftrightarrow\ $}}2
        {(+)}{{$\oplus\ $}}1
        {(-)}{{$\ominus\ $}}1
        {_i}{{$_{i}$}}1
        {_j}{{$_{j}$}}1
        {_{i,j}}{{$_{i,j}$}}3
        {_{j,i}}{{$_{j,i}$}}3
        {_0}{{$_0$}}1
        {_1}{{$_1$}}1
        {_2}{{$_2$}}1
        {_n}{{$_n$}}1
        {_p}{{$_p$}}1
        {_k}{{$_n$}}1
        {-}{{$\ms{-}$}}1
        {@}{{}}0
        {\\delta}{{$\delta$}}1
        {\\R}{{$\R$}}1
        {\\Rplus}{{$\Rplus$}}1
        {\\N}{{$\N$}}1
        {\\times}{{$\times\ $}}1
        {\\tau}{{$\tau$}}1
        {\\alpha}{{$\alpha$}}1
        {\\beta}{{$\beta$}}1
        {\\gamma}{{$\gamma$}}1
        {\\ell}{{$\ell\ $}}1
        {--}{{$-\ $}}1
        {\\TT}{{\hspace{1.5em}}}3        
      }

\lstdefinelanguage{pvsNums}[]{pvs}
{
  numbers=left,
  numberstyle=\tiny,
  stepnumber=2,
  numbersep=4pt
}

\lstdefinelanguage{pvsNumsRight}[]{pvs}
{
  numbers=right,
  numberstyle=\tiny,
  stepnumber=2,
  numbersep=4pt
}

\lstnewenvironment{BigPVS}%
  {\lstset{language=BigPVS}}
  {}

\lstnewenvironment{PVSNums}%
  {
  \if@firstcolumn
    \lstset{language=pvs, numbers=left, firstnumber=auto}
  \else
    \lstset{language=pvs, numbers=right, firstnumber=auto}
  \fi
  }
  {}

\lstnewenvironment{PVSNumsRight}%
  {
    \lstset{language=pvs, numbers=right, firstnumber=auto}
  }
  {}

\newcommand{\linefigpvs}[9]{

}

\lstdefinelanguage{pvsproof}{
  basicstyle=\tt \figuresize,
  mathescape=true,
  tabsize=4,
  sensitive=false,
  columns=fullflexible,
  keepspaces=false,
  flexiblecolumns=true,
  basewidth=0.05em,
}

\newcommand{\abs}[1]{\left\lvert#1\right\rvert}
\newcommand{\pair}[1]{\left\langle#1\right\rangle}

\newcommand{\norm}[1]{\left\lvert\left\lvert#1\right\rvert\right\rvert}
\newcommand{\argmin}[2]{\underset{#2}{\operatorname{argmin}} #1}

\newcommand{\minel}[2]{\underset{#2}{\operatorname{min}} #1}

\def\colorc{\mathit{[c]}}
\def\colord{\mathit{[d]}}

\def\NEPrev{{\mathit{NEPrev}}}
\def\NEPrevi{\NEPrev_i}
\def\NEPrevs{\NEPrev_s}
\def\Members{{\mathit{Entities}}}
\def\Membersi{\Members_i}
\def\Membersj{\Members_j}
\def\signal{{\mathit{signal}}}
\def\signali{\signal_i}
\def\signalj{\signal_j}
\def\signals{\signal_s}
\def\token{{\mathit{token}}}
\def\tokeni{\token_i}
\def\next{{\mathit{next}}}
\def\nexti{\next_i}
\def\nextic{\nexti\colorc}
\def\nextj{\next_j}
\def\nextjc{\nextj\colorc}
\def\nextjd{\nextj\colord}
\def\dist{{\mathit{dist}}}
\def\disti{\dist_i}
\def\distj{\dist_j}
\def\distic{\disti\colorc}
\def\distjc{\distj\colorc}
\def\failed{{\mathit{failed}}} % failed
\def\failedi{\failed_i}

\def\Nbrs{{\mathit{Nbrs}}} % neighbors
\def\Nbrsi{\Nbrs_i}
\def\Nbrsj{\Nbrs_j}
\def\Nbrss{\Nbrs_s}
\def\ID{\mathit{ID}}
\def\SID{\mathit{ID}_S}
\def\TID{\mathit{ID}_T}

\def\tid{\mathit{tid}}
\def\tidc{\tid_c}
\def\sid{\mathit{sid}}
\def\sidc{\sid_c}

\def\NF{\mathit{NF}}
\def\TC{\mathit{TC}}

\def\GR{\mathit{G_R}}
\def\ER{\mathit{E_R}} 
\def\VR{\mathit{V_R}} 
\def\CSC{\mathit{CSC}} % color shared cells
\def\SC{\mathit{SC}} % shared colors (of the color-shared cells)

\def\GE{\mathit{G_E}} % entity graph
\def\EE{\mathit{E_E}}
\def\VE{\mathit{V_E}}

\def\etype{\mathit{color}} % entity type/color
\newcommand{\etypeCommand}[1]{\etype(#1)} % entity type/color
\def\etypep{\etypeCommand{p}}
\def\etypeq{\etypeCommand{q}}
\def\etypei{\etype_i}
\def\etypej{\etype_j}

\def\Route{\mathit{Route}}
\def\Lock{\mathit{Lock}}
\def\Signal{\mathit{Signal}}
\def\Move{\mathit{Move}}

\def\rhoc{\rho_c}

\def\path{\mathit{path}}
\def\pathi{\path_i}
\def\pathic{\pathi\colorc}
\def\pathid{\pathi\colord}
\def\pathj{\path_j}
\def\pathjc{\pathj\colorc}
\def\pint{\mathit{pint}}
\def\pinti{\pint_i}

\def\pintic{\pinti\colorc}

\def\lock{\mathit{lock}}
\def\locki{\lock_i}
\def\lockj{\lock_j}
\def\lockic{\locki\colorc}
\def\lockid{\locki\colord}
\def\lockjc{\lockj\colorc}
\def\nlock{\mathit{nlock}}

\def\lockcolors{\mathit{lcs}}
\def\lockcolorsi{\lockcolors_i}

\def\Safe{\mathit{Safe}}
\def\Safei{\Safe_i}

\def\Cell{{\auto{Cell}}}
\def\Celli{\Cell_i}
\def\Cellj{\Cell_j}

\def\System{{\auto{System}}}

\def\FaceRegion{\mathit{FR}}
\def\SideNormal{n}
\def\SideNormalij{\SideNormal(i,j)}

\def\MoveVector{u}
\def\MoveVectorij{\MoveVector(i,j)}
\def\Side{\mathit{Side}}
\def\Sideij{\Side(i,j)}

\newcommand{\bdry}[1]{\partial #1}

\def\EntityIdx{{\mathit{I}}}

\def\partition{{\mathit{P}}}
\def\partitioni{\partition_i}
\def\partitionj{\partition_j}

\newcommand{\EntityDisc}[2]{\mathit{B(#1, #2)}}

\def\EntityDiscpl{\EntityDisc{p}{\l}}
\def\EntityDiscpd{\EntityDisc{p}{\d}}

\newcommand{\ns}[1]{\mathit{ns}(#1)}
\def\nsi{\ns{i}}

\def\transferRegion{\mathit{TR}}
\def\innerTransferRegion{\mathit{ITR}}
\def\transferRegioni{\transferRegion_i}
\def\innerTransferRegioni{\innerTransferRegion_i}

\def\innerTransferRegionj{\innerTransferRegion_j}
\def\transferRegionis{\transferRegioni(s)}
\def\innerTransferRegionis{\innerTransferRegion_i(s)}

\def\safetyRegion{\mathit{SR}}
\def\safetyRegioni{\safetyRegion_i}

\def\safetyRegionis{\safetyRegioni(s)}

\def\bdryi{\bdry{\partitioni}}
\def\bdryj{\bdry{\partitionj}}

\def\Env{\mathit{Env}}

\def\Vertices{\mathit{V}}
\def\Verticesi{\Vertices_i}

\def\ResetEntity{\mathit{ResetEntity}}

\def\l{{\mathit{l}}}
\def\v{{\mathit{v}}}
\def\d{{\mathit{d}}}
\def\rs{{\mathit{r_s}}}

\def\nextc{\next\colorc}
\def\pintc{\pint\colorc}
\def\pathc{\path\colorc}
\def\lockcolorsc{\lockcolors\colorc}

\def\nextic{\nexti\colorc}
\def\nextjc{\nextj\colorc}
\def\distc{\dist\colorc}
\def\distic{\disti\colorc}
\def\pathic{\pathi\colorc}
\def\lockic{\locki\colorc}

\def\lockcolorsic{\lockcolorsi\colorc}

\def\pc{\overline{p}}
\def\qc{\overline{q}}

\journal{Theoretical Computer Science A}

\begin{document}

\begin{frontmatter}



\title{\titlename}


\author{Taylor~T.~Johnson\corref{cor1}}
 \ead{taylor.johnson@acm.org}
  
\author{Sayan~Mitra}
 \ead{mitras@illinois.edu}
 
 \cortext[cor1]{Corresponding author}
  
 \address{Coordinated Science Laboratory, University of Illinois at Urbana-Champaign,Urbana, IL 61801, USA}

%\address[label1]{}

\begin{abstract}
We study the problem of distributed traffic control in the partitioned plane, where the movement of all entities (robots, vehicles, etc.) within each partition (cell) is coupled.  Establishing liveness in such systems is challenging, but such analysis will be necessary to apply such distributed traffic control algorithms in applications like coordinating robot swarms and the intelligent highway system.  We present a formal model of a distributed traffic control protocol that guarantees minimum separation between entities, even as some cells fail.  Once new failures cease occurring, in the case of a single target, the protocol is guaranteed to self-stabilize and the entities with feasible paths to the target cell make progress towards it.  For multiple targets, failures may cause deadlocks in the system, so we identify a class of non-deadlocking failures where all entities are able to make progress to their respective targets.  The algorithm relies on two general principles: temporary blocking for maintenance of safety and local geographical routing for guaranteeing progress.  Our assertional proofs may serve as a template for the analysis of other distributed traffic control protocols.  We present simulation results that provide estimates of throughput as a function of entity velocity, safety separation, single-target path complexity, failure-recovery rates, and multi-target path complexity.
\end{abstract}

\begin{keyword}
distributed systems \sep swarm robotics \sep formal methods
\end{keyword}

\end{frontmatter}

% \linenumbers


\section{Introduction}
\seclabel{intro}
%\taylor{Cell transmission model from berkeley for modeling traffic: \url{http://en.wikipedia.org/wiki/Cell_Transmission_Model}}
%\taylor{Show this gridlock scenario: \url{http://en.wikipedia.org/wiki/Gridlock}}

%
Highway and air traffic flows are nonlinear switched dynamical systems that give rise to complex phenomena such as abrupt phase transitions from fast to sluggish flow~\cite{helbing1998,kerner1998,daganzo1999}.
Our ability to monitor, predict, and avoid such phenomena can have a significant impact on the reliability and capacity of physical traffic networks. 
Traditional traffic protocols, such as those implemented for air traffic control are {\em centralized\/}~\cite{nolan1994book}---a {\em coordinator\/} periodically collects information from the vehicles, decides and disseminates waypoints, and subsequently the vehicles try to blindly follow a path to the waypoint.
Wireless vehicular networks~\cite{borgonovo2003mac,karpiriski2006,manvi2009,azimi2011} and autonomous vehicles~\cite{thrun2007,urmson2008} present new opportunities for {\em distributed\/} traffic monitoring~\cite{yang2004mobi,misener2005,hoh2008mobisys} and control~\cite{girard2001cdc,mamei2003isads,abbott2004tr,kelly2007,kowshik2008hscc,dresner2008jair}.
While these protocols may still rely on some centralized coordination, they should scale and be less vulnerable to failures compared to their centralized counterparts.
In this paper, we propose a fault-tolerant distributed traffic control protocol, formally model it, and formally prove its correctness.

A \emph{traffic control protocol} is a set of rules that determines the routing and movement of certain physical \emph{entities}, such as vehicles, robots, or packages, over an underlying {\em graph\/}, such as a road network, air-traffic network, or warehouse conveyor system.
Any traffic control protocol should guarantee:
\begin{inparaenum}[(a)]
\item ({\em safety\/}) that the entities always maintain some minimum physical separation, and
\item ({\em progress\/}) that the entities eventually arrive at a given a destination (or target) vertex.
\end{inparaenum}
In a distributed traffic control protocol, each entity determines its own next-waypoint, or each vertex in the underlying graph determines the next-waypoints for the entities in an appropriately defined neighborhood.

In this paper, we study the problem of distributed traffic control in a partitioned plane where the motions of entities within a partition are {\em coupled\/}.  
The problem can be described as follows (refer to~\figreftwo{exampleSystemTriangular}{exampleSystemOneLane}).
The \emph{environment}---the geographical space of interest---is partitioned into regions or \emph{cells}.
Each entity is assigned a certain type or {\em color\/}.
For each color, there is one \emph{source cell} and one \emph{target cell} of the same color.
The source cells produce entities of some color, and the target cells only consume entities of a particular color, so the goal is to move entities of color $c$ to the target of color $c$.
The motion of all entities within a cell are coupled, in the sense that they all either move identically, or they all remain stationary (we discuss the motivation for this below).
If some entities within some cell $i$ touch the boundary of a neighboring cell $j$, those entities are transferred to $j$.
Thus, the role of the distributed traffic control protocol is to control the motion of the cells so that the entities 
\begin{inparaenum}[(a)]
\item always have the required safe separation, and 
\item reach their respective targets, when feasible.
\end{inparaenum}

The coupling mentioned above that requires entities within a cell to move identically may appear strong at first sight.
After all, under low traffic conditions, individual drivers control the movement of their cars within a particular region of the highway, somewhat independently of the other drivers in that region.
However, on highways under high-traffic, high-velocity conditions, it is known that coupling may emerge spontaneously, causing the vehicles to form a fixed lattice structure and move with near-zero relative speed~\cite{helbing1998,weiss1999}.
In other scenarios, coupling arises because passive entities are moved around by active cells.
For example, this occurs with packages being routed on a grid of multi-directional conveyors~\cite{omniwheel2008,an2011iccps}, and molecules moving on a medium according to some controlled chemical gradient.
Finally, even where the entities are active and cells are not, the entities can cooperate to emulate a virtual active cell expressly for the purposes of distributed coordination.  
This idea has been explored for mobile robot coordination in~\cite{gilbert2009taas} using a cooperation strategy called \emph{virtual stationary automata}~\cite{dolev2005podc,nolte2007icdcs}.
%
%In the steady-state, highway vehicles have emergent behavior that could be modeled in a manner similar to the system presented here, as there are various regions on the highway where all vehicles are traveling at the same speed.
%

%
In this paper, we present a distributed traffic control protocol that guarantees \emph{safety at all times}, even when some cells fail permanently by crashing.
The protocol also guarantees \emph{eventual progress} of entities toward their targets, provided (a) that there exists a path through non-faulty cells to the entities' respective targets, and (b) failures have not introduced unrecoverable deadlocks.
Specifically, the protocol is \emph{self-stabilizing}~\cite{arora1993,dolev2000book}, in that once new failures stop occurring, the composed system automatically returns to a state from which progress can be made.
%
%Failure cessation means that the sets of non-faulty and faulty cells are fixed, not that the set of faulty cells is empty, in which case.
%Crash failures mean that faulty nodes stay failed.
%
The algorithm relies on the following four mechanisms.
\begin{enumerate}[(a)]
\item There is a \emph{routing} rule to maintain local routing tables to each target at each non-faulty cell.
This routing protocol is self-stabilizing and allows our protocol to tolerate crash failures of cells.
\item There is a \emph{mutual exclusion} and scheduling mechanism to ensure moving entities over distinctly colored overlapping paths do not introduce deadlocks.
The locking and scheduling mechanism ensures one-way traffic can make progress over shared routes (traffic intersections).
\item There is a \emph{signaling} rule between neighbors that guarantees safety while preventing deadlocks.
Roughly speaking, the signaling mechanism at some cell fairly chooses among its neighboring cells that contain entities, indicating if it is safe for one of these cells to apply a movement in the direction of the cell doing the signaling.
This permission-to-move policy turns out to be necessary, because movement of neighboring cells may otherwise result in a violation of safety in the signaling cell, if entity transfers occur.
\item The \emph{movement} policy causes all entities on a cell to either move with the same constant velocity in the direction of their destination, or remain stationary to ensure safety.
This policy abstracts more complex motion modeling.
\end{enumerate}

\begin{figure}[t]
	\centering
	\begin{minipage}[t]{0.4125\linewidth}
		\vspace{0pt} % alignment hack
		\centering
		\includegraphics[width=\columnwidth]{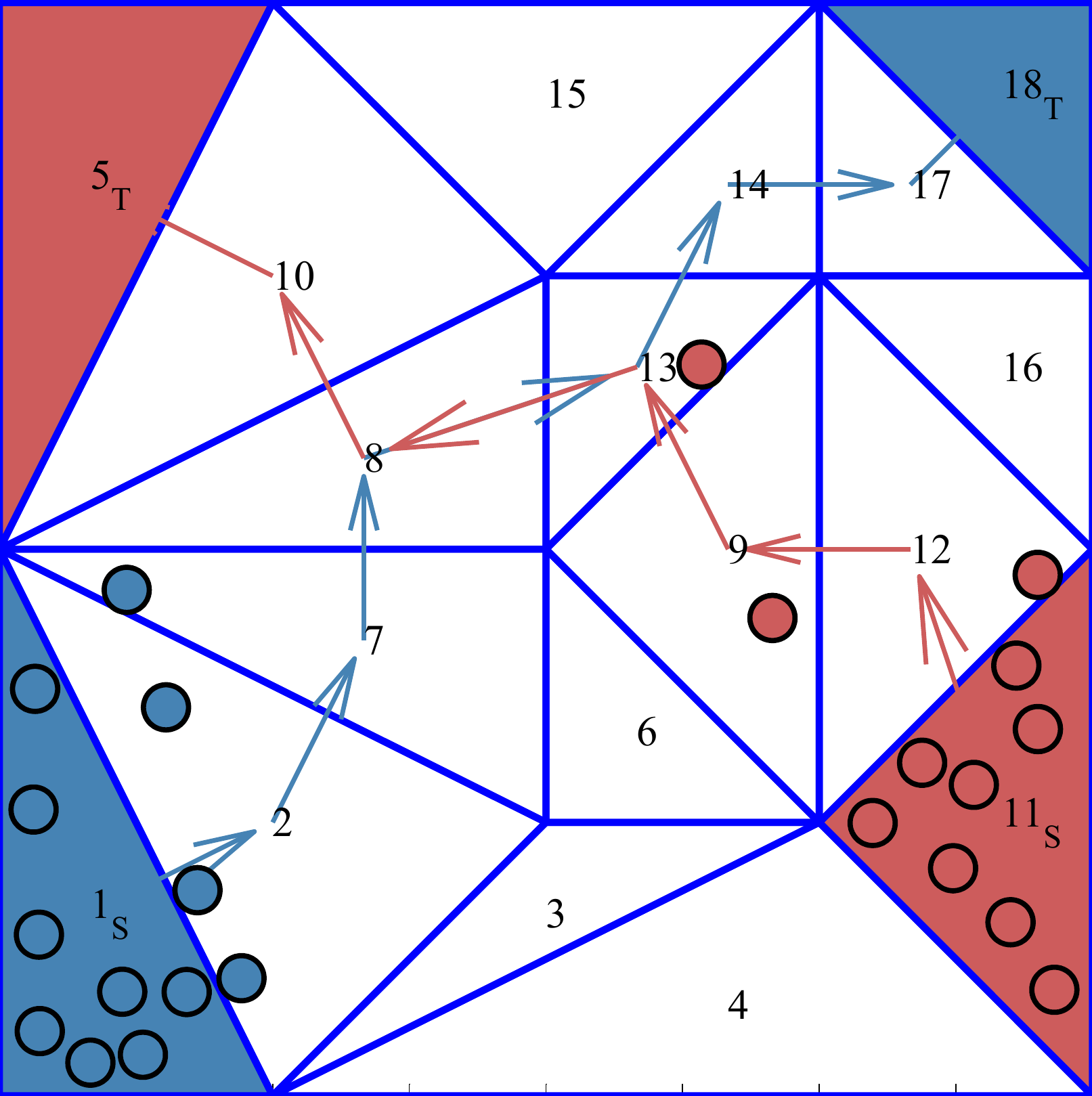}
		\caption{Source cells ($1$ and $11$) produce entities that flow toward the target cell ($18$ and $5$) of the appropriate color.  Source-to-target paths overlap at cells $8$ and $13$.  In this execution, the blue entity on cell $7$ is waiting for the red entities to leave the overlapping cells.}
		%and requires the cells with entities of different colors to coordinate to ensure progress.
		%---this inductively forces all cells on the sequence to the blue source (i.e., $7$, $2$, and $1$) to wait.
		\figlabel{exampleSystemTriangular}
	\end{minipage}
	\hspace{0.01\linewidth}
	\begin{minipage}[t]{0.25\linewidth}
		\vspace{0pt} % alignment hack
		\centering
		\includegraphics[width=\columnwidth]{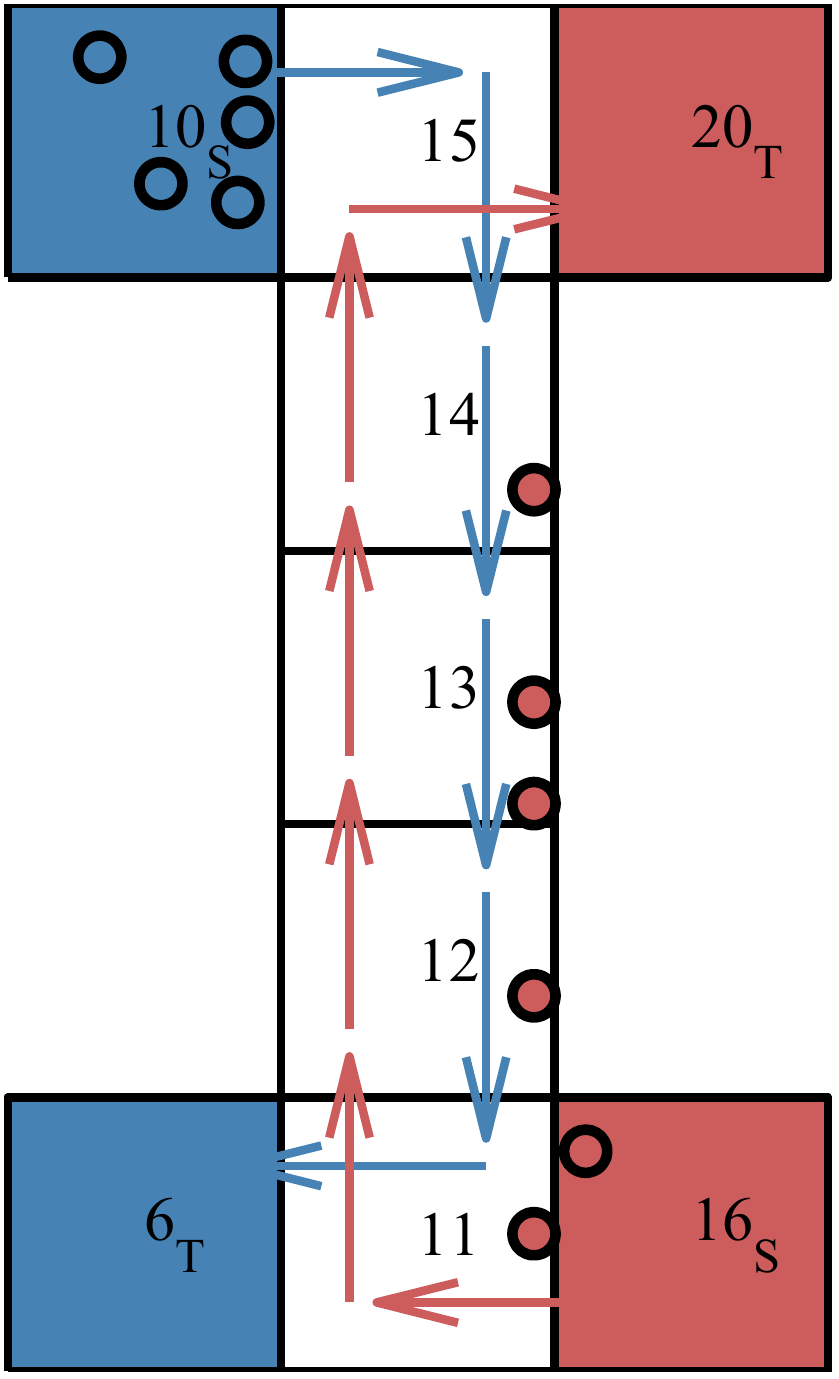}
		\caption{If cell $10$ moved its blue entities onto the shared one-lane ``bridge'' ($11$, $12$, $13$, $14$, $15$), then all entities would be deadlocked.}
		%Cells $10$ and $16$ must infinitely often block their entities from entering the shared bridge for fairness (see~\assref{fairness}).
		\figlabel{exampleSystemOneLane}
	\end{minipage}
\end{figure}

We establish these safety and progress properties through systematic assertional reasoning.
For safety properties, we establish inductive invariants and for stabilization we use global ranking functions.
To show that all entities reach their destinations (when feasible), we use a combination of ranking functions and fairness-based reasoning on infinite executions.
These proof techniques may serve as a template for the analysis of other distributed traffic control protocols.
Our analysis is generally independent of the size of the environment, number of cells, and number of entities.
Additionally, only neighboring cells communicate with one another and the communication topology is fixed (aside from failures).
For these reasons, this problem can serve as a case study in automatic parameterized verification of distributed cyber-physical systems~\cite{loos2011fm,platzer2011hscc,johnson2012iccps,johnson2012forte}.

%
%However, nondeterminism arises in three places, namely 
%\begin{inparaenum}[(i)]
%\item arbitrary safe placement of entities upon the source at initialization, 
%\item failures, and
%\item fairly choosing amongst neighbors that want to move entities to the same cell.
%\end{inparaenum}

We present simulation results that illustrate the influence (or the lack thereof) of several factors on throughput.
\begin{inparaenum}[(a)]
\item Throughput decreases exponentially with path length until saturation, as which point it decreases roughly linearly with path length.
\item Throughput decreases roughly linearly with required safety separation and cell velocity.
\item Throughput decreases roughly exponentially until it saturates as a function of path complexity measured in number of turns along a path.
\item Throughput decreases roughly exponentially with failure rate, and increases linearly with recovery rates, under a model where crash failures are not permanent and cells may recover from crashing.
\item Throughput decreases roughly exponentially until it saturates as a function of the percentage of overlapping cells between different colored targets.
\end{inparaenum}

\paragraph*{Contributions over Previous Work}
In previous work~\cite{johnson2010icdcs}, we analyzed a similar problem, but have significantly generalized our results in this paper.
\begin{enumerate}[(A)]
\item We consider general tessellations (including triangulations) that define the partitioning, while we considered uniform square partitions in~\cite{johnson2010icdcs}.
We also present results on partitioning schemes that cannot work for our formulation of the problem.
\item We allow entities of multiple colors, each flowing to a different target, while in~\cite{johnson2010icdcs}, we only allowed entities of one color, all of which flowed to the nearest target.
This generalization lets source-to-target paths of different colors overlap, creating intersections, and requires several changes to the algorithm, including adding a mutual exclusion and scheduling mechanisms used to control traffic intersections.
This generalization is significant because it makes the problem applicable to a much wider class of systems.
\item We extended our simulation results to allow for these generalizations, and characterized the cost on throughput due to the extra coordination required to allow multiple colors.
\end{enumerate}

\paragraph*{Paper Organization}
The rest of the paper is organized as follows.
First,~\secref{model} introduces the model of the physical system.
Next in~\secref{algorithm}, we present the distributed traffic control algorithm.
Then in~\secref{analysis}, we define and prove the safety and progress properties.
\ssref{safety} establishes safety.
Subsequently, we establish a progress property that shows entities eventually reach their targets in spite of failures (when possible).
First in~\ssref{routing}, it is shown that the routing protocol to find any target from any cell with a physical path through non-faulty cells to that target is self-stabilizing.
Then in~\ssref{lock}, we show how overlapping paths to different targets (traffic intersections) can be scheduled.
Finally, in~\ssref{movement}, it is shown that entities on any cell with a feasible physical path to their target eventually reach their target.
Simulation results and interpretation are presented in~\secref{sim}, followed by a brief discussion of related work and further extensions, and a conclusion in~\secrefthree{rw}{discussion}{conclusion}.

\section{Physical System Model}
\seclabel{model}
We describe the physical system in this section.
For a set $K$, we define $K_{\bot} \deq K \cup \{\bot\}$ and $K_{\infty} \deq K \cup \{\infty\}$.
For $N \in \naturals$, let $[N] \deq \{1,\ldots, N\}$.
The $\norm{\cdot}$ brackets are used for the Euclidean norm of a vector.

\paragraph*{Partitioning}
The system consists of $N$ convex polygonal \emph{cells} partitioning a polygonal environment.
Let $\ID \deq [N]$ be the set of unique identifiers for all cells in the system.%, and each cell has a unique identifier $i \in \ID$.
The planar environment $\Env$ is some given simply connected polygon.
A partition $\partition$ of $\Env$ is a set of closed, convex polygonal cells $\{\partitioni\}_{i \in \ID}$ such that:
\begin{enumerate}[(a)]
\item %$\forall i, j \in \ID . i \neq j \Rightarrow \interior{\partitioni} \cap \interior{\partitionj} = \emptyset$, that is, 
the interiors of the cells are pairwise disjoint,
\item %$\Env = \bigcup_{i \in \ID} \partitioni$, that is, 
the union of the cells is the original polygonal environment, and
\item %$\forall i, j \in \ID . i \neq j$, if $I = \partitioni \cap \partitionj \neq \emptyset$, then $I$ is either 
%\begin{inparaenum}[(i)]
%\item a common side to both $\partitioni$ and $\partitionj$ or
%\item a common vertex of both.
%\end{inparaenum}
cells only touch one another at a point or along an entire side.
\end{enumerate}
The first two conditions are the standard definition of a partition, while the third restricts any cell from being adjacent along one of its sides to more than one other cell.
%
%We note that the definition of a triangulation satisfies these requirements.
%
Thus, cell $i$ occupies a convex polygon $\partitioni$ in the Euclidean plane.
The boundary of cell $i$ is denoted by $\bdryi$.
%$\bdryi \deq \partitioni \setminus \interior{\partitioni}$.
%
We denote the vertices (extreme points) of $\partitioni$ as $\Verticesi$.
We denote the number of sides of $\partitioni$ as $\nsi$.
Let $\Sideij \deq \bdryi \cap \bdryj$ be the common side of adjacent cells $i$ and $j$---we will refer to $\Sideij$ as both an index and a line segment (set of points).

\paragraph*{Communications}
Cell $i$ is said to be a \emph{neighbor} of cell $j$ if the boundaries of the cells share a common side.
The set of identifiers of all neighbors of cell $i$ is denoted by $\Nbrsi$.
This definition of neighbors can naturally be represented as a graph, so let $\Delta$ be the worst-case diameter of such a neighbor communication graph\footnote{The diameter of this graph is not static, it may change due to failures, but the worst case is always a path graph, so $\Delta = N$.}.
For each cell $i \in \ID$ and each neighboring cell $j \in \Nbrsi$, let the \emph{side normal vector} from $i$ to $j$, denoted $\SideNormalij$, be the unit vector orthogonal to $\Sideij$ and pointing into cell $j$ from the common side $\Sideij$.

Each cell is controlled by software that implements the distributed traffic control algorithm described in the next section.
We consider synchronous protocols that operate in rounds.
At each round, each cell exchanges messages bearing state information with its neighbors.
Then, each cell updates its software state and decides the (possibly zero) velocity with which to move any entities on it.
Until the beginning of the next round, the cells continue to operate according to this velocity, which may lead to entity transfers.

\paragraph*{Entities}
Each cell may contain a number of {\em entities\/}.
Each entity occupies a circular area and represents a physical object (or overapproximation of) such as an aircraft, car, robot, or package.
Every entity that may ever be in the system has a unique identifier drawn from an index set $\EntityIdx$.
This assumption is for presentation only, and the algorithm does not rely on knowing entity identifiers.
For an entity $p \in \EntityIdx$, we denote the coordinates of its center by $\pc \deq (p_x, p_y) \in \reals^2$.

The open circular area (disc) centered at $\pc$ of radius $r$ representing entity $p$ is denoted $\EntityDiscpl$.
The radius of an entity is $\l$, and $\rs$ is the minimum required inter-entity safety gap.
We define the total safety spacing radius as $\d \deq \rs + \l$.
For simplicity of presentation, we work with uniform entity radii $\l$ and safety gaps $\rs$.
If they differ, we would take $\l$ and $\rs$ to be the maximums over all entities.
We instantiate $\EntityDiscpl$, which represents the physical space occupied by entity $p$, and we also instantiate $\EntityDiscpd$, which is entity $p$'s total safety area.

\paragraph*{Entity Colors, Source Cells, and Target Cells}
There are $\abs{C}$ types (or \emph{colors}) of entities, where $C$ is some finite, ordered set.
The color of some entity $p \in \EntityIdx$ is denoted as $\etypep$.
For each $c \in C$, there is a \emph{source cell} $\sidc$ and a \emph{target cell} $\tidc$.
All other cells are \emph{ordinary cells}.
For simplicity of presentation, we assume there is a unique source and target, but the algorithms and the results generalize for when $\sidc$ and $\tidc$ are sets.

Entity $p$'s color $\etypep$ designates the target cell entity $p$ should eventually reach.
The source $\sidc$ produces entities of color $c$ and the target $\tidc$ consumes entities of color $c$.
The sets of target and source identifiers are denoted $\TID \subseteq \ID$ and $\SID \subseteq \ID$, respectively.
%
%We assume each cell $i$ knows whether it is in the sets $\SID$ or $\TID$, but does not know the source or target status of any other cell (that is, cell $i$ does not know the set $\SID$ or $\TID$, only whether it is in the set or not).

\paragraph*{Entity Movement}
All the entities within a cell move identically---either they remain stationary or they move with some constant velocity $0 < \v < \l$ in the direction of one of the sides of the cell.
Thus $\v$ is the maximum cell velocity, or the greatest distance traveled by any entity over one synchronous round.
We require $\v > 0$ to ensure progress.
We require that $\v < \l$ to ensure entities do not collide when transfers occur.
%cells do not violate the safety requirement from one round to the next when entity transfers occur.
%
Cell velocity may differ in each cell so long as each is upper bounded by $\v$.
This movement is determined by the algorithm controlling each cell.
When a moving entity touches a side of a cell, it is instantaneously transferred to the neighboring cell beyond that side, so that the entity is entirely contained in the new cell.

\paragraph*{Safety and Transfer Regions}
\taylor{check all geometry with $\l$ and $\d$}
The safety region on side $s$ of a cell is the area within the cell where (the centers of) new entities entering the cell from side $s$ can be placed.
For a side $s$ of some cell $i$, the \emph{safety region on side $s$} $\safetyRegionis$ is the area on $\partitioni$ at most $3\d$ distance measured orthogonally from side $s$.
Analogously, the transfer region on side $s$ of a cell is the area within a cell where (the centers of) entities reside when those entities will be transfered to the neighboring cell on that side.
The \emph{transfer region on side $s$} $\transferRegionis$ is the region in the partition $\partitioni$ at most $\l$ distance measured orthogonally from side $s$.
%
%That is, $\transferRegionis$ and $\safetyRegionis$ are respectively the set of points at most $\l$ or $3\d$ distance from side $s$ of $\partitioni$.
%
For a cell $i$, the \emph{transfer region} $\transferRegioni$ and \emph{safety region} $\safetyRegioni$ are respectively the unions of $\transferRegionis$ and $\safetyRegionis$ for each side $s$ of $\partitioni$.
We refer to the \emph{inner side(s)} of $\transferRegioni$, $\transferRegionis$, $\safetyRegioni$, or $\safetyRegionis$, as the side(s) touching the inside of the annulus, and denote them by $\innerTransferRegioni$, $\innerTransferRegionis$, etc.

For example, in~\figref{transferSafety}, the transfer region for the square cell $3$ is the square annulus between the smaller cyan square and the larger blue square (the boundary $\bdry{\partition_3}$ of cell $3$).
Similarly, for the triangular cell $1$ in~\figref{transferSafety}, the transfer region is the triangular annulus between the smaller cyan triangle and the larger blue triangle.
Thus, the distance measured orthogonally between the sides of the cyan polygons representing the boundary of the transfer region, and the sides of the blue polygons is always $\l$.
In~\figref{transferSafety}, for the square cell $3$, the safety region is the square annulus between the smaller red square and the larger blue square.
%
%The choice of $3\d$ is needed by observing the red circular entity situated roughly on the vertical line $x = 100$ in~\figref{transferSafety}.

%Next we illustrate a computational method to determine the safety and transfer regions.
%%
%Recall that the incenter of a polygon is the center of a polygon's inscribed circle (if one exists)\footnote{Any triangle and any regular polygon have inscribed circles, but other polygons with four or more sides usually do not.}.
%%
%The incenter $\incenteri$ of cell $i$ is located at the intersection of the polygon's angle bisectors.
%%
%The incenter may or may not coincide with the centroid of cell $i$, which is the average of cell $i$'s vertices.
%
%
%We use a homothetic transformation defined as $H(x) = C - \lambda (C - x)$ for $x, C \in \mathbb{R}^2$, and $\lambda \neq 0$.
%%
%Intuitively, $C$ is the center of the transformation and $\lambda$ is the scaling factor.
%%
%If $\abs{\lambda} < 1$, such a transformation is known as a dilation.
%%
%Set $C$ to be the incenter of cell $i$.
%%
%Given the radius $R_i$ of the inscribed circle for cell $i$, we define the transfer region scaling factor $\transferScale = R_i - L$ and the safety region scaling factor $\safetyScale = R_i - 3d$.
%%
%The vertices of the transfer and safety region can then be found by using the appropriate scaling factor and computing $H(v)$ for each vertex $v \in \Verticesi$.

\begin{figure}[t]
	\centering%
	\begin{minipage}[t]{0.5\linewidth}
		\vspace{0pt} % alignment hack
		\centering
		\includegraphics[width=\columnwidth]{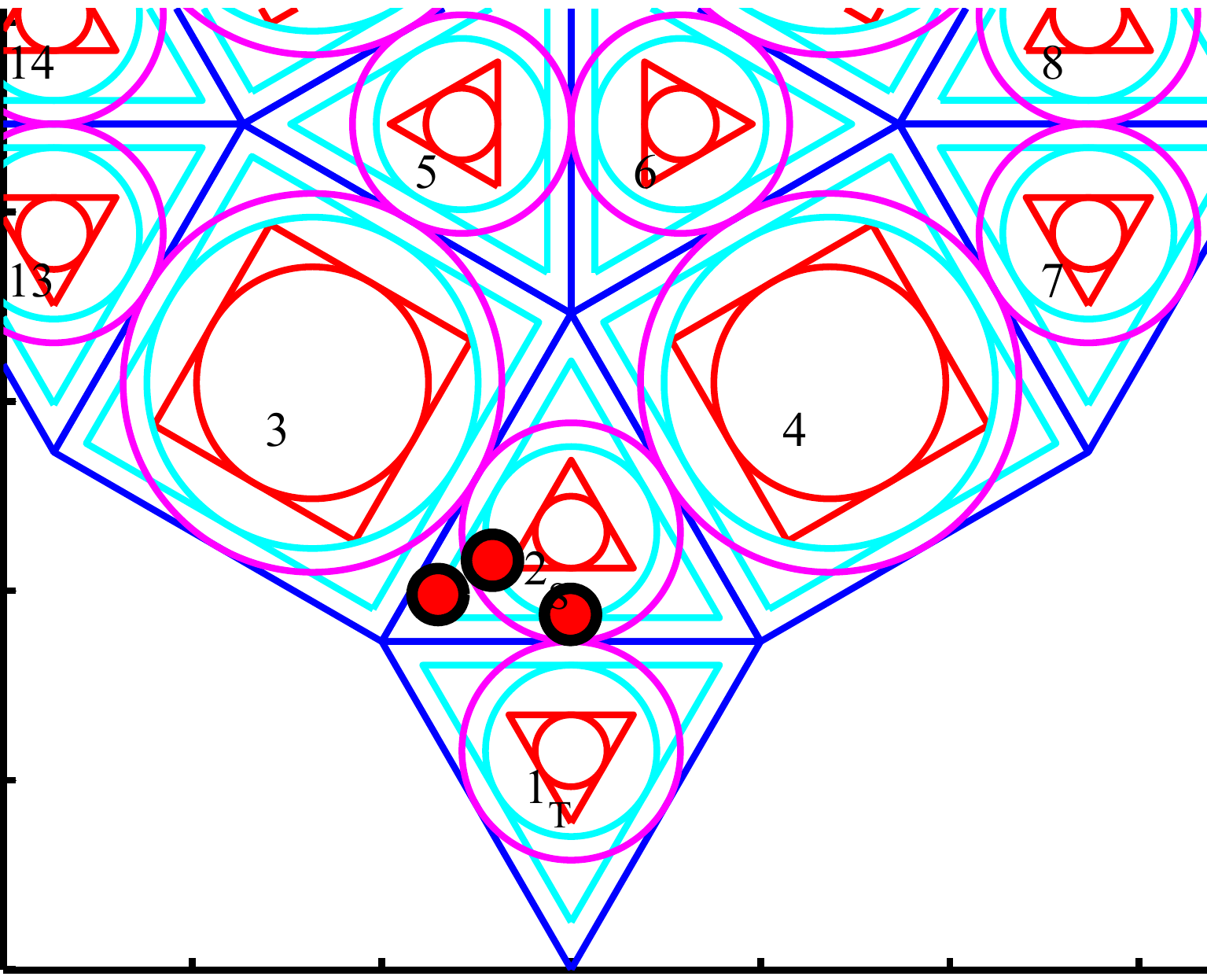}
		\caption{Safety regions (areas between red and blue) and transfer regions (areas between cyan and blue) for the squares and triangles composing the snub square tiling tessellation.}
		\figlabel{transferSafety}
	\end{minipage}
	\hspace{0.01\linewidth}
	\begin{minipage}[t]{0.4\linewidth}
		\vspace{0pt} % alignment hack
		\centering
		\includegraphics[width=\columnwidth]{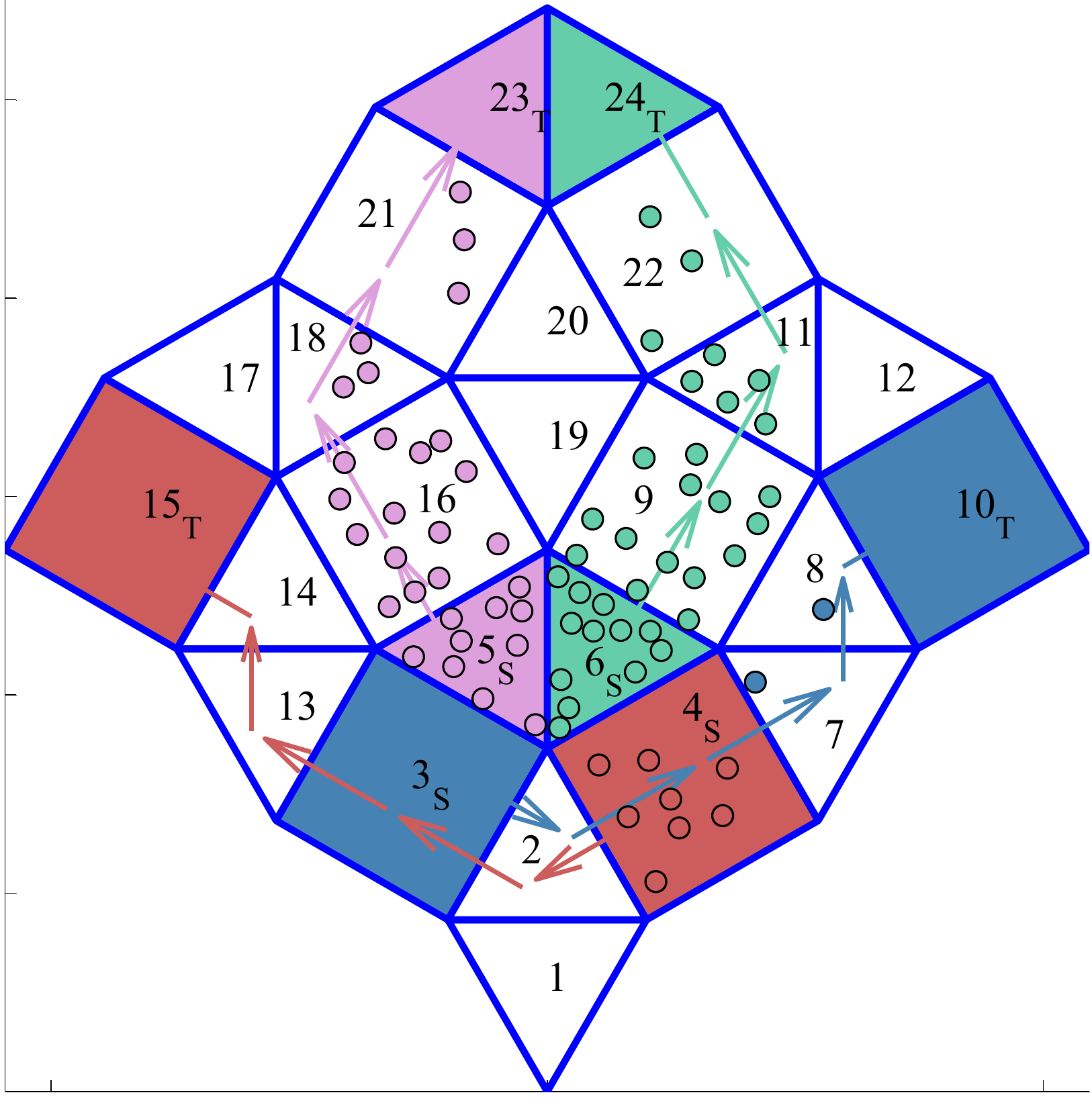}
		\caption{Blue and red paths overlap at cells $2$, $3$, and $4$.  Blue entities on cells $7$ and $8$ have traversed the intersection and then the red source ($4$) produces entities.  Red and blue sources producing entities simultaneously would cause a deadlock.}
		\figlabel{exampleSystemSnubSquare}
	\end{minipage}
\end{figure}

\paragraph*{Geometric Assumptions}
We assume that the polygonal environment $\Env$ and its partition $\partition$ have shapes and sizes such that each cell in the partition is large enough for an entity to lie completely on it.
Particularly, we require for each cell $i \in \ID$ that the transfer region $\transferRegioni$ is nonempty.
\taylor{previous: really mean to say cell minus transfer region is nonempty, i.e., there is at least one point in the cell not on the transfer region: also may want to actually say this for safety, e.g., the cell minus safety region is nonempty (otherwise cell sizes are problematic on receiving from multiple sides, etc....)}
%there may exist $p$ centered at some $(p_x, p_y) \in \mathbb{R}^2$ such that $D_{\l}(p) \setminus \partitioni = \emptyset$, that is, the open disc corresponding to entity $p$ lies entirely within cell $\partitioni$.
%
We also assume the following assumptions to ensure transferring entities between cells is well-defined.
%
%That is, there are partitions for which the constraints defining the new transfer point are infeasible.
%
\begin{assumption}
(Projection Property): For each $i \in \ID$, for each side $s$ of $\partitioni$, there exists a constant vector field over $\partitioni$ that drives every point in $\partitioni$ to some point on side $s$ without exiting $\partitioni$.
%
%$V : \partitioni \mapsto s$
%
\asslabel{projectionProperty}
\end{assumption}

\sayan{$\ast$ next paragraph and assumption}
By definition, the cells form a partition.
However, partially because there is ``empty space'' between the transfer regions of the cells, the transfer regions do not form a partition.
Even if we remove this empty space by translating the transfer regions so the sides of transfer regions of neighboring cells coincide, they still may not form a partition (see~\figref{transferSafety} for an example where the transfer regions cannot form a partition).
This is because, for the shared side $s$ of neighboring cells $i$ and $j$, the inner sides of the transfer regions on $\partitioni$ and $\partitionj$ may have different lengths, even though the shared side $s$ obviously had the same length for $\partitioni$ and $\partitionj$.
%
%We note that the following assumption is satisfied when the constant vector field provided by the projection property (\assref{projectionProperty}) passes through both vertices of $\innerTransferRegionj(\Sideij)$, for any neighbor $j$ of $i$.
%
\begin{assumption}
(Transfer Feasibility): For any $i \in \ID$ and any $j \in \Nbrsi$, consider their common side $\Sideij$.
The length of the inner side $\innerTransferRegioni(\Sideij)$ line segment equals the length of the inner side $\innerTransferRegionj(\Sideij)$ line segment.
%
%Let $a_i, b_i \in \reals^2$ be the endpoints of $\innerTransferRegioni(\Sideij)$ and let $a_j, b_j \in \reals^2$ be the endpoints of $\innerTransferRegionj(\Sideij)$, where we recall that $\innerTransferRegioni(\Sideij)$ is the inner side of the transfer region of $i$ corresponding to $\Sideij$, and likewise for $\innerTransferRegionj(\Sideij)$.
%
%There exists a line passing through $a_i$ and $a_j$, and a translation of this line (that is, a line with the same slope) passing through $b_i$ and $b_j$.
%
\asslabel{transferFeasibility}
\end{assumption}

\section{Distributed Traffic Control Algorithm}
\seclabel{algorithm}
Next, we describe the \emph{discrete transition system} $\Celli$ that specifies the software controlling an individual cell $\partitioni$ of the partition $\partition$.

\paragraph*{Preliminaries}
A {\em variable\/} is a name with an associate type.
For a variable $x$, its type is denoted by $\mathit{type(x)}$ and it is the set of values that $x$ can take.
A {\em valuation\/} (or state) for a set of variables $X$ is denoted by $\vx$, and is a function that maps each $x \in X$ to a point in $\mathit{type(x)}$.
Given a valuation $\vx$ for $X$, the valuation for a particular variable $v \in X$, denoted by $\vx.v$, is the restriction of $\vx$ to $\{v\}$.
The set of all possible valuations of $X$ is denoted by $val(X)$.
Many variables return cell identifiers that we use to access variables of other cells using subscripts, and if the valuation of these variables are restricted to the same state, we will drop the particular state on the subscripted variables for more concise notation.
For instance, suppose $\vx.\nexti \in \ID$, then $\vx.\next_{\vx.\nexti}$ would be written $\vx.\next_{\nexti}$.

\begin{figure}[t]
	\centering%
	\begin{minipage}[t]{0.475\linewidth}
		\vspace{0pt} % alignment hack
		\centering
		\includegraphics[width=\columnwidth]{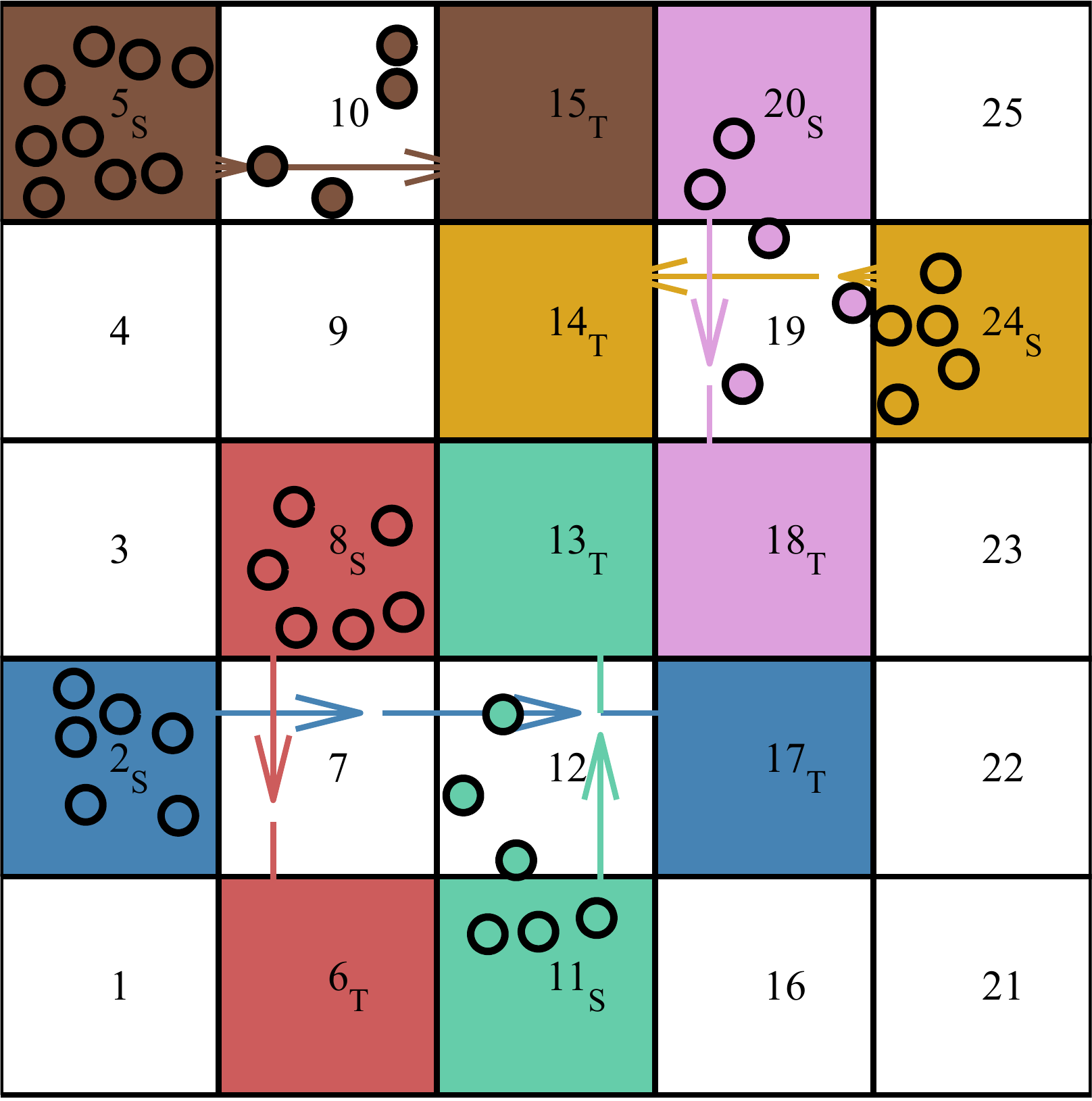}
		\caption{Example illustrating the computation of the color-shared cells and shared colors, stored in the $\pintc$ and $\lockcolorsic$ variables, respectively. The color-shared cells are any cells on overlapping paths, and $\lockcolorsic$ corresponds to the colors of each disjoint set of color-shared cells.}
		\figlabel{exampleLockcolors}
	\end{minipage}
	\hspace{0.01\linewidth}
	\begin{minipage}[t]{0.475\linewidth}
		\vspace{0pt} % alignment hack
		\centering
		\includegraphics[width=\columnwidth]{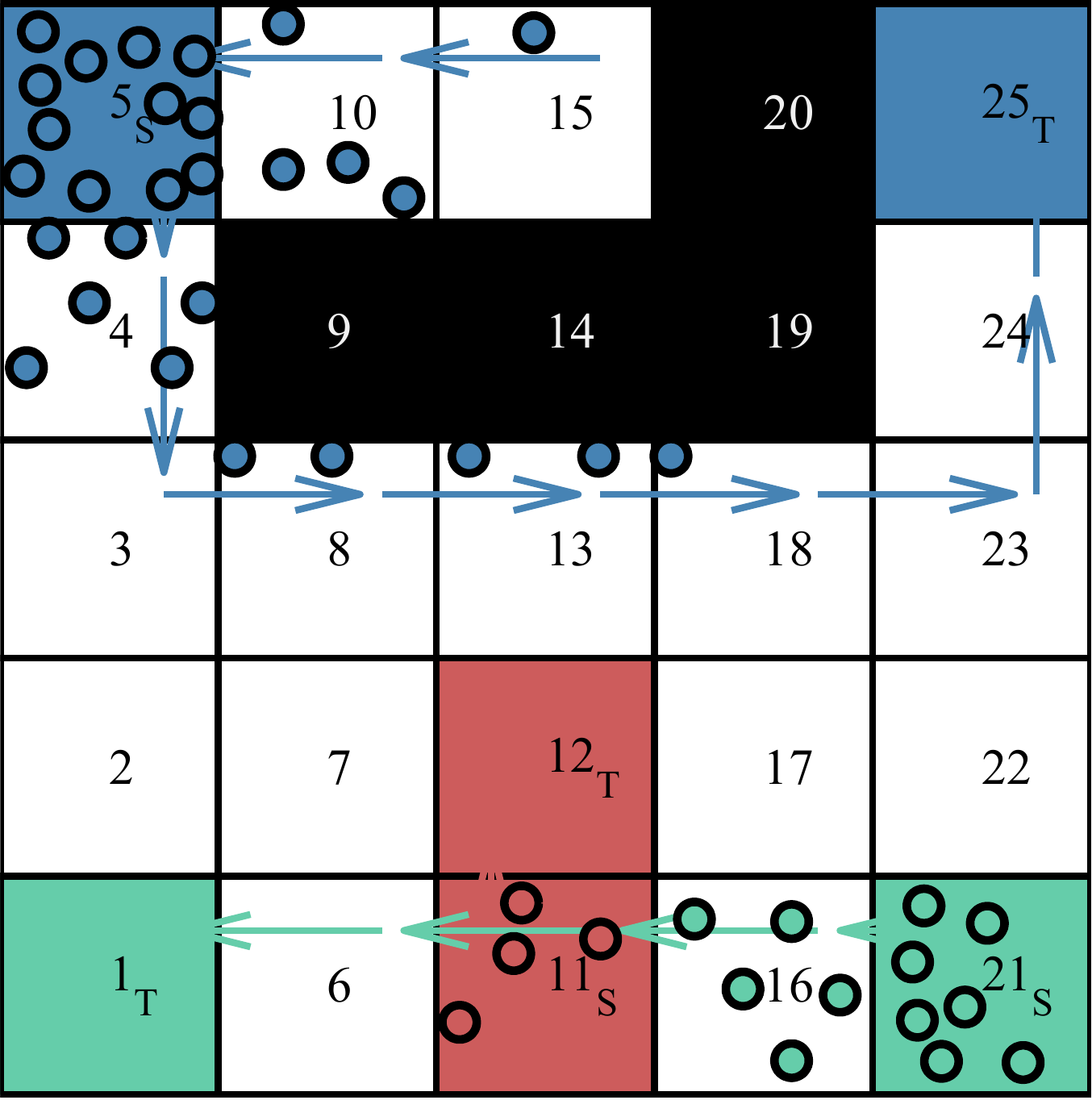}
		\caption{Example illustrating the two fairness requirements (\assref{fairness}) for proving liveness. Cells $9$, $14$, $19$, and $20$ failed, causing the original source-target path for blue to change from cells $5$, $10$, $15$, $20$, $25$.  If source cell $5$ does not place new entities fairly, then entities on cells $10$ and $15$ may never reach the target. A similar situation occurs with paths of multiple colors in the lower part of the image.}
		\figlabel{exampleFairness}
	\end{minipage}
\end{figure}

A {\em discrete transition system\/} $\A$ is a tuple $\langle X, Q_0, A, \rightarrow \rangle$, where:
\begin{enumerate}[(i)]
\item $X$ is a set of variables and $val(X)$ is called the set of {\em states\/}, 
\item $Q_0 \subseteq val(X)$ is the set of {\em start states\/},
\item $A$ is a set of {\em transition names\/}, and
\item $\rightarrow \subseteq val(X) \times A \times val(X)$ is a set of {\em discrete transitions\/}.
For $(\vx_k, a, \vx_{k+1}) \in \rightarrow$, we also use the notation $\vx_k \arrow{a} \vx_{k+1}$.
\end{enumerate}

An {\em execution fragment\/} of a discrete transition system $\A$ is a (possibly infinite) sequence of states $\alpha = \vx_0, \vx_1, \ldots$, such that for each index appearing in $\alpha$, $(\vx_k, a, \vx_{k+1}) \in \rightarrow$ for some $a \in A$.
An {\em execution\/} is an execution fragment with $\vx_0 \in Q_0$.
A state $\vx$ is said to be {\em reachable\/} if there exists a finite execution that ends in $\vx$.
$\A$ is said to be \emph{safe} with respect to a set $S \subseteq val(X)$ if all reachable states are contained in $S$.
A set $S$ is said to be {\em stable\/} if, for each $(\vx, a, \vx') \in \rightarrow$, $\vx \in S$ implies that $\vx' \in S$.
$\A$ is said to {\em stabilize} to $S$ if $S$ is stable and every execution fragment eventually enters $S$.

\paragraph*{Cells}
We assume messages are delivered within bounded time and computations are instantaneous.
Under these assumptions, the system can be modeled as a collection of discrete transition systems.
The overall system is obtained by composing the transition systems of the individual cells.
We first present the discrete transition system corresponding to each cell, and then describe the composition.

The variables associated with each $\Celli$ are as follows, with initial values of the variables shown in~\figref{celli} using the `:=' notation.
\begin{enumerate}[(a)]
\item $\Membersi$ is the set of identifiers for entities located on cell $i$.
Cell $i$ is said to be nonempty if $\Membersi \neq \emptyset$.
%
%\item $\Mnei$ is a bit indicating whether cell $i$ was (previously) nonempty and is used in scheduling multi-color flows.
%
\item $\etypei$ designates the entity colors on the cell, or $\bot$ if there are none\footnote{It will be established that cells contain entities of only a single color, see~\invref{onecolor}.}.
\item $\failedi$ indicates whether or not $i$ has failed.
\item $\NEPrevi$ are the nonempty neighbors attempting to move entities (of any color) toward cell $i$.
\item $\tokeni$ is a token used for fairness to indicate which neighbor may move toward $i$.
\item $\signali$ is the identifier of a neighbor of $\Celli$ that has permission to move toward $\Celli$.
% (as otherwise we would repeatedly refer to the entities' colors)
%
%\item $\fdi$ is a bit indicating whether or not one of $j$'s neighbors has failed.
%
\end{enumerate}
Additionally, the following variables are defined as arrays for each color $c \in C$.
The notation $\nextic$ means the $c^{th}$ entry of the $\next$ variable of cell $i$, and so on for the other variables.
\begin{enumerate}[(a)]
\item $\nextic$ is the neighbor towards which $i$ attempts to move entities of color $c$.
\item $\distic$ is the estimated distance---the number of cells---to the nearest target cell consuming entities of color $c$.
\item $\lockic$ is a boolean variable for a lock of color $c$ that some cells require to be able to move entities.
%
%\item $\nlockic$ is a boolean variable that is true if cell $i$ needs the color $c$ lock to accept entities.
%
\item $\pathic$ is the set of cell identifiers from any source of color $c$ (and any nonempty cell with entities of color $c$) to the target of color $c$.  This variable and the next two are local variables, but they are storing some global information.
\item $\pintic$ is the set of cell identifiers in traffic intersections with cells of color $c$ (where $\pathic$ and $\pathid$ have nonempty intersection for some $d \neq c$).
\item $\lockcolorsic$ is the set of colors that are involved in traffic intersections with the color $c$ path.
\taylor{check that all variables have been listed and informally defined}
\end{enumerate}
When clear from context, the subscripts in the names of the variables are dropped.
A state of $\Celli$ refers to a valuation of all these variables, i.e., a function that maps each variable to a value of the corresponding type.
The complete system is an automaton, called $\System$, consisting of the composition of all the cells. 
A state of $\System$ is a valuation of all the variables for all the cells.
We refer to states of $\System$ with bold letters $\vx$, $\vx'$, etc.
%
%\begin{figure}[t]
%	\centering
%	\hrule
%	\two{.47}{.47}
%		{\lstinputlisting[language=ioaNums,lastline=13]{code/celli.hioa}}
%		{\lstinputlisting[language=ioaNumsRight,firstline=14,firstnumber=14]{code/celli.hioa}}
%	%\begin{lstlisting}[language=ioa,firstline=1,numbers=left,stepnumber=2,numbersep=-8pt]
%	%\end{lstlisting}
%	\hrule
%	\caption{Specification of $\Celli$ listing its variables, initial conditions, and transitions.  Subscripts are dropped for readability.}
%	\figlabel{celli}
%\end{figure}

\begin{figure}[t]
	\centering
	\mybox{\columnwidth}{\columnwidth}{
		\twosep{.54}{.44}
		{
\begin{lstlisting}[language=ioaNums,lastline=13,numbersep=-3pt]
	variables
		$\Members$ : $\tcon{Set}[P]$ := {}
		$\NEPrev$: $\tcon{Set}[\ID_{\bot}]$ := {}
		$\signal$, $\token$ : $\ID_{\bot}$ := $\bot$
		$\etype$ : $C_{\bot}$ := $\bot$
		$\failed$ : $\booleans$ := $\false$
		$\next$  : $[C \rightarrow \ID_{\bot}]$, init $\forall c \in C$, $\nextc := \bot$
		$\dist$  : $[C \rightarrow \naturals_{\infty}]$, init $\forall c \in C$, $\distc := \infty$
		$\path$  : $[C \rightarrow \tcon{Set}[\ID_{\bot}]]$, init $\forall c \in C$, $\pathc$ := {}
		$\pint$  : $[C \rightarrow \tcon{Set}[\ID_{\bot}]]$, init $\forall c \in C$, $\pintc$ := {}
		$\nlock$ : $[C \rightarrow \mathbb{B}]$, init $\forall c \in C$, $\nlock$ := $\true$
		$\lock$  : $[C \rightarrow \mathbb{B}]$, init $\forall c \in C$, $\lock$ := $\false$
		$\lockcolors$ : $[C \rightarrow \tcon{Set}[C]]$, init $\forall c \in C$, $\lockcolorsc$ := {}
\end{lstlisting}}
		{
\begin{lstlisting}[language=ioaNumsRight,firstline=14,firstnumber=14,numbersep=-10pt]
	transitions
		$\act{fail}(i)$
			eff $\failed$ := $\true$
				for each $c \in C$
					$\distc$ := $\infty$; $\nextc$ := $\perp$
	
		$\act{update}$
			eff $\Route$; $\Lock$; $\Signal$; $\Move$
\end{lstlisting}}
	\centering
	\makebox[\columnwidth][l]{\hrulefill~~~~}
	\centering
	\parbox{0.95\columnwidth}{
		\caption{Specification of $\Celli$ listing its variables, initial conditions, and transitions.  Subscripts are dropped for readability.}
		\figlabel{celli}
		}
	}
\end{figure}

%$\nlocki$
\taylor{check to state all as shared or private}
Variables $\tokeni$, $\failedi$, $\locki$, and $\NEPrevi$ are private to $\Celli$, while $\Membersi$, $\disti$, $\nexti$, $\pathi$, $\etypei$, and $\signali$ can be read by neighboring cells of $\Celli$.
%
%See ~\figref{sharedvars}.
%
This has the following interpretation for an actual message-passing implementation.
At the beginning of each round, $\Celli$ broadcasts messages containing the values of these variables and receives similar values from its neighbors.
Then, the computation of this round updates the local variables for each cell based on the values collected from its neighbors.

Variable $\Membersi$ is a special variable because it can also be written to by the neighbors of $i$.
This is how we model transfer of entities between cells.
For a state $\vx$, for some $a \in A$ such that $\vx \arrow{a} \vx'$, for some $i \in \ID$, for some $j \in \Nbrsi$, for some entity $p \in \vx.\Membersi$, then entity $p$ transfers from cell $i$ to $j$ when $p \in \vx'.\Membersj$.
We use the notation $p'$ to denote the state of entity $p$ at $\vx'$ where $\vx \arrow{a} \vx'$ for some $a \in A$.
%
%\begin{figure}[!ht]
%	\centering
%	\includegraphics[scale=0.20]{image/sharedvariables.pdf}
%	\caption{The interaction between a pair of neighboring cells is modeled with shared variables $\Members$, $\dist$, $\next$, and $\signal$.}
%	\figlabel{sharedvars}
%\end{figure}
%

\paragraph*{Actions for the Composed System}
$\System$ is a discrete transition system modeling the composition of all the cells, and has two types actions: $\act{fail}$s and $\act{update}$s.
A $\act{fail}(i)$ transition models the crash failure of the $i^{th}$ cell and sets $\failedi$ to $\true$, $\distic$ to $\infty$ for each $c \in C$, and $\nextic$ to $\bot$ for each $c \in C$.
Cell $i$ is called \emph{faulty} if $\failedi$ is $\true$, otherwise it is called \emph{non-faulty}.
The set of identifiers of all faulty and non-faulty cells at a state $\vx$ is denoted by $F(\vx)$ and $\NF(\vx)$, respectively.
A faulty cell does nothing---it never moves and it never communicates\footnote{$\disti = \infty$ can be interpreted as $i$'s neighbors not receiving a timely response from $i$.}.

An $\act{update}$ transition models the evolution of all non-faulty cells over one synchronous round.
For readability, we describe the state-change caused by an $\act{update}$ transition as a sequence of four functions (subroutines), where for each non-faulty $i$, 
\begin{enumerate}[(a)]
\item $\Route$ computes the variables $\disti$ and $\nexti$,
\item $\Lock$ computes the variables $\pathi$, $\pinti$, $\lockcolorsi$, and $\locki$,
\item $\Signal$ computes (primarily) the variable $\signali$, and
\item $\Move$ computes the new positions of entities.
\end{enumerate}
We note that in the single-color case considered in~\cite{johnson2010icdcs}, the $\Lock$ subroutine is unnecessary.
% and any condition using a variable it modifies can be assumed vacuously true.

The entire $\act{update}$ transition is atomic, so there is no possibility to interleave $\act{fail}$ transitions between the subroutines of $\act{update}$.
Thus, the state of $\System$ at (the beginning of) round $k+1$ is obtained by applying these four functions to the state at round $k$.
Now we proceed to describe the distributed traffic control algorithm that is implemented through these functions.

\paragraph*{$\Route$}
For each cell and each color, the $\Route$ function (\figref{route}) constructs a distance-based routing table to the target cell of that color.
This relies only on neighbors' estimates of distance to the target.
Recall that failed cells have $\distc$ set to $\infty$ for every color $c \in C$.
From a state $\vx$, for each $i \in \NF(\vx)$, the variable $\distic$ is updated as $1$ plus the minimum value of $\distjc$ for each neighbor $j$ of $i$.
If this results in $\distic$ being infinity, then $\nextic$ is set to $\bot$, but otherwise it is set to be the identifier with the minimum $\distc$ where ties are broken with neighbor identifiers.

%%
%The route function is also responsible for detecting failures, which is unnecessary if $\abs{C} = 1$.
%%
%This simple failure detection mechanism causes cells to stop moving until routes have stabilized to avoid causing deadlocks as failures occur.
%%
%The mechanism is simple: if $\nextic \neq \bot$ and $\dist_{\nextic}\colorc = \infty$, then $i$ believes $\nextic$ has failed.

\begin{figure}[t]
	\centering
	\mybox{\columnwidth}{\columnwidth}{
		%\twosep{.54}{.44}
		{
\begin{lstlisting}[language=ioaNums,numbersep=-3pt]
	if $\neg \failedi$ then
		$\etypei$ := $\{ c \in C : \exists p \in \Membersi \wedge \etypep = c \}$
		if $i \notin \TID$ (*@\lnlabel{algo:distT}@*) then 
			for each $c \in C$
				$\distic$ := $\left( \minel{\distjc}{j \in \Nbrsi} \right) + 1$ (*@\lnlabel{algo:distmin}@*)
				if $\distic$ = $\infty$ then $\nextic$ := $\perp$
				else $\nextic$ := $\argmin{\pair{\distjc, j}}{j \in \Nbrsi }$ (*@\lnlabel{algo:nextmin}@*)
\end{lstlisting}}%,lastline=4
		%{}
		%{\lstinputlisting[language=ioaNumsRight,firstline=5,firstnumber=5,numbersep=-10pt]{code/route.hioa}}
	\centering
	\makebox[\columnwidth][l]{\hrulefill~~~~}
	\centering
	\parbox{0.95\columnwidth}{
		\caption{$\Route$ function for $\Celli$.  This function computes a minimum distance vector routing spanning tree rooted composed of non-faulty cells for each color, rooted at each target.}
		\figlabel{route}
		}
	}
\end{figure}

Next, we introduce some definitions used to relate the system state to the variables used in the algorithm.
For a state $\vx$, we inductively define the \emph{color $c$ target distance} $\rhoc$ of a cell $i \in \ID$ as the smallest number of non-faulty cells between $i$ and $\tidc$:
\begin{align*}
 \rhoc(\vx,i) \deq & \left\{ \begin{array}{ll}
		\infty & \mbox{if} \ \vx.\failedi, \\
		0 & \mbox{if} \ i = \tidc \wedge \neg \vx.\failedi, \\
		1 + \min\limits_{j \in \vx.\Nbrsi} \rhoc(\vx, j) & \mbox{otherwise}.
	\end{array} \right.
\end{align*}
%
%We remind the reader that we assumed earlier that there is a unique target $\tidc$ for each color $c \in C$.
%
A cell is said to be \emph{target-connected to color $c$} if $\rhoc$ is finite.
We define 
\begin{align*}
\TC(\vx, c) & \deq \{ i \in \NF(\vx) \ | \ \rhoc(\vx, i) < \infty \}
\end{align*}
as the set of cells that are target-connected to $\tidc$.

For a state $\vx$ and a color $c \in C$, we define the \emph{routing graph} as $\GR(\vx, c) = (\VR(\vx, c), \ER(\vx, c))$, where the vertices and directed edges are, respectively, 
\begin{align*}
\VR(\vx, c) & \deq \NF(\vx) \ \mbox{and}	\\
\ER(\vx, c) & \deq \{ (i, j) \in \VR(\vx, c) : \rhoc(\vx,j) = \rhoc(\vx,i) + 1 \}.
\end{align*}
Under this definition, $\GR(\vx, c)$ is a spanning tree rooted at $\tidc$.
We will show that the graph induced by the $\nextic$ variables stabilizes to the routing graph $\GR(\vx, c)$ at some state $\vx$ (\corref{next}).
We previously introduced $\Delta$ as the worst-case diameter of the communication graph, and will refer to $\Delta(\vx)$ as the exact diameter at some state $\vx$.

\paragraph*{$\Lock$}
\sayan{$\ast$: rewrote a lot of this, moved definitions from analysis section to here to avoid repetition}
The $\Lock$ function (\figref{lock}) executes after $\Route$, and schedules traffic over intersections (the cells where source-to-target paths of different colors overlap).
To avoid deadlock scenarios, $\Lock$ maintains an invariant that entities of at most one color are on these intersections.

Moving entities over intersections requires some global coordination as illustrated by the following analogy.
Consider the policy used to coordinate cars going in opposite directions over a one-lane bridge (see~\figref{exampleSystemOneLane}), where there is a traffic signal on each side of the bridge.
%
%Suppose initially, both lights are red and there are cars waiting on each side to travel over an east-west one-lane bridge.
%
The algorithm chooses one traffic light, allowing some cars to safely travel over the bridge in one direction.
After some time, the algorithm switches the lights (first turning green to red, and after the road is clear, turning red to green) allowing traffic to flow in the opposite direction.
Then this process repeats.

\begin{figure}[t]
	\centering
	\mybox{\columnwidth}{\columnwidth}{
		%\twosep{.49}{.49}
		{
\begin{lstlisting}[language=ioaNums,numbersep=-3pt]
	if $\neg \failedi$
		for each $c \in C$
			if $i = \sidc \vee \etypei = c \vee i \in \pathic$ then	(*@\lnlabel{pathif}@*)
				$\pathic$ := $\pathic \cup \{i\} \cup \{\nextic\}$	(*@\lnlabel{pathset}@*)
			$//$ $\mbox{gossip the entity graph}$
			for each $j \in \Nbrsi$, $\pathic$ := $\pathic \cup \pathjc$	(*@\lnlabel{pathgossip}@*)
			$//$ $\mbox{compute the set of color-shared cells}$
			$\pintic := \{ j \in \pathic \cap \pathid : \exists c \neq d \in C\}$	(*@\lnlabel{pathint}@*)
			if $\pintic \neq \emptyset$
				$\lockcolorsic := $
					$\{ d \in C : c \neq d \wedge \pathic \cap \pathid \neq \emptyset \}$
			$//$ $\mbox{graphs stabilized and i needs a lock for color c}$
			if round > $2\Delta \wedge i \in \pintic \wedge \neg \lockic$ (*@\lnlabel{pathmutex}@*)
				$\mbox{Initiate mutual exclusion algorithm between all}$
				$\mbox{color-shared cells in}$ $\pintic$ $\mbox{using }$ $\lockcolorsi$ $\mbox{ as input}$
				$\mbox{Eventually, a color }$ $d$ $\mbox{ is returned}$.
				$\mbox{On return}$, if $d = c$ then $\lockic := \true$
			$//$ $\mbox{detect if color-shared cells are empty}$
			if round > $2\Delta \wedge i \in \pintic \wedge \lockic$ (*@\lnlabel{pathempty}@*)
				$\mbox{Initiate distributed snapshot algorithm to decide}$
				$\mbox{if all color-shared cells are empty after previously}$
				$\mbox{being nonempty with entities of color c.}$
				$\mbox{On return}$, if $\mbox{all cells are empty}$ then
				$\lockic := \false$
\end{lstlisting}}%,lastline=11
		%{\lstinputlisting[language=ioaNumsRight,firstline=12,firstnumber=12,numbersep=-10pt]{code/lock.hioa}}
	\centering
	\makebox[\columnwidth][l]{\hrulefill~~~~}
	\centering
	\parbox{0.95\columnwidth}{
		\caption{$\Lock$ function for $\Celli$.  This function computes the color-shared cells---the cells in intersections---for each color, and then ensures liveness by giving a lock to only one color on each intersection.}
		\figlabel{lock}
		}
	}
\end{figure}

Two parts of the previous example require global coordination and are included in the $\Lock$ function.
The first is how to chose the direction in which cars are allowed to travel---this is accomplished through the use of a mutual exclusion algorithm.
The second is when to allow cars to travel in the opposite direction---this is accomplished by determining when the intersection is empty.
We now describe this global coordination more formally.

%Obviously all the routing trees overlap at all vertices since they are spanning trees, but the cells that have entities will only be either on sources, or on an \emph{entity graph} defined from sources (and nonempty cells) toward the target of the same color.
%
For defining the locking algorithm, we first define intersections.
For this we introduce the notion of an entity graph.
Cell $i$ is said to be in the \emph{entity graph} of some color $c$ at state $\vx$ if one of the following conditions hold:
\begin{inparaenum}[(a)]
\item $i$ is $\sidc$,
\item in state $\vx$, $i$ has entities of color $c$, or 
\item in state $\vx$, $i$ is the neighbor closest to $\tidc$ of a cell already in the entity graph.
%---as defined by the number of cells---
%
\end{inparaenum}
Formally, we define the \emph{color $c$ entity graph} at state $\vx$ as $\GE(\vx, c) = (\VE(\vx, c), \EE(\vx, c))$, which is the following subgraph of the color $c$ routing graph $\GR(\vx, c)$.
The vertices of $\GE(\vx,c)$ are inductively defined as 
\begin{align*}
\VE(\vx, c) \deq \{ i \in \NF(\vx) : i = \sidc \vee \vx.\etypei = c \vee \left( \exists j \in \VE(\vx, c) . (i,j) \in \ER(\vx, c)  \right)  \}.
\end{align*}
%\left( \exists j \in \VE(\vx, c) \wedge \rhoc(\vx,j) = \rhoc(\vx,i) + 1  \right)
%
The edges of $\GE(\vx,c)$ are $\EE(\vx, c) \deq \{ (i, j) \in \VE(\vx, c) \times \VE(\vx, c) : (i, j) \in \ER(\vx, c) \}$.
For example, if all cells are empty, then $\VE(\vx, c)$ is the sequence of cell identifiers defined by following the minimum distance (as defined by $\rhoc$) from the source to the target of color $c$.
That is, each $\GE(\vx, c)$ is a simple path graph from source to target\footnote{Once cells have failed, this may stabilize to be a tree from any cell with entities of color $c$ to the target of color $c$.}.

Now we describe how the entity graph of each color $c$ is computed by each cell $i$ as the $\pathic$ variable.
If $i$ is on the entity graph of color $c$, then we add $i$ and $i$'s $\next$ variable for color $c$ to the entity graph (see~\figref{lock},~\lnreftwo{pathif}{pathset}).
Once the $\nextic$ variables stabilize (\corref{next}) and after an additional order of diameter rounds, the variable $\pathic$ contains all the entity graphs since we gossip these graphs (\lnref{pathgossip}).
That is, the graph formed by the $\pathic$ variables stabilizes to equal $\GE(\vx,c)$, and contains the sequence of identifiers from any source or nonempty cell of color $c$ to the target of color $c$ (\corref{pathGlobal}).

Next, the variable $\pintic$ is computed to be the set of cell identifiers on the color $c$ entity graph that overlaps with any other colored entity graph (\lnref{pathint}).
The cells involved in such non-empty intersections represent physical traffic intersections, and are called \emph{color-shared cells}.
These cells require coordinated locking for traffic flow to progress.
Cell $i$ is in $\pintic$ if and only if it will need a lock for color $c$.% (\lnref{nlock} and~\lnref{nlockoff}).

Formally, we define the \emph{$c$ color-shared cells}, for a state $\vx$, for any $c \in C$, as
\begin{align*}
\CSC(\vx, c) = \{ \VE(\vx, c) : \exists d \in C . c \neq d \wedge \VE(\vx, c) \cap \VE(\vx, d) \neq \emptyset \}.
\end{align*}
In~\figref{exampleSystemTriangular}, these are cells $8$ and $12$.
The $\pintic$ variables stabilize to equal $\CSC(\vx,c)$, at some state $\vx$, for any color $c$ (\corref{pintGlobal}).

%Since we know for each color that each target-connected cell to that color knows the sequence of identifiers that will have entities of other colors on them to reach their respective targets by~\corref{pathGlobal}, the intersections of the entity graphs computed as $\pintic$ in~\figref{lock},~\lnref{pathint} identify any cells that require coordination.

%The scheduling logic is then, for any color-shared cells, only allow entities of one color to flow through these cells at a time.
%%
%This is achieved through the use of locks, where a cell $i$ which needs a lock (is a color-shared cell) may only receive the permission to move $\signalj = i$ for a neighbor $j$ if cell $i$ has that lock.
%%
%For all the intersecting colors, one is selected via a consensus algorithm.

Next, we need to determine the colors that will need to coordinate to schedule traffic through the color-shared cells.
Then, a mutual exclusion algorithm is initiated between all cells for each disjoint set of cell colors in $\pintic$.
Formally, we define the \emph{$c$ shared colors}, for a state $\vx$, for any $c \in C$, as
\begin{align*}
\SC(\vx, c) = \{ d \in C : c \neq d \wedge \CSC(\vx, d) = \CSC(\vx, c) \}.
\end{align*}
The $\lockcolorsic$ variables stabilize at some state $\vx$ to equal $\SC(\vx,c)$, for any color $c$.

In general, up to $\abs{C}$ colors could be involved in intersections, as well as all the smaller permutations.
For instance, consider~\figref{exampleLockcolors} with $6$ colors at some state $\vx$.
Here, the blue and red entity graphs overlap, green and blue entity graphs overlap, but red and green do not, and independently, the purple and yellow entity graphs overlap (that is, not with blue, red, nor green), but no colors overlap with brown.
Then $\SC(\vx, c)$ is $\{blue, red, green\}$ for $c$ equal to blue, red, or green, $\SC(\vx, c)$ is $\{yellow, purple\}$ for $c$ equal to yellow or purple, and $\SC(\vx, c)$ is empty for $c$ equal to brown.
Two mutual exclusion algorithms would be initiated, one with blue, red, and green as the input set of values, and another with yellow and purple as the input set.
Upon these two instances terminating, one element of the first set, say $green$, would be chosen and given a lock, and one element, say $yellow$, of the second set would also be given a lock.
The entities of these colors progress over the color-shared cells toward their intended targets.
Finally, once the color-shared cells are empty again, $green$ and $yellow$ would each be removed from the respective input sets for fairness, and another mutual exclusion algorithm is initiated.

%$//$ $\mbox{if i is in the set of c color-shared cells, then i needs a lock for color c}$
%if $i \in \pintic$ then $\nlockic := \true$	(*@\lnlabel{nlock}@*)
%$//$ $\mbox{if i is not a color-shared cell, then it does not need a lock}$
%else $\nlockic := \false$	(*@\lnlabel{nlockoff}@*)
%if $\nlockic$

%\begin{figure}[b]
%	\centering
%	\hrule
%	\two{.47}{.47}
%			{\lstinputlisting[language=ioaNums,lastline=10]{code/signal.hioa}}
%			{\lstinputlisting[language=ioaNumsRight,firstline=11,firstnumber=11]{code/signal.hioa}}
%	%\begin{lstlisting}[language=ioa,firstline=1,numbers=right,stepnumber=2,numbersep=-8pt]	
%	%\end{lstlisting}
%	\hrule
%	\caption{$\Signal$ function for $\Celli$.}
%	\figlabel{signal}
%\end{figure}

\begin{figure}[t]
	\centering
	\mybox{\columnwidth}{\columnwidth}{
		\twosep{.49}{.49}
		{
\begin{lstlisting}[language=ioaNums,lastline=10,numbersep=-3pt]
	if $\neg \failedi$ /\ $round > 2\Delta$ then
		$cn := \{ d \in C : \exists j \in \Nbrsi \ \mbox{s.t.} \ \nextjd = i $
									$\wedge \etypej = d \}$
	
		if $\etypei = \bot$ then $c := $ choose from $cn$
		else $c := \etypei$
	
		$\NEPrevi$ := 
			{$j$ \in $\Nbrsi$ : $\nextjc$ = $i$ /\ $\Membersj$ \ne $\emptyset$} (*@\lnlabel{algo:type}@*)
	
\end{lstlisting}}
		{
\begin{lstlisting}[language=ioaNumsRight,firstline=11,firstnumber=11,numbersep=-10pt]
		if $\tokeni$ = $\bot$ then 
			$\tokeni$ := choose from $\NEPrevi$
		
		let $j$ = $\tokeni$
		if \A p \in $\Membersi$ : $\pc \notin \safetyRegion(i, j)$	(*@\lnlabel{signalSafe}@*)
			/\ ($\etypei \neq \bot \Rightarrow \etypei = \etypej$) (*@\lnlabel{signalOneColor}@*)
			/\ ($j \in \pintic \Rightarrow \lockjc$)		(*@\lnlabel{signalLock}@*)
		then (*@\lnlabel{algo:lockout}@*)
			$\signali$ := $j$ 
			if $\abs{\NEPrevi}$ > 1 then (*@\lnlabel{algo:token1}@*)
				$\tokeni$ := choose from $\NEPrevi \setminus \{j\}$
			elseif $\abs{\NEPrevi}$ = 1 then (*@\lnlabel{algo:token4}@*)
				$\tokeni$ := choose from $\NEPrevi$
			else $\tokeni$ := $\perp$
		else $\signali$ := $\bot$; $\tokeni$ := $j$ (*@\lnlabel{algo:tokenSame}@*)
\end{lstlisting}}
	\centering
	\makebox[\columnwidth][l]{\hrulefill~~~~}
	\centering
	\parbox{0.95\columnwidth}{
		\caption{$\Signal$ function for $\Celli$.  Cell $i$ signals fairly to some neighbor $j$ if it is safe for $j$ to move its entities toward $i$.}
		\figlabel{signal}
		}
	}
\end{figure}

\paragraph*{$\Signal$}
The $\Signal$ function (\figref{signal}) executes after $\Lock$.
It is the key part of the protocol for maintaining safe entity separations, guaranteeing each cell has entities of only a single color, and ensuring progress of entities to the target.
Roughly, each cell implements this through the following policies:
\begin{inparaenum}[(a)]
\item only accept entities from a neighbor when it is safe to do so, 
\item only accept entities with the same color as the entities currently on the cell (or an arbitrary color if the cell is empty),
\item if a lock is needed, then only let entities move if it is acquired, and
\item ensure fairness by providing opportunities infinitely often for each nonempty neighbor to make progress.
\end{inparaenum}

First $i$ computes a temporary variable $cn$, which is the set of colors for any neighbor that has entities of some color, with the corresponding $\next$ variable set to cell $i$.
Next, cell $i$ picks a color $c$ from this set if it is empty, or the color of its own entities if it is nonempty, and will attempt to allow some cell with this chosen color to move toward itself.
Then, cell $i$ sets $\NEPrevi$ to be the subset of $\Nbrsi$ for which $\next$ has been set to $i$ \emph{and} $\Members$ is nonempty.
If $\tokeni$ is $\bot$, then it is set to some arbitrary value in $\NEPrevi$, but it continues to be $\bot$ if $\NEPrevi$ is empty.
Otherwise, $\tokeni = j$ for some neighbor $j$ of $i$ with nonempty $\Membersj$.
This is accomplished through the conditional in~\lnref{algo:gap} as a step in guaranteeing fairness.

It is then checked if there is any entity $p$ with center $\pc$ in the safety region of $\Celli$ on the side corresponding to $\tokeni$.
If there is such an entity, then $\signali$ is set to $\bot$, which blocks the neighboring cell with identifier $\tokeni$ from moving its entities in the direction of $i$, thus preventing entity transfers and ensuring safety.
Otherwise, if there is no entity with center in the safety region on side $\tokeni$, then $\signali$ is set to $\tokeni$ to allow $\tokeni$ to move its entities toward $i$.
Subsequently, $\tokeni$ is updated to a value in $\NEPrevi$ that is different from its previous value, if that is possible according to the rules just described (\lnsref{algo:token1}{algo:token4}).

\paragraph*{$\Move$}
Finally, the $\Move$ function (\figref{movement}) models the physical movement of all the entities on cell $i$ over a given round.
For cell $i$, let $j$ be $\nextic$, where $c$ is $\etypei$ (which may be $\bot$ if cell $i$ has no entities).
Every entity in $\Membersi$ moves in the direction of $j$ if and only if $\signalj$ is set to $i$.
The direction followed from cell $i$ to $j$ is $\MoveVectorij$, which is any vector satisfying~\assref{projectionProperty}.
%
%Intuitively, this is any line (or translation thereof) that can uniformly map each point in $\partitioni$ to a point on $\Sideij$.
%
For example, for a square (or rectangular) cell $i$, one choice for $\MoveVectorij$ is the unit vector orthogonal to $\Sideij$ and pointing into $j$.
%
%For a triangular cell $i$, let $o = \{ x \in \mathbb{R}^2 : \Verticesi \setminus (\Verticesi \cap \Verticesj)\}$ be the vertex of $i$ not on the side shared with $j$.
%%
%Let $r$ be the centroid of $\partitioni$, that is, $\frac{1}{\nsidesi} \sum_{a=1}^{\nsidesi} \Verticesi^a$, where $\Verticesi^a$ is the $a^{th}$ element of $\Verticesi$.
%%
%Then, define $\MoveVectorij \deq \frac{o - r}{\norm{o - r}}$ as the unit vector going from the vertex of $i$ not shared with $j$ and through the centroid of $\partitioni$.
%%
In the case of an equilateral triangular cell $i$, one choice for $\MoveVectorij$ is also any orthogonal vector pointing into $j$.
%%
%For a parallelogram cell $i$, $\MoveVectorij$ is parallel to the sides of $i$.

The movement toward cell $j$ may lead to some entities crossing the boundary of $\Celli$ into $\Cellj$, in which case, they are removed from $\Membersi$.
If $j$ is not the target matching the transferred entities' color, then the removed entities are added to $\Membersj$.
In this case (\lnref{algo:reset}), any transferred entity $p$ is placed so that $D_{\l}(p)$ touches a single point of (is tangent to) $\Sideij$, the shared side of cells $i$ and $j$, and lies on the inner side of the transfer region of cell $j$ on side $\Sideij$.
%
%In more detail, entity $p$'s position is set to be $\ResetEntity(p, i, j) \deq \{x \in \mathbb{R}^2 : \exists t \in \mathbb{R}, x = \pc + t \MoveVectorij \wedge D_{\l}(p) \cap \partitioni = \emptyset \wedge D_{\l}(p) \cap \FaceRegion(j,i) \neq \emptyset \wedge D_{\l}(p) \cap \partitionj \setminus \FaceRegion(i,j) = \emptyset\}$, which is necessarily a single point.
%
Resetting entity positions is a conservative approximation to the actual physical movement of entities.
If $j$ is the target matching the transferred entities' color, then the removed entities are not added to any cell and thus no longer exist in $\System$.
%

%\begin{figure}[t]
%	\centering
%	\hrule
%		\two{.47}{.47}
%			{\lstinputlisting[language=ioaNums,lastline=5]{code/move.hioa}}
%			{\lstinputlisting[language=ioaNumsRight,firstline=7,firstnumber=14]{code/move.hioa}}
%	%\begin{lstlisting}[language=ioa,firstline=1,numbers=left,stepnumber=2,numbersep=-8pt]
%	%\end{lstlisting}
%	\hrule
%	\caption{$\Move$ function for $\Celli$.}
%	% area of the package minus the next cell's area has some points, i.e., package does not lie completely on the cell
%	% intersection between line at safety distance in next cell and the unit vector orthogonal to the face being transitioned through is a point
%	\figlabel{movement}
%\end{figure}

\begin{figure}[t]
	\centering
	\mybox{\columnwidth}{\columnwidth}{
		%\twosep{.49}{.49}
		{
\begin{lstlisting}[language=ioaNums,numbersep=-3pt]
	let c = $\etypei$
	let $j$ = $\nextic$
	if $\neg \failedi$ /\ $\signalj$ = $i$ then
		for each p $\in$ $\Membersi$
			$p$ := $\pc$ + $\v \MoveVectorij$ (*@\lnlabel{algo:movement}@*)
			if $\pc \in \transferRegioni$ (*@\lnlabel{algo:gap}@*) then
				$\Membersi$ := $\Membersi$ $\setminus$ {p}
				if $j \neq \tidc$ (*@\lnlabel{algo:movementTB}@*) then
					$\Membersj$ := $\Membersj$ \cup {p} (*@\lnlabel{algo:reset}@*)
					$//$ $\mbox{point on shared side along movement vector}$
					$\pc$ := $(\pc + \v \MoveVectorij) \cap \Sideij$
					$//$ $\mbox{inner transfer region along orthogonal line}$
					$\pc$ := $(\pc + \SideNormalij) \cap \innerTransferRegionj(\Sideij)$ 
(*@%\taylor{we add the normal vector beyond the side length, need temp variable}
@*) (*@%					$\Mnej := \true$
@*)
\end{lstlisting}}%,lastline=7
		%{\lstinputlisting[language=ioaNumsRight,firstline=8,firstnumber=8,numbersep=-10pt]{code/move.hioa}}
	\centering
	\makebox[\columnwidth][l]{\hrulefill~~~~}
	\centering
	\parbox{0.95\columnwidth}{
		\caption{$\Move$ function for $\Celli$.  If $i$ has received a signal to move from $j$, it updates the positions of all entities on it to move in $j$'s direction, which may lead to some entities transferring from cell $i$ to $j$.}
		\figlabel{movement}
		}
	}
\end{figure}

The source cells $i \in \SID$, in addition to the above, add a finite number of entities in each round to $\Membersi$, such that the addition of these entities does not violate the minimum gap between entities at $\Celli$.
In the remainder of the paper, we will analyze $\System$ to show that in spite of failures, it maintains safety and liveness properties to be introduced in the next section.
\section{Analysis of Distributed Traffic Control}
\seclabel{analysis}
In this section, we present an analysis of the safety and liveness properties of $\System$.
Roughly, the safety property requires that there is a minimum gap between entities on any cell, and the liveness property requires that all entities that reside on cells with feasible paths to the corresponding target eventually reach that target.
\subsection{Safety and Collision Avoidance}
\sslabel{safety}
A state is safe if, for every cell, the boundaries of all entities in the cell are separated by a distance of $\rs$.
For any state $\vx$ of $\System$, we define:
\begin{align*}
\Safei (\vx)  &\deq \forall p, q \in \vx.\Membersi . p \neq q \Rightarrow \norm{ \pc - \qc} \geq 2 \l + \rs, \ \mbox{and} \\
\Safe(\vx) & \deq  \forall i \in \ID, \Safei(\vx).
\end{align*}
This definition allows entities in different cells to be closer than $2\l + \rs$ apart, but their centers will be spaced by at least $2\l$.
We proceed by proving some preliminary properties of $\System$ that will be used for proving $\Safe$ is an invariant.

The first property asserts that entities' cannot come close enough to the sides of cells to reside on multiple cells.
This is because any entity whose boundary touches the side of a cell is transferred to the neighboring cell on that side (if one exists), and then the entity's position is reset to be completely within the new cell.
%
%There are partitions that do, however,~\assref{transferFeasibility} ensures that this is not possible.
%
\assref{transferFeasibility} restricts the allowed partitions to ensure entity transfers are well-defined.
For instance, some of the cells in the snub square tiling in~\figref{transferSafety} do not satisfy~\assref{transferFeasibility}.
Consider an entity transfer from cell $3$ to cell $5$.
There is no constant vector connecting the transfer regions of cell $3$ to those of cell $5$.
This is because the side length of the transfer region of the triangular cell $5$ is shorter than the side length of the transfer region of the square cell $3$.
However, in a transfer from cell $1$ to cell $2$ or vice-versa, the side lengths are the same.
We also note that the assumption is only necessary for entity transfers from a cell with a longer transfer side length to a neighboring cell with smaller corresponding transfer side length.
For example, a transfer from cell $5$ to cell $3$ is feasible.

Under~\assref{transferFeasibility}, we have the following invariant, which states that the $\l$-ball around each entity in a cell is completely contained within the cell.
\begin{inv}
In any reachable state $\vx$, $\forall i \in \ID$, $\forall p \in \vx.\Membersi$, $D_{\l}(p) \setminus \partitioni = \emptyset$.
\invlabel{facetSpacing}
\end{inv}
%reaches a point beyond the boundary of the transfer region, $(p_x, p_y) \in \closure{\partitioni \setminus \transferRegioni}$, where $\closure{\cdot}$ is the closure.

The next invariant states that cells' $\Members$ sets are disjoint.
This is immediate from the $\Move$ function since entities are only added to one cell's $\Members$ upon being removed from a different cell's $\Members$.
\begin{inv}
In any reachable state $\vx$, for any $i, j \in \ID$, if $i \neq j$, then $\vx.\Membersi \cap \vx.\Membersj = \emptyset$.
\invlabel{disjoint}
\end{inv}

The following invariant states that cells contain entities of a single color in spite of failures.
This follows from the $\Signal$ routine in \figref{signal}, where \lnref{signalOneColor} requires that if some neighbor $j$ is attempting to move entities toward cell $i$, then the color of $i$ is either $\bot$ or equal to the color of $j$.
%
%We remark that this complicates the liveness analysis.
%
\begin{inv}
In any reachable state $\vx$, for all $i \in \ID$, for all $p, q \in \vx.\Membersi$, $\etypep = \etypeq$.
%
% (or equivalently, $\abs{\etypei} \leq 1$).
\invlabel{onecolor}
\end{inv}

Next, we define a predicate that states that if $\signali$ is set to the identifier of some neighbor $j \in \Nbrsi$, then there is a large enough area from the common side between $i$ and $j$ where no entities reside in $\Celli$.
Recall that $\Sideij$ is the line segment shared between neighboring cells $i$ and $j$.
For a state $\vx$, $H(\vx) \deq$ $\forall i \in \ID$, $\forall j \in \Nbrsi$, if $\vx.\signali = j$, then the following holds:
% for at most one $j \in \Nbrsi$:
%
\taylor{restate next in terms of safety region on $j$?}
\begin{align*}
\forall p \in \vx.\Membersi, \minel{\norm{ \pc - x }}{x \in \Sideij} \geq 3d \eqlabel{lnt}.
\end{align*}
$H(\vx)$ is not an invariant property because once entities move the property may be violated.
However, for proving safety, all that needs to be established is that at the {\em point of computation of the $\signal$ variable\/} this property holds.
The next key lemma states this.
\begin{lemma}
For all reachable states $\vx$, $H(\vx) \Rightarrow H(\vx_S)$ where $\vx_S$ is the state obtained by applying the $\Route$, $\Lock$, and $\Signal$ functions to $\vx$.
\lemlabel{safeSignal}
\end{lemma}
\begin{proof}
Fix a reachable state $\vx$, an $i \in \ID$, and an $j \in \Nbrsi$ such that $\vx.\signali = j$.
Let $\vx_R$ be the state obtained by applying the $\Route$ function to $\vx$, $\vx_L$ be the state obtained by applying the $\Lock$ function to $\vx_R$, and $\vx_S$ be the state obtained by applying the $\Signal$ function to $\vx_L$.

First, observe that both $H(\vx_R)$ and $H(\vx_L)$ hold.
This is because the $\Route$ and $\Lock$ functions do not change any of the variables involved in the definition of $H(\cdot)$.
Next, we show that $H(\vx_L)$ implies $H(\vx_S)$.
If $\vx_S.\signali \neq j$ then the statement holds vacuously.
Otherwise, $\vx_S.\signali = j$, then since (a) $H(\vx_L)$ holds, and (b)~\figref{signal},~\lnref{algo:gap} is satisfied, we have that $H(\vx_S)$.
\end{proof}
%next remark is when using linear vector fields
%todo: potential problem: since we allow failures, there could be a sequence of failures such that the next pointer changes many times while entities still reside on a cell. if this occurs, then the minimum distance between entities on the cell could decrease.  Need a statement like: if the next pointer changes, then the vector field in the new direction does not cause the entity spacing to decrease, or might need it to actually increase, and this may be a true statement based on how many concave triangles can be next to each other in a triangulation of the plane. This may also be an issue if next pointers change due to colors changing...
%solution one: assume constant vector fields
%The following lemma states that the spacing between entities does not decrease too far.

The following lemma asserts that if there is a cycle of length two formed by the $\signal$ variables---which could occur due to failures---then entity transfers cannot occur between the involved cells in that round.
\begin{lemma}
Let $\vx$ be any reachable state and $\vx'$ be a state that is reached from $\vx$ after a single $\act{update}$ transition (round).
If $\vx.\signali = j$ and $\vx.\signalj = i$, then $\vx.\Membersi = \vx'.\Membersi$ and $\vx.\Membersj = \vx'.\Membersj$.
\lemlabel{safeV}
\end{lemma}
\begin{proof}
No entities enter either $\vx'.\Membersi$ or $\vx'.\Membersj$ from any other $m \in \Nbrsi$ or $n \in \Nbrsj$ since $\vx.\signali = j$ and $\vx.\signalj = i$.
It remains to be established that $\nexists p \in \vx.\Membersj$ such that $p' \in \vx'.\Membersi$ where $p = p'$ or vice-versa.
Suppose such a transfer occurs.
For the transfer to have occurred, $\pc$ must be such that $\pc' = (p_x, p_y) + \v \MoveVectorij$ by~\figref{movement},~\lnref{algo:movement}.
But for $\vx.\signali = j$ to be satisfied, it must have been the case that $D_{\l}(p) \cap \partitioni = \emptyset$ by~\figref{signal},~\lnref{algo:gap} and since $\v < \l$, a contradiction is reached.
\end{proof}

Using the previous results, we now prove that $\System$ preserves safety even when some cells fail.
\begin{rtheorem}
In any reachable state $\vx$ of $\System$, $\Safe(\vx)$.
\thmlabel{safety}
\end{rtheorem}
\begin{proof}
The proof is by standard induction over the length of any execution of $\System$.
The base case is satisfied by the assumption that initial states $\vx \in Q_0$ satisfy $\Safe(\vx)$.
For the inductive step, consider any reachable states $\vx$, $\vx'$ and an action $a \in A$ such that $\vx \arrow{a} \vx'$.
Fix $i \in \ID$ and assuming $\Safei(\vx)$, we show that $\Safei(\vx')$.
If $a = \act{fail}_i$, then $\Safei(\vx')$ since no entities move.

For $a = \act{update}$, there are two cases to consider by~\invref{disjoint}.
First, $\vx'.\Membersi \subseteq \vx.\Membersi$, that is, no new entities were added to $i$, but some may have transfered off $i$.
There are two sub-cases: if $\vx'.\Membersi = \vx.\Membersi$, then all entities in $\vx.\Members$ move identically and the spacing between two distinct entities $p$, $q \in \vx'.\Membersi$ is unchanged.
Let $j = \nextic$ where $c = \etypei$ by~\invref{onecolor}.
That is, $\forall p, q \in \vx.\Membersi$, $\forall p', q' \in \vx'.\Membersi$ such that $p' = p$ and $q' = q$ and where $p \neq q$, $\norm{(p_x',p_y') - (q_x',q_y')} = \norm{(p_x,p_y) + \v \MoveVectorij, (q_x,q_y) + \v \MoveVectorij}$ (\figref{movement}, \lnref{algo:movement}).
It follows by the inductive hypothesis that $\norm{(p_x',p_y') - (q_x',q_y')} \geq d$.
The second sub-case arises if $\vx'.\Membersi \subsetneq \vx.\Membersi$, then $\Safei(\vx')$ is either vacuously satisfied or it is satisfied by the same argument just stated.

The second case is when $\vx'.\Membersi \nsubseteq \vx.\Membersi$, that is, there was at least one entity transfered to $i$.
Consider any such transferred entity $p' \in \vx'.\Membersi$ where $p' \notin \vx.\Membersi$.
There are two sub-cases.
The first sub-case is when $p'$ was added to $\vx'.\Membersi$ because $i$ is a source, that is, $i \in \SID$.
In this case, the specification of the source cells states that the entity $p'$ was added to $\vx'.\Membersi$ without violating $\Safei(\vx')$, and the proof is complete.
Otherwise, $p'$ was added to $\vx'.\Membersi$ by some neighbor $j \in \vx.\Nbrsi$, so $p' \in \vx.\Membersj$ but $p' \notin \vx.\Membersi$, and $p' \in \vx'.\Membersi$ but $p' \notin \vx'.\Membersj$.
From~\lnref{algo:reset} of~\figref{movement}, we have that that $(p'_x, p'_y) = \ResetEntity(p, i, j)$.
The fact that $p'$ transferred from $\Cellj$ in $\vx$ to $\Celli$ in $\vx'$ implies that $\vx.\nextj = i$ and $\vx.\signali = j$---these are necessary conditions for the transfer by~\figref{signal},~\lnref{signalSafe}.
Thus, applying the predicate $H(\vx)$ at state $\vx$ and by~\lmref{safeSignal}, it follows that for every $q \in \vx.\Membersi$, $(q_x, q_y) \notin \FaceRegion(i,j)$.
It must now be established that if $p'$ is transfered to $\vx'.\Membersi$, then every $q' \in \vx'.\Membersi$, where $q' \neq p'$ satisfies $(q'_x, q'_y) \notin \FaceRegion(i,j)$, which means that any entity $q$ already on $i$ did not move toward the transfered entity $p$ that is now on $i$.
This follows by application of~\lmref{safeV}, which states that if entities on adjacent cells move towards one another simultaneously, then a transfer of entities cannot occur.
This implies that the discs of all entities $q'$ in $\vx'.\Membersi$ are farther than $\rs$ of the borders of any transfered entity $p'$, implying $\Safei(\vx')$.
Finally, since $i$ was chosen arbitrarily, $\Safe(\vx')$.
\end{proof}
\thmref{safety} shows that $\System$ is safe in spite of failures.

\subsection{Stabilization of Spanning Routing Trees}
\sslabel{routing}
Next, we show under some additional assumptions, that once new failures cease to occur, $\System$ recovers to a state where each non-faulty cell with a feasible path to its target computes a route toward it.
This route stabilization is then used in showing that any entity on a non-faulty cell with a feasible path to its target makes progress toward it.
Our analysis relies on the following assumptions on cell failures and the placement of new entities on source cells.
The first assumption states that no target cells fail, and is reasonable and necessary because if any target cell did fail, entities of that color obviously cannot make progress.
\begin{assumption}
No target cells $t \in \TID$ may fail.
%
%$\auto{Cell}_{tid}$ does not fail, as otherwise it is possible that $dist_{tid} = \infty$, implying that $TC = \emptyset$ and all the following properties vacuously follow, and 
%
\asslabel{noTargetFailures}
\end{assumption}

\sayan{$\ast$: moved next assumption discussion out of statement}
The next assumption ensures source cells place entities fairly so that they may not perpetually prevent any neighboring cell or any color-shared cell from making progress.
The assumption is needed because it provides a specification of how the source cells behave, which has not been done so far.
The assumption is reasonable because it essentially says that traffic is not produced perpetually without any break.
\begin{assumption}
(Fairness): Source cells place new entities without perpetually blocking either 
\begin{inparaenum}[(i)]
\item any of their nonempty non-faulty neighbors, or
\item any cell $i \in \CSC(\vx, c)$, where $c$ is the color of source $s$.
\end{inparaenum}
\asslabel{fairness}
\end{assumption}
Formally, the first fairness condition states, for any execution $\alpha$ of $\System$, for any color $c \in C$, for any source cell $\sidc$, if there exists an $i \in \Nbrss$, such that for every state $\vx$ in $\alpha$ after a certain round, $i \in \vx.\NEPrevs$, then eventually $\signals$ becomes equal to $i$ in some round of $\alpha$.
The second fairness condition states, for any execution $\alpha$ of $\System$, for any state $\vx \in \alpha$, for any color $c \in C$, for any source cell $\sidc$, if there exists an $i \in \NF(\vx)$ such that $i \in \CSC(\vx, c)$, and for every state $\vx$ in $\alpha$ after a certain round, if cell $i$ is nonempty, then eventually $\signalj$ becomes equal to $i$ in some round of $\alpha$, where $j$ is a neighbor of $i$.
Such conditions can be ensured if we suppose some oracle placing entities on source cells follows the same round-robin like scheme defined in the $\Signal$ subroutine in~\figref{signal}.
Scenarios where each of these cases can arise are illustrated in~\figref{exampleFairness}.

A fault-free execution fragment $\alpha$ be a sequence of states starting from $\vx$ and along which no $\act{fail}(i)$ transitions occur.
That is, a fault-free execution fragment is an execution fragment with no new failure actions, although there may be existing failures at the first state $\vx$ of $\alpha$, so $F(\vx)$ need not be empty.
Throughout the remainder of this section, we will consider fault-free executions that satisfy~\assreftwo{noTargetFailures}{fairness}.

\begin{lemma}
\lemlabel{dist} % put at top to avoid weird vertical spacing bug
Consider any reachable state $\vx$ of $\System$, any color $c \in C$, and any $i \in \TC(\vx, c) \setminus \{\tidc\}$.
Let $h = \rhoc(\vx,i)$.
Any fault-free execution fragment $\alpha$ starting from $\vx$ stabilizes within $h$ rounds to a set of states $S$ with all elements satisfying:
\begin{align*}
	\distic & = h, \ \mbox{and}  \\
	\nextic & = i_n, \ \mbox{where} \ \rhoc(\vx, i_n) = h - 1.
\end{align*}
\end{lemma}
\begin{proof}
Fix an arbitrary state $\vx$, a fault-free execution fragment $\alpha$ starting from $\vx$, a color $c \in C$, and $i \in \TC(\vx, c) \setminus \{\tidc\}$.
We have to show that (a) the set of states $S$ is closed under $\act{update}$ transitions and (b) after $h$ rounds, the execution fragment $\alpha$ enters $S$.

First, by induction on $h$ we show that $S$ is stable.
Consider any state $\vy \in S$ and a state $\vy'$ that is obtained by applying an $\act{update}$ transition to $\vy$.
We have to show that $\vy' \in S$.
For the base case, $h = 1$, so $\vy.\distic = 1$ and $\vy.\nextic = \tidc$.
From~\lnreftwo{algo:distmin}{algo:nextmin} of the $\Route$ function in~\figref{route}, and that there is a unique $\tidc$ for each color $c$, it follows that $\vy'.\distic$ remains $1$ and $\vy'.\nextic$ remains $\tidc$.
For the inductive step, the inductive hypothesis is, for any given $h$, if for any $j \in \NF(\vx)$, $\vy.\distjc = h$ and $\vy.\nextjc = m$, for some $m \in \ID$ with $\rhoc(\vx,m) = h-1$, then 
\begin{align*}
\vy'.\distjc = h \ \mbox{and} \  \vy'.\nextjc = m.
\end{align*}
Now consider $i$ such that $\rhoc(\vy,i) = \rhoc(\vy',i) = h+1$.
In order to show that $S$ is closed, we have to assume that $\vy.\distic = h+1$ and $\vy.\nextic = m$, and show that the same holds for $\vy'$.
Since $\rhoc(\vy',i) = h+1$, $i$ does not have a neighbor with target distance smaller than $h$.
The required result follows from applying the inductive hypothesis to $m$ and from~\lnreftwo{algo:distmin}{algo:nextmin} of~\figref{route}.

Second, we have to show that starting from $\vx$, $\alpha$ enters $S$ within $h$ rounds.
Once again, this is established by induction on $h$, which is $\rhoc(\vx,i)$.
Consider any state $\vy$ such that $\rhoc(\vx,i) = \rhoc(\vy,i)$.
The base case only includes the target distances satisfying $h = \rhoc(\vy,i) = 1$ and follows by instantiating $i_n = \tidc$.
For the inductive case, assume for the inductive hypothesis that at some state $\vy$, $\vy.\distjc = h$ and $\vy.\nextjc = i_n$ such that $\rhoc(\vy, i_n) = h - 1$, where $i_n$ is the minimum identifier among all such cells (since we used cell identifiers to break ties).
Observe that there is one such $j \in \vy.\Nbrsi$ by the definition of $\TC$.
Then at state $\vy'$, by the inductive hypothesis and~\lnreftwo{algo:distmin}{algo:nextmin} of~\figref{route}, $\vy'.\distic = \vy'.\distjc + 1 = h + 1$.
\end{proof}

The following corollary of~\lmref{dist} states that, after new failures cease occurring, for all target-connected cells, the graph induced by the $\nextc$ variables stabilizes to the color $c$ routing graph, $\GR(\vx, c)$, within at most the diameter of the communication graph number of rounds, which is bounded by $\Delta(\vx)$.
%
%This follows since the value of $h$ in~\lmref{dist} for any target-connected cell is at most $O(N)$.
%
\begin{corollary}
Consider any execution $\alpha$ of $\System$ with an arbitrary but finite sequence of $\act{fail}$ transitions.
For any state $\vx \in \alpha$ at least $2 \Delta(\vx)$ rounds after the last $\act{fail}$ transition, for any $c \in C$, every cell $i$ target-connected to color $c$ has $\vx.\nextic$ equal to the identifier of the next cell along such a route.
\corlabel{next}
\end{corollary}

%The next lemma states that after routes have stabilized, only the target, call it $t$, of color $c$ may have three $\next$ variables of color $c$ equal to $t$, and any other cell $i$ of color $c$ may have at most two $\next$ variables of color $c$ equal to $i$.
%%
%The lemma obviously generalizes to other planar graphs, replacing three, and respectively two, with the maximum degree $\Delta(\vx)$ of the graph, and respectively $\Delta - 1$.
%%
%\begin{lemma}
%%
%For any state $\vx$ after route stabilization, for any $c \in C$, for any $i \in \NF(\vx) \setminus \TID$, for all distinct $m, n, p \in \ID$, if $\vx.\next_m(c) = i$ and $\vx.\next_n(c) = i$, then $\vx.\next_p(c) \neq i$.
%%
%\end{lemma}

%
The following corollary of~\lemref{dist} and~\corref{next} states that within $2\Delta(\vx)$ rounds after routes stabilize, for each color $c \in C$, the identifiers in the $\pathic$ variables equal the vertices of the color $c$ entity graph $\GE(\vx, c)$.
The result follows since routes stabilize and that $\Lock$ is a function of $\next$ and $\path$ variables only, and that $\pathi$ variables are gossiped in~\figref{lock},~\lnref{pathgossip}.
\begin{corollary}
Consider any execution $\alpha$ of $\System$ with an arbitrary but finite sequence of $\act{fail}$ transitions.
For any state $\vx \in \alpha$ at least $2 \Delta(\vx)$ rounds after the last $\act{fail}$ transition, for every $c \in C$, every cell $i$ target-connected to color $c$ has $\pathic = \VE(\vx, c)$.
\corlabel{pathGlobal}
\end{corollary}

The next corollary of~\lmref{dist} states that eventually the values of the $\pintc$ variables equal the set of color-shared cells $\CSC(\vx, c)$ for any cell $i$ and color $c$.
This is important because the mutual exclusion algorithm is initiated between the cells in $\pintc$ (\figref{lock},~\lnref{pathmutex}).
\begin{corollary}
Consider any execution $\alpha$ of $\System$ with an arbitrary but finite sequence of $\act{fail}$ transitions.
For any state $\vx \in \alpha$ at least $2 \Delta(\vx)$ rounds after the last $\act{fail}$ transition, for every $c \in C$, every cell $i$ target-connected to color $c$ has $\vx.\pintc = \CSC(\vx, c)$.
\corlabel{pintGlobal}
\end{corollary}
\subsection{Scheduling Entities through Color-Shared Cells}
\sslabel{lock}
In this section, we show that there is at most a single color on the set of color-shared cells if there are no failures.
We then show that any cell that requests a lock eventually gets one, under an additional assumption that failures do not cause entities of more than one color to reside on the set of color-shared cells.
Because failures cause the routing graphs and entity graphs to change, the color-shared cells that could previously be scheduled may now be deadlocked.
Additionally, because we separately lock each disjoint set of color-shared cells to allow entities of some color to flow toward their target, it could be the case that the intermediate states between when the failure occurred and when routes have stabilized allowed entities to move in such a way that deadlocks the system.
Such deadlocks could be avoided if a centralized coordinator informs every non-faulty cell to disable their signals when a failure is detected.
\taylor{It is not sufficient for the cells immediately around a failure to disable their signals, see the one-lane bridge dynamic failure example (i.e., even if the cells around a newly failed cell disable their signals, a deadlock might be caused by a cell far away moving an entity onto the (new) set of color-shared cells).}
The assumption states that with failures, the color-shared cells either all have the same-colored entities, or have no entities (and combinations thereof).

\begin{assumption}
\emph{Feasibility of Locking after Failures}: For any reachable state $\vx$, for any color $c \in C$, consider the color-shared cells $\CSC(\vx, c)$.
%Consider any reachable state $\vx$ after routes (by~\corref{next}) and color-shared cells (by~\corref{pathGlobal}) stabilize, for any color $c \in C$, consider the color-shared cells $\CSC(\vx, c)$.
%
For all distinct cells $i, j \in \CSC(\vx, c)$ either $\vx.\etypei = \vx.\etypej$ or $\vx.\etypei = \bot$.
\asslabel{feasibleLocking}
\end{assumption}

The next lemma states that without failures, there are entities of at most a single color on the set of color-shared cells.
The result is not an invariant because failures may cause the set of color-shared cells to change, resulting in deadlocks, which is why we need~\assref{feasibleLocking}.
By~\invref{onecolor}, we know that there are entities of at most a single color in each cell, so the following invariant is stated in terms of the color $\etypei$ of each cell.
We emphasize that~\assref{feasibleLocking} is unnecessary if there are no failures, as the algorithm ensures there are entities of at most a single color on the color-shared cells by the following lemma.
\begin{lemma}
%entity version
%For any reachable state $\vx$, for any $c \in C$, for any $i \in \CSC(\vx, c)$, if $\vx.\lockic = \false$, then for all $j \in \CSC(\vx, c)$, for all $p \in \vx.\Membersj$, we have $\etypep \neq c$.
%
If there are no failures, for any reachable state $\vx$, for any $c \in C$, for any $i \in \CSC(\vx, c)$, if $\neg \vx.\lockic$, then for all $j \in \CSC(\vx, c)$, we have $\vx.\etypej \neq c$.
\lemlabel{singleColorShared}
\end{lemma}
\begin{proof}
%proof by inductive invariance
%
The proof is showing an inductive invariant, supposing no failures occur.
For the initial state, all cells are empty, so we have $\vx.\etypei = \bot$ for any $i \in \ID$.
%
%For the inductive step, only the $\act{update}$ action modifies $\etype$.
For the inductive step, we are only considering $\act{update}$ actions by assumption.
%\taylor{say $\CSC(\vx, c) = \CSC(\vx', c)$}
%
In the pre-state, we have $\neg \vx.\lockic$ and $\forall j \in \CSC(\vx, c)$, we have $\vx.\etypej \neq c$.
Fix some $c \in C$ and some $i \in \CSC(\vx, c)$.
For any subsequent state $\vx'$, if $\vx'.\lockic$, the result follows vacuously.
If $\neg \vx'.\lockic$, we must show $\forall j \in \CSC(\vx, c)$ that $\vx.\etypej \neq c$, so fix some $j \in \CSC(\vx, c)$.
If $j \in \CSC(\vx', c)$, the result follows, since by the inductive hypothesis, $\vx.\etypej = \vx'.\etypej \neq c$.
If $j \notin \CSC(\vx', c)$, the condition in $\Signal$ (\figref{signal},~\lnref{signalLock}) cannot be satisfied since $\neg \vx'.\lockic$.
Thus, no cell with entities of color $c$ could move toward any cell in $\CSC(\vx', c)$, and we have $\vx'.\etypej \neq c$.
%\taylor{check proof}
%
%proof by contradiction:
%Fix some reachable state $\vx$, some color $c$, and some $i \in \CSC(\vx, c)$, and suppose there is some cell $j \in \CSC(\vx, c)$ such that $\etypej = c$.
%
%Since $\vx.\etypej = c$, we know there is some $p \in \vx.\Membersj$ with $\etypep = c$.%
%
%Thus, we know that at some prior state $\vx'$ in the prefix of any execution reaching $\vx$, the condition in $\Signal$ (\figref{signal},~\lnref{signalLock}) was satisfied, and thus $\vx'.\lockjc = \true$.
%
%However, this contradicts $\vx'.\lockic = \false$.
%\taylor{bug in this proof, we only know $\vx.\lockic$ is false, not that $\vx'.\lockic = \false$}
%
\end{proof}

%
%\begin{inv}
%%
%For any reachable state $\vx$, for any $c \in C$, for any $i \in \CSC(\vx, c)$, if $\vx.\etypei \neq \bot$, then $\vx.\etypei = c$.
%%
%\taylor{Todo: change to say that if a color has the lock, then all the cells on the color-shared cells are of this color? We basically want to say that for a cell to move entities of color $c$ toward the color-shared cells, the color-shared cells must be the same color.  If $\signalj = i$ and $j \in \CSC(\vx, c)$, then for all cells $k \in \CSC(\vx, c)$, we have $\vx.\etypei = \vx.\etypej$.  Maybe if lock, then all the same is the way to go.}
%\invlabel{singleColorShared}
%%
%\end{inv}
%
%\begin{proof}
%%
%Let $k$ be the index of some state, $\vx_k$.
%%
%If $k < 2 \Delta$ rounds of the last $\act{fail}$ transition, then the result follows by~\assref{feasibleLocking}.
%%
%If $k \geq 2 \Delta$, then by~\corref{pintGlobal}, we know that for any color $c \in C$, $\CSC(\vx_k, c) = \pintc$.
%%
%Next, we suppose any mutual exclusion algorithm initiated in~\figref{lock},~\lnref{pathmutex} satisfies the standard safety property that it grants a lock to at most one element in its set of input values.
%%
%Here, this means that if it returns, then it returns at most a single colors involved in each set of color-shared cells.
%
%Call this color $d$.
%%
%There are two cases.
%%
%If $c = d = \vx_k.\etypei$, then we are done.
%%
%Otherwise, if $d \neq c$, by~\figref{signal},~\lnref{signalLock}, then $\signali \neq j$.
%\taylor{some problems in previous---have to explicitly do returning the lock, etc.}
%%
%\end{proof}

The next lemma states that without failures, or with ``nice'' failures as described by~\assref{feasibleLocking}, that any cell requesting a lock of some color will eventually get it, and thus it may move entities onto the color-shared cells.
\begin{lemma}
For any reachable state $\vx$ satisfying~\assref{feasibleLocking}, for any $c \in C$, for any $i \in \NF(\vx)$, if $i \in \vx.\pintc$ and all cells in $\CSC(\vx, c)$ are empty, then eventually a state $\vx'$ is reached where $\vx'.\lockic$.
\lemlabel{eventualLocking}
\end{lemma}
\begin{proof}
%
%By~\lemref{singleColorShared} and~\assref{feasibleLocking}, there is at most a single color on the color-shared cells $\CSC(\vx, c)$, for any $c \in C$.
%
By correctness of the mutual exclusion algorithm, eventually a color $d \in \SC(\vx', c)$ is returned and $\vx'.\lockid = \true$ (\figref{lock},~\lnref{pathmutex}).
If $c = d$, then the result follows.
If $c \neq d$, by~\lemref{singleColorShared} and~\assref{feasibleLocking}, we know that no other color aside from $c$ has entities on any cell $j \in \CSC(\vx', c)$.
The next time the mutual exclusion algorithm is initiated, $d$ is excluded from the input set to the mutual exclusion algorithm (\figref{lock},~\lnref{pathempty}), and by repeated argument, eventually $\lockic$.
\end{proof}
\subsection{Progress of Entities towards their Targets}
\sslabel{movement}
Using the results from the previous sections, we show that once new failures cease occurring, for every color $c \in C$, every entity of color $c$ on a cell that is target-connected eventually gets to the target of color $c$.
The result (\thmref{progress}) uses two lemmas which establish that, along every infinite execution with a finite number of failures, every nonempty target-connected cell gets permission to move infinitely often (\lmref{signal}), and a permission to move allows the entities on a cell to make progress towards the target (\lmref{move}).

For the remainder of this section, we fix an arbitrary infinite execution $\alpha$ of $\System$ with a finite number of failures, satisfying~\assref{feasibleLocking}.
Let $\vx_f$ be any state of $\System$ at least $2 \Delta(\vx)$ rounds after the last failure, and $\alpha'$ be the infinite failure-free execution fragment $\vx_f$, $\vx_{f+1}$, $\ldots$ of $\alpha$ starting from $\vx_f$.
For any $c \in C$, observe that the number of target-connected cells remains constant starting from $\vx_f$ for the remainder of the execution.
That is, $\TC(\vx_f, c) = \TC(\vx_{f+1}, c) = \TC(\ldots, c)$, so we fix $\TC(c) = \TC(\vx_f, c)$.

%\taylor{The next lemma talks about the distance of one entity from its target.  We actually didn't use a global ranking function.  We used a combination of the following lemma, which says that if an entity moves, it makes progress to the target, and the next lemma, that every entity gets to move infinitely often, to establish liveness.  We could define one, but I think it will be very complicated.  We can use something more like a ranking function as the sum of all these distance functions, then say how this decreases infinitely often.  A proper ranking ranking function would have to involve the token alternation for setting signal, which is where the infinitely often part comes from.}
%
\begin{lemma}
For any $c \in C$, for any $i \in \TC(c)$, for some $j \in \vx_f.\Nbrsi$, if $k > f$, $\vx_k.\signalj = i$, and $\vx_k.\nextic = j$, for any entity $p \in \vx_k.\Membersi$, let the distance function be defined by the lexicographically ordered tuple 
%
%For any $c \in C$, for any $i \in \TC(c)$, for some $j \in \vx_f.\Nbrsi$ such that $\vx_f.\nextic = j$, for any entity $p \in \vx_f.\Membersi$, let the distance function be defined by the lexicographically ordered tuple 
%
\begin{align*}
R(\vx, p) = \pair{\rhoc(\vx, i), ds - \pc},
\end{align*}
where $ds$ is the point on the shared side $\Sideij$ defined by the line passing through $\pc$ with direction $\MoveVectorij$.
Then, $R(\vx_{k+1}, p) < R(\vx_{k}, p)$.
\lemlabel{move}
\end{lemma}
\begin{proof}
The first case is when no entity transfers from $i$ to $j$ in the $k+1^{th}$ round: if $p' \in \vx_{k+1}.\Membersi$ such that $p' = p$, then $\norm{ds - \pc'} < \norm{ds - \pc}$.
In this case, the result follows since a velocity $\v > 0$ is applied towards cell $j$ by $\Move$ in~\figref{movement},~\lnref{algo:movement}.
The second case is when some entity $p$ transfers from $i$ to $j$, so $p' \in \vx_{k+1}.\Membersj$ such that $p' = p$.
In this case, we have $\rhoc(\vx_k,j) < \rhoc(\vx_k,i)$, since the distance between $j$ and $\tidc$ is smaller than the distance between $i$ and $\tidc$ since routes have stabilized by~\lemref{dist}.
In either case, $R(\vx_{k+1}, p) < R(\vx_k, p)$, so entity $p$ is closer to the appropriate target.
\end{proof}

The following lemma states that all cells with a path to the target receive a signal to move infinitely often, so~\lemref{move} applies infinitely often.
\begin{lemma}
For any $c \in C$, consider any $i \in \TC(c) \setminus \tidc$, such that for all $k > f$, if $\vx_k.\Membersi \neq \emptyset$, then $\exists k' > k$ such that $\vx_{k'}.\signal_{\nextic} = i$.
\lemlabel{signal}
\end{lemma}
\begin{proof}
Fix some $c \in C$.
Since $i \in \TC(c)$, there exists $h < \infty$ such that for all $k > f$, $\rhoc(\vx_k,i) = h$.
We prove the lemma by inducting on $h$.
The base case is $h=1$.
Fix $i$ and instantiate $k' = f + \ns{\tidc}$.
By~\lmref{dist}, for any $t \in \TID$, for all non-faulty $i \in \Nbrs_{t}$, $\vx_{f}.\nextic = t$ since $k > f$.
For all $k > f$, if $\vx_k.\Membersi \neq \emptyset$, then $\signal_{\tidc}$ changes to a different neighbor with entities every round.
It is thus the case that $\abs{\vx_k.\NEPrev_{\tidc}} \leq \ns{\tidc}$ and since $\Members_{\tidc} = \emptyset$ always, exactly one neighbor satisfies the conditional of~\figref{signal},~\lnref{algo:gap} in any round, then within $\ns{\tidc}$ rounds, $\signal_{\tidc} = i$.

For the inductive case, let $k_s = k + h$ be the step in $\alpha$ after which all non-faulty $a \in \Nbrsi$ have $\vx_{k_s}.\next_a\colorc = i$ by~\lemref{dist}.
Also by~\lemref{dist}, $\exists m \in \Nbrsi$ such that $\vx_{k_s}.\dist_m < \vx_{k_s}.\disti$, implying that after $k_s$, $\abs{\vx_{k_s}.\NEPrevi} \leq \nsi$ since $\vx_{k_s}.\nexti = m$ and $\vx_{k_s}.\next_m \neq i$.
By the inductive hypothesis, $\vx_{k_s}.\signal_{\nextic} = i$ infinitely often.
If $i \in \SID$, then entity initialization does not prevent $\vx_k.\signali = a$ from being satisfied infinitely often by the second assumption introduced in~\ssref{routing}.
It remains to be established that $\signali = a$ infinitely often.
Let $a \in \vx_{k_s}.\NEPrevi$ where $\rhoc(\vx_{k_s}, a) = h+1$.

In any of the following cases, if $i \in \vx_{k_s}.\pintc$ and all cells $j \in \CSC(\vx_{k_s}, c)$ are empty, then by \lemref{eventualLocking}, eventually $\lockic$.
%\taylor{add more details...need if also empty}
%
If $\abs{\vx_{k_s}.\NEPrevi} = 1$, then since the inductive hypothesis satisfies $\signal_{\nextic} = i$ infinitely often, then~\lmref{move} applies infinitely often, and thus $\Membersi = \emptyset$ infinitely often, finally implying that $\signali = a$ infinitely often.

If $\abs{\vx_{k_s}.\NEPrevi} > 1$, there are two sub-cases.
The first sub-case is when no entity enters $i$ from some $d \neq a \in \vx_{k_s}.\NEPrevi$, which follows by the same reasoning used in the $\abs{\vx_{k_s}.\NEPrevi} = 1$ case.
The second sub-case is when a entity enters $i$ from $d$, in which case it must be established that $\signali = a$ infinitely often.
This follows since if $\vx_{k'}.\tokeni = a$ where $k' > k_t > k_s$ and $k_t$ is the round at which an entity entered $i$ from $d$, and the appropriate case of~\lemref{safeSignal} is not satisfied, then $\vx_{k'+1}.\signali = \bot$ and $\vx_{k'+1}.\tokeni = a$ by~\figref{signal},~\lnref{algo:tokenSame}.
This implies that no more entities enter $i$ from either cell $d$ satisfying $d \neq a$.
Thus $\tokeni = a$ infinitely often follows by the same reasoning $\abs{\vx_{k_s}.\NEPrevi} = 1$ case.
\end{proof}

The final theorem establishes that entities on any cell in $\TC(c)$ eventually reach the target in $\alpha'$.
\begin{rtheorem}
For any $c \in C$, consider any $i \in \TC(c)$, $\forall k > f$, $\forall p \in \vx_k.\Membersi$, $\exists k' > k$ such that $p \in \vx_{k'}.\Members_{\nextic}$.
\thmlabel{progress}
\end{rtheorem}
\begin{proof}
Fix $c \in C$, $i \in \TC(c)$, a round $k > f$ and $p \in \vx_k.\Membersi$.
Let $h = \max_{i \in \TC(c)} \rhoc(\vx_f, i)$ which is finite.
By~\lmref{dist}, at every round after $k_s = k + h$ for any $i \in \TC(c)$, the sequence of identifiers $\beta = i$, $\vx_{k_s}.\nextic$, $\vx_{k_s}.\next_{\nextic}\colorc$, $\ldots$ forms a fixed path to $\tidc$.
Applying~\lmref{signal} to $i \in \TC(c)$ shows that there exists $k_m \geq k_s$ such that $\vx_{k_m}.\signal_{\nextic} = i$.
Now applying~\lmref{move} to $\vx_{k_m}$ establishes movement of $p$ towards $\vx_{k_s}.\nextic$, which is also $\vx_{k_m}.\nextic$.
\lmref{signal} further establishes that this occurs infinitely often, thus there is a round $k' > k_m$ such that $p$ gets transferred to $\vx_{k_m}.\Members_{\nextic}$.
\end{proof}
By an induction on the sequence of identifiers in the path $\beta$, it follows that entities on any cell in $\TC(c)$ eventually get consumed by the target.

\subsubsection*{Summary of Results}
In this section, we establish several invariant properties culminating in proving safety of the system, which meant that entities never collide, in spite of failures.
Next, we proved that the routing algorithm used to construct paths to the destinations is self-stabilizing in spite of arbitrary crash failures.
We next showed under an assumption that failures do not introduce deadlock scenarios that the locking algorithm allows multi-color flows to mutual exclusively take control of intersections (color-shared cells).
Finally, under a fairness assumption, we established the main progress property through two results, that any cell gets permission to move infinitely often, and that any cell with a permission to move decreases the distance of any entities on it from its destination.
\section{Simulation Experiments}
\seclabel{sim}
We have performed several simulation studies of the algorithm for evaluating its throughput performance.
In this section, we discuss the main findings with illustrative examples taken from the simulation results.
We implemented the simulator in Matlab, and all the partition figures displayed in the paper are created using it.

\begin{figure}[t]
	\centering%
	\begin{minipage}[t]{0.475\linewidth}
		\centering
		\includegraphics[width=\columnwidth]{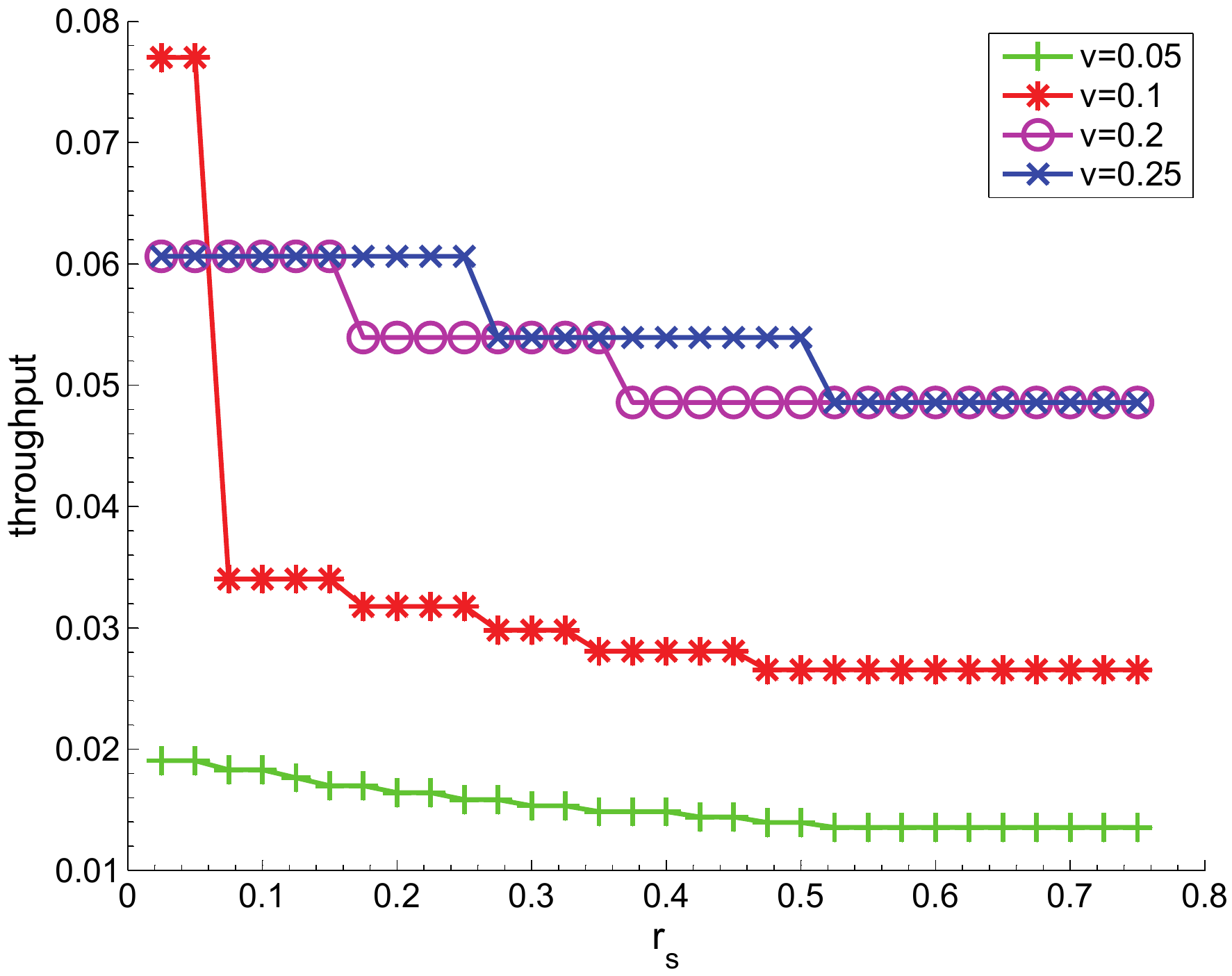}
		\caption{Throughput versus safety spacing $\rs$ for several values of $\v$, for $K = 2500$, $\l = 0.25$ for $\System$ with an $8 \times 8$ unit square tessellation.}
		\figlabel{rsplot}
	\end{minipage}
	\hspace{0.01\linewidth}
	\begin{minipage}[t]{0.475\linewidth}
		\centering
		\includegraphics[width=\columnwidth]{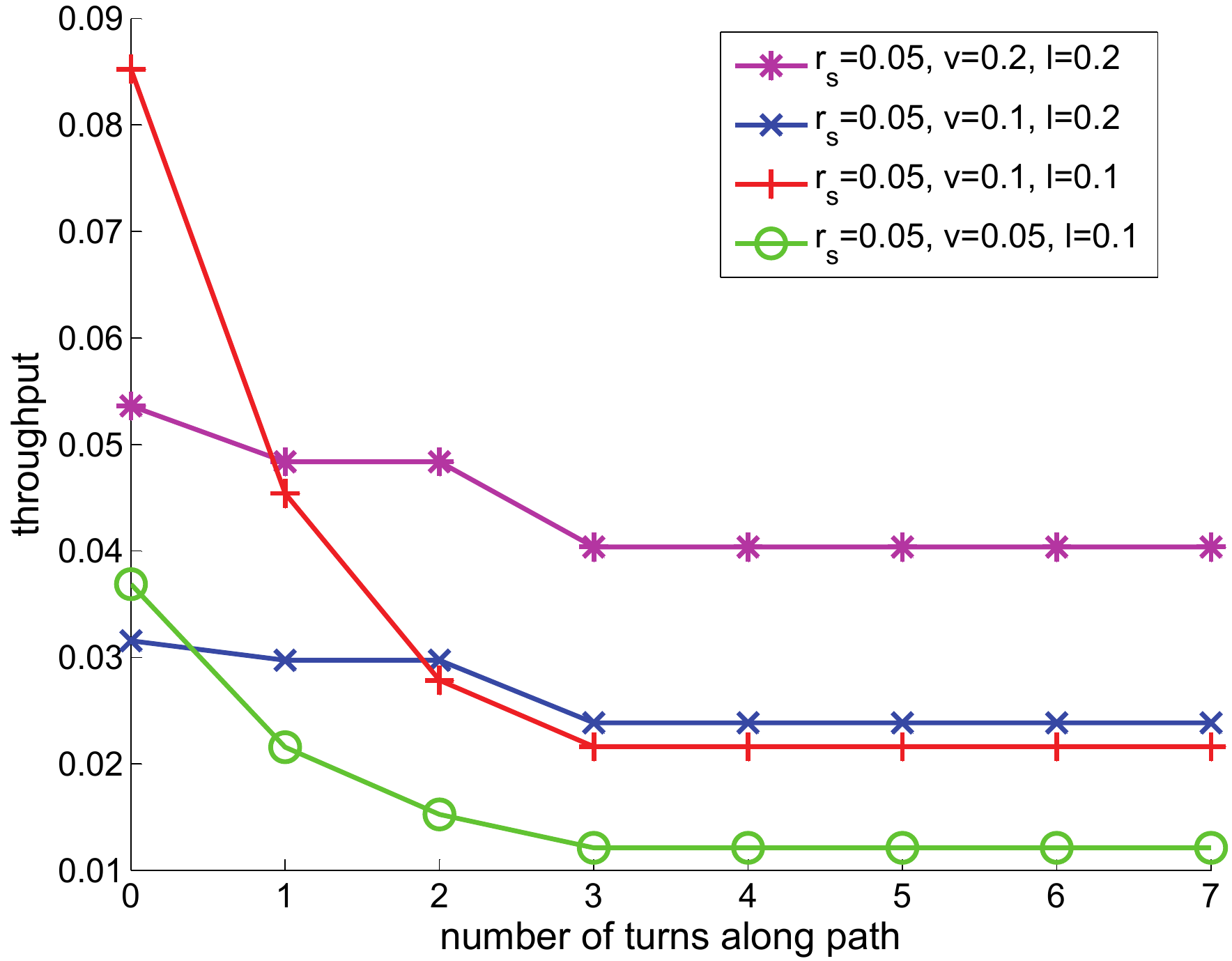}
		\caption{Throughput versus number of turns along a path, for a path of length $8$, where $K = 2500$, $\rs = 0.05$, and each of $\l$ and $\v$ are varied for $\System$ with an $8 \times 8$ unit square tessellation.}
		\figlabel{turnsplot}
	\end{minipage}
\end{figure}

Let the {\em $K$-round throughput\/} of $\System$ be the total number of entities arriving at the target over $K$ rounds, divided by $K$.
We define the {\em average throughput\/} (henceforth throughput) as the limit of $K$-round throughput for large $K$.
All simulations start at a state where all cells are empty and subsequently entities are added to the source cells.

\paragraph*{Single-color throughput without failures as a function of $\rs$, $\l$, $\v$}
Rough calculations show that throughput should be proportional to cell velocity $\v$, and inversely proportional to safety distance $\rs$ and entity radius $\l$. 
\figref{rsplot} shows throughput versus $\rs$ for several choices of $\v$ for an $8 \times 8$ unit square tessellation instance of $\System$ with a single entity color.
The parameters are set to $\l=0.25$ and $K=2500$.
%
%The entities move along the path $\beta \deq \pair{1,0}, \pair{1,1}, \pair{1,2}, \pair{1,3}, \pair{1,4}, \pair{1,5}, \pair{1,6}, \pair{1,7}$ with length $8$.
%
The entities move along a line path where the source is the bottom left corner cell and the target is the top left corner cell.
For the most part, the inverse relationship with $v$ holds as expected: all other factors remaining the same, a lower velocity makes each entity take longer to move away from the boundary, which causes the predecessor cell to be blocked more frequently, and thus fewer entities reach $\tid$ from any element of $\SID$ in the same number of rounds.
In cases with low velocity (for example $\v = 0.1$) and for very small $\rs$, however, the throughput can actually be greater than that at a slightly higher velocity.
We conjecture that this somewhat surprising effect appears because at very small safety spacing, the potential for safety violation is higher with faster speeds, and therefore there are many more blocked cells per round.
We also observe that the throughput saturates at a certain value of $\rs$ ($\approx 0.55$).
This situation arises when there is roughly only one entity in each cell.

\paragraph*{Single-color throughput without failures as a function of the path}
For a sufficiently large number of rounds $K$, throughput is independent of the length of the path.
This of course varies based on the particular path and instance of $\System$ considered, but all other variables fixed, this relationship is observed.
More interesting however, is the relationship between throughput and path complexity, measured in the number of turns along a path.
\figref{turnsplot} shows throughput versus the number of turns along paths of length $8$.
This illustrates that throughput decreases as the number of turns increases, up to a point at which the decrease in throughput saturates.
This saturation is due to signaling and indicates that there is only one entity per cell.

\paragraph*{Single-color throughput under failure and recovery of cells}
Finally, we considered a random failure and recovery model in which at each round each non-faulty cell fails with some probability $p_f$ and each faulty cell recovers with some probability $p_r$~\cite{deville2009sss}.
A \emph{recovery} sets $\failedi = \false$ and in the case of $\tid$ also resets $\dist_{\tid} = 0$, so that eventually $\Route$ will correct $\nextj$ and $\distj$ for any $j \in \TC$.
Intuitively, we expect that throughput will decrease as $p_f$ increases and increase as $p_r$ increases.
\figref{singleTargetFailPlot} demonstrates this result for $0.01 \leq p_f \leq 0.05$ and $0.05 \leq p_r \leq 0.2$.
There is a diminishing return on increasing $p_r$ for a fixed $p_f$, in that for a fixed $p_f$ increasing $p_r$ results in smaller throughput gains.

\begin{figure}[t]
	\centering%
	\begin{minipage}[t]{0.475\linewidth}
		\centering
		\includegraphics[width=\columnwidth]{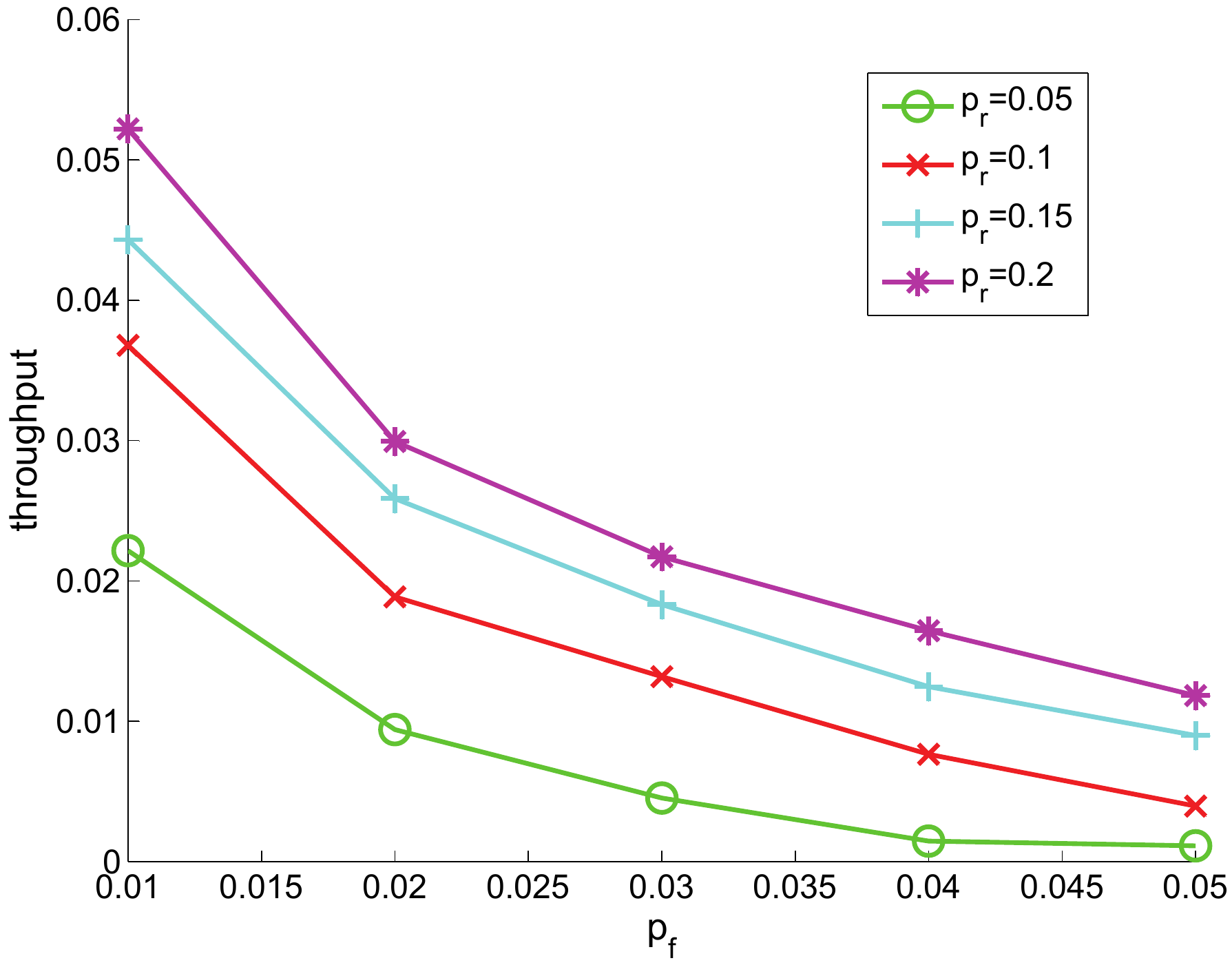}
		\caption{Throughput versus failure rate $p_f$ for several recovery rates $p_r$ with an initial path of length $8$, where $K = 20000$, $\rs = 0.05$, $\l = 0.2$, and $\v = 0.2$ for $\System$ with an $8 \times 8$ unit square tessellation.}
		\figlabel{singleTargetFailPlot}
	\end{minipage}
	\hspace{0.01\linewidth}
	\begin{minipage}[t]{0.475\linewidth}
		\centering
		\includegraphics[width=\columnwidth]{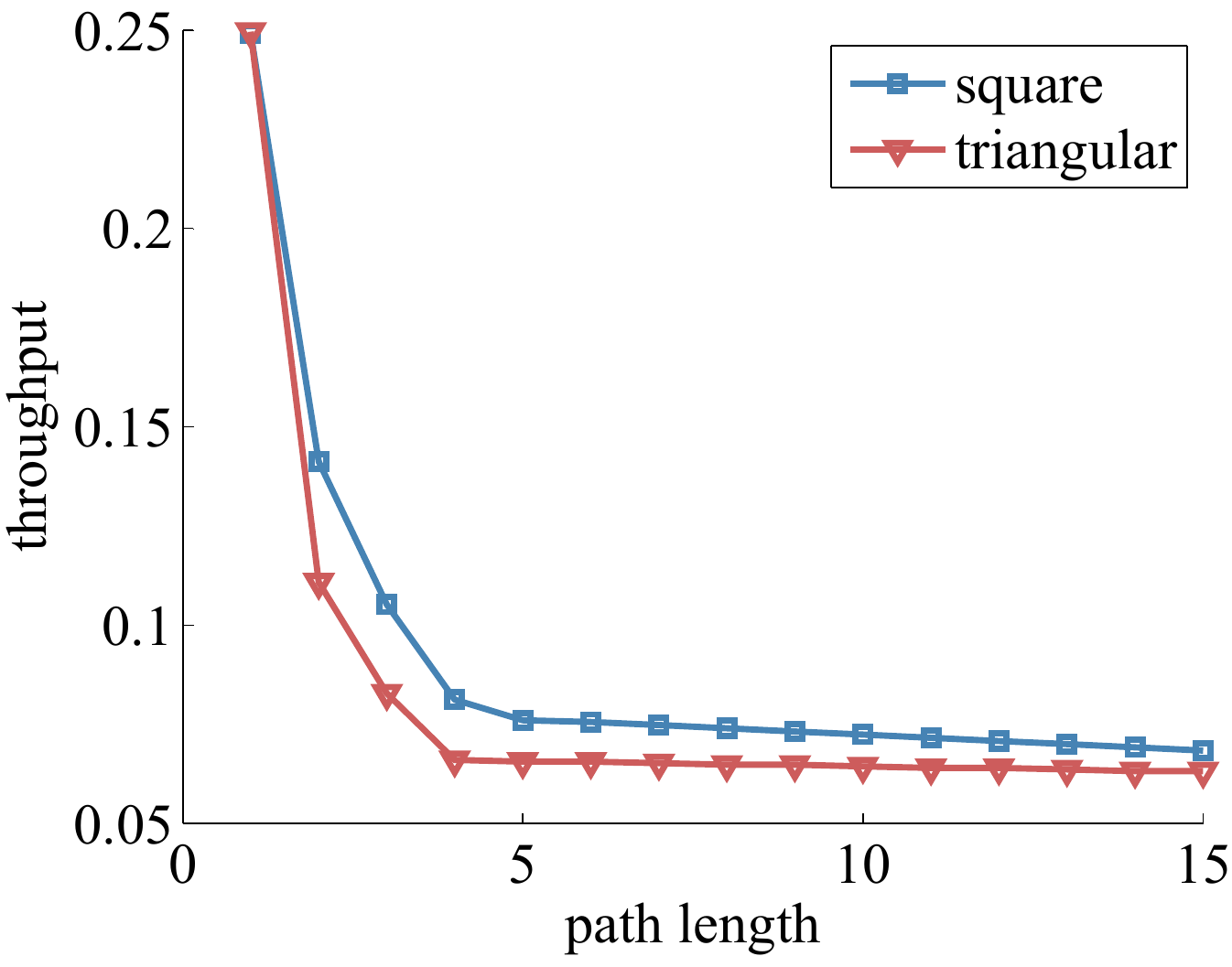}
		\caption{Throughput versus increasing path length of square (blue) and equilateral triangular (red) partitions.}
		\figlabel{pathLength}
	\end{minipage}
\end{figure}

\paragraph*{Multi-color throughput as a function of the number of intersecting cells}
Now we discuss the influence of multi-color throughput.
In the case where the paths between different sources and targets do not overlap, all the results from the single-color simulation results apply.
In the case where the paths do overlap, the mutual exclusion algorithm runs to ensure no deadlocks occur.
This additional control logic will have an influence on the throughput.
For the multi-color cases, we consider the summed throughput, which is the sum of the throughputs for each color.

\figref{throughputPathOverlap} shows the roughly exponential decrease in throughput as the fraction of overlapping paths increases for two colors with path length $8$ and no turns.
The fraction of overlapping paths is defined as the number of vertices in the color-shared cells $\CSC(\vx, c)$.
As the fraction increases, the paths lie completely on top of one another, so in this case with path length $8$, we have no overlap, $1$ cell overlap, etc.

\paragraph*{Multi-color throughput as a function of the number of intersecting colors}
Intersections (that is, having at least one color-shared cell) have a fixed cost on throughput.
Specifically, the summed throughput of there being two overlapping colors on a cell is the same as the summed throughput of three or more.
\figref{throughputColorOverlap} shows this fixed decrease in throughput as the number of overlapping colors increases for a fixed path of length $3$ with $3$ color-shared cells, where the decrease in throughput from having no overlaps to having one color overlapping is about $4.5$ times.
Once there are two colors, all additional colors do not decrease throughput.
This observation agrees with intuition---the decrease in throughput due to an intersection is independent of the number of destinations for the entities that must pass through that intersection.

\begin{figure}[t]
	\centering%
	\begin{minipage}[t]{0.475\linewidth}
		\centering
		\includegraphics[width=\columnwidth]{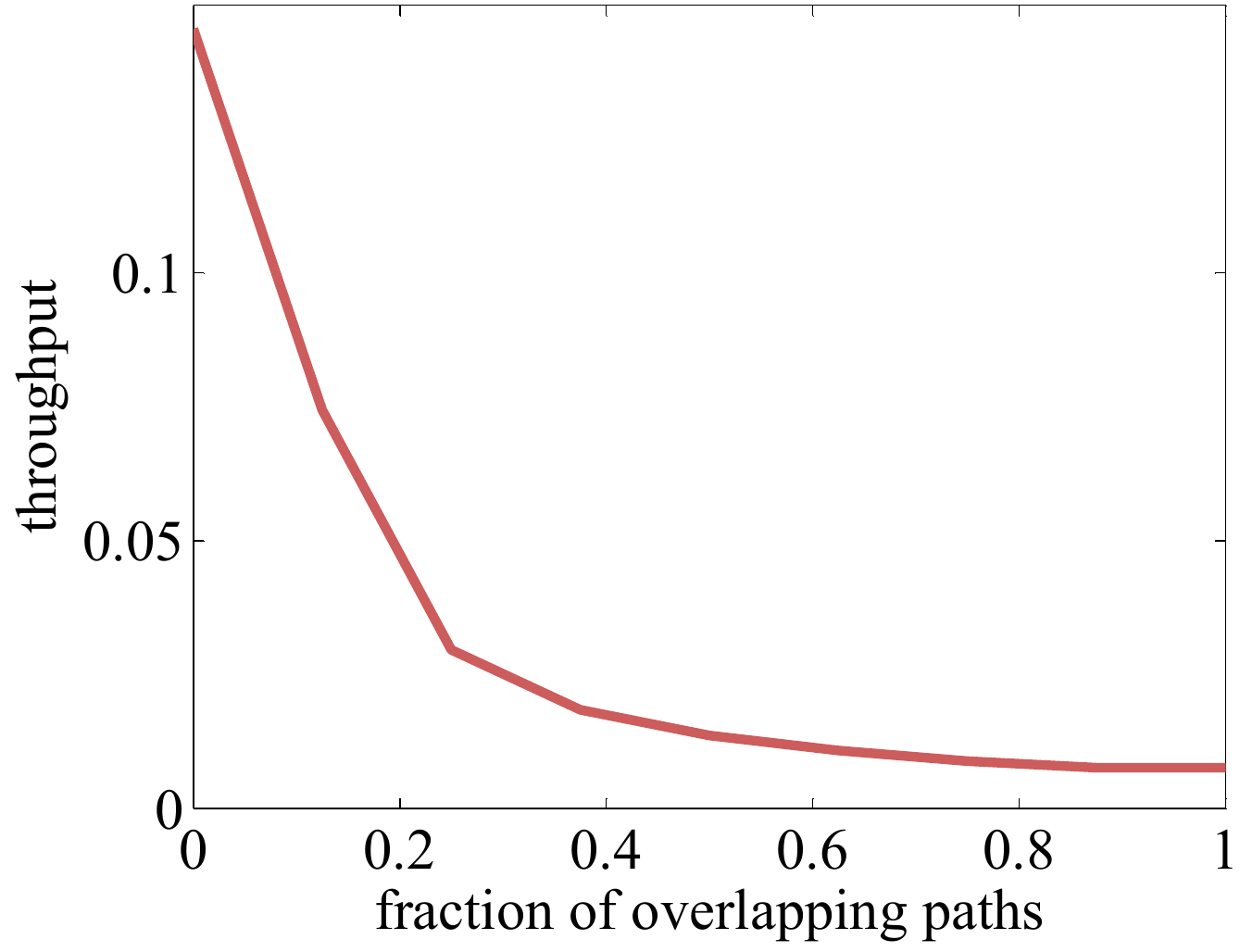}
		\caption{Throughput versus fraction of path overlap for two colors on a $1 \times 16$ unit square tessellation.}
		\figlabel{throughputPathOverlap}
	\end{minipage}
	\hspace{0.01\linewidth}
	\begin{minipage}[t]{0.475\linewidth}
		\centering
		\includegraphics[width=\columnwidth]{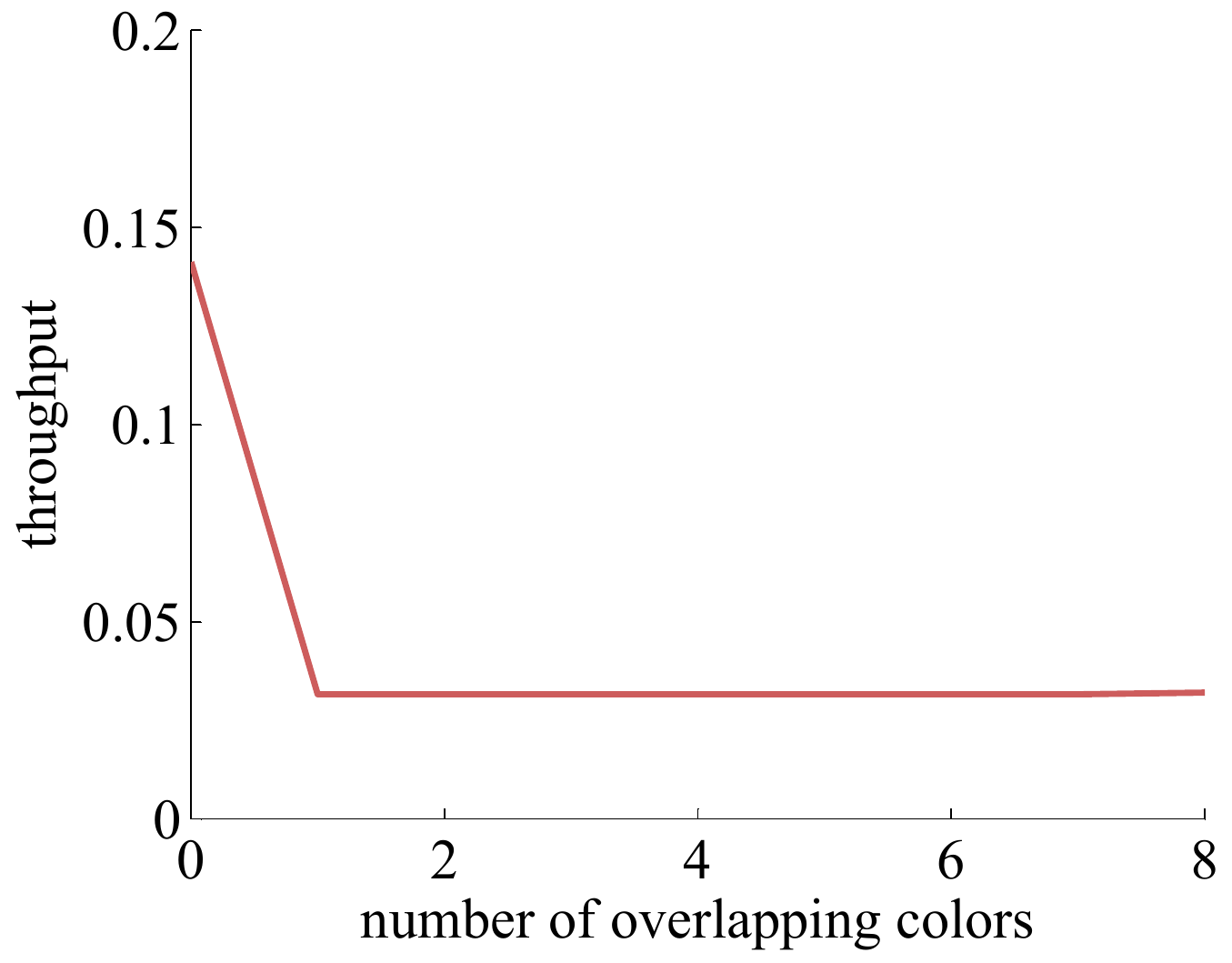}
		\caption{Throughput versus number of overlapping colors on a $1 \times 3$ unit square tessellation.}
		\figlabel{throughputColorOverlap}
	\end{minipage}
\end{figure}

%Remarks:
%The statement such as the probability that there exists a path from source to target approaches 1 as k goes to infinity is interesting.  This should be enough to ensure that there exists a path from source to target for k' rounds, where k' is the length of the path.
%
%Let a $p_f$-\textit{fault execution} be an execution along which nodes may fail with probability $p_f$.

%Let a $p_r$-\textit{recovery execution} be an execution along which faulty nodes become non-faulty with probability $p_r$.
%
\section{Related Work}
\seclabel{rw}
There is a large amount of work on traffic control in transportation systems (see, e.g.,~\cite{nolan1994book,ioannou1997book}) and robotics (see, e.g.,~\cite{bullo2009book}).
We briefly summarize some of the more related work, but highlight that we are presenting a formal model of an example of such systems.
%
%Cellular Automata
%
%Embedded graph grammars
Distributed air and automotive traffic control have been studied in many contexts.
Human-factors issues are considered in~\cite{leveson2001atm,prevot2002} to ensure collision avoidance between the coordination of numerous pilots and a supervisory controller modeling the semi-centralized air traffic control components.
The Small Aircraft Transportation Protocol (SATS) is semi-distributed air traffic control protocol designed for small airports without radar, so pilots and their aircraft coordinate among themselves to land after being assigned a landing sequence order by an automated system at the airport~\cite{abbott2004tr}.
SATS has been formally modeled and analyzed using a combination of model checking and automated theorem proving~\cite{munoz2006}.
SATS and this paper share an abstraction: the physical environment is a priori partitioned into a set of regions of interest, and properties about the whole system are proved using compositional analysis.
Safe conflict resolution maneuvers for distributed air traffic control are designed in~\cite{tomlin1998}.
A formal model of the traffic collision avoidance system (TCAS) is developed and analyzed for safety in~\cite{livadas1999rtss}.
TCAS is a system deployed on aircraft that alerts pilots when other aircraft are in close proximity and guides them along safe trajectories.

%Vehicular networks are being designed using various wireless communication media like existing cellular tower infrastructure (vehicle-to-roadside) or ad-hoc wireless (vehicle-to-vehicle).
%
%A control architecture for each automobile involved in adaptive cruise control maneuvers while avoiding collisions is developed in~\cite{girard2001cdc}.
%
%A distributed motion coordination method called co-fields is presented in~\cite{mamei2003isads}, and is used to reduce congestion and traffic jams in urban road networks.
%\sayan{define cofields}
%
%Cooperative collision warning communication protocols and systems relying on vehicular networks are developed in~\cite{yang2004mobi,misener2005}.
%
%Sensor networks for smart roads with potential applications like safe driving monitoring are developed in~\cite{karpiriski2006}.
%
%A distributed traffic simulation is developed in~\cite{kelly2007} and used to study various travel scenarios, but the work does not include a formal analysis of correctness and presents only a paragraph discussing collision avoidance.
%
A distributed algorithm (executed by entities, vehicles in this case) for controlling automotive intersections without any stop signs is presented in~\cite{kowshik2008hscc}.
%
%A distributed intersection control system is developed with an extensive performance analysis in~\cite{dresner2008jair}, where an intersection manager schedules automobiles through the intersection.
%%
%Mechanical failures of entities and a failure-mode analysis are considered in~\cite{dresner2008att} for automated intersections.
%%
%An architecture for data dissemination in vehicular networks with cellular and/or ad-hoc wireless is developed in~\cite{manvi2009}.
%
Some methods for ensuring liveness for automotive intersections are presented in~\cite{au2011aaai}.
A method to detect the mode of a hybrid system control model of an autonomous vehicle in intersections is developed in~\cite{verma2011}, and is used to reduce conservatism of the maximally controlled invariant set (the set of collision-free controls).
Efficient distributed intersection control algorithms are developed in~\cite{colombo2012hscc}.
There is a large amount of work on flocking~\cite{olfati2006} and platooning~\cite{varaiya1993,dolginova1997hart,swaroop1999,johnson2011jnsa}.
Only a few works consider failures in such systems, like the arbitrary failures considered in~\cite{gupta2006cdc,franceschelli2008acc}, the actuator failures considered in~\cite{johnson2011jnsa}, or in synchronization of swarm robot systems in~\cite{christensen2009}.

Distributed robot coordination on discrete abstractions like~\cite{wurman2008,gilbert2009taas,kloetzer2010tr,roozbehani2011,durham2011tr,ding2011ram,chen2012tr} can be viewed as traffic control.
For instance,~\cite{gilbert2009taas} establishes a formal connection between the continuous and the discrete parts of these protocols, and also presents a self-stabilizing algorithm with similar analysis to the analysis in this paper.
These works also decompose the continuous problem into a discrete abstraction by partitioning the environment, but all these works allow at most a single entity (robot) in each partition, while our framework allows numerous entities in each partition.
If several entities are to visit some destination in~\cite{kloetzer2010tr,ding2011ram,chen2012tr}, like our targets here, that destination is represented as the union of a set of partitions and each entity must reside in one of these partitions.
%
%Based on the size of the entity, it may be possible that only a single entity can geometrically fit on a cell in our partition, and we call such scenarios \emph{single-entity partitions}.
%
%We can thus include some of the results of~\cite{kloetzer2010tr,chen2012tr} within our framework, but mention that these works allow nonlinear dynamics within a single partition, and also automatically synthesize the distributed behavior from a given temporal logic specification.
%
%We take the alternative approach and design the routing and control mechanism manually.

The Kiva Systems robotic warehouse~\cite{wurman2008} is a robotic traffic control system on square partitions, and can be described in our framework by allowing a single entity per cell.
In these warehouse systems, there is a central coordinator scheduling tasks, but the robots are responsible for path planning using an A$^*$-like search algorithm~\cite{wurman2008}.
However, several deadlock scenarios are identified when performing such path planning~\cite{roozbehani2011}.
The \emph{Adaptive Highways Algorithm} presented in~\cite{roozbehani2011} for scheduling entities relies on using the tentative trajectories of other robots collected by the central controller.
Deadlocks are also observed in other distributed robotics path-planning algorithms on discrete partitions in~\cite{luna2010iros}.
%
%A distributed robotics coverage algorithm is developed for partitioned environments in~\cite{durham2011tr}, where robots ensure an entire environment is searched and/or monitored.
%
Deadlock scenarios can also arise without a discrete abstraction, such as in the doorways considered in~\cite{perez2008iros}, the path formation algorithms of~\cite{nouyan2008}, or the warehouse automation system of~\cite{kamagaew2011icara}.

Lastly, we mention that most of these works on traffic control from aviation, automotive, swarm robotics, and warehouse automation applications can be modeled within the framework of spatial computing~\cite{zambonelli2005,beal2006is,bachrach2010}.
%
%\taylor{more details? maybe mention the virtual automata references again here}

\section{Discussion}
\seclabel{discussion}
In this section, we discuss some ways to generalize assumptions used in the paper and some alternative methods.
In this paper, we presented a distributed traffic control algorithm for the partitioned plane, which moves entities without collision to their destinations, in spite of failures.
While our algorithm is presented for two-dimensional partitions, an extension to some three-dimensional partitions (e.g., cubes and tetrahedra) follows in an obvious way.
An extension to the more general case where there are multiple sources and multiple targets of each color---and entities of each color move toward the nearest target of that color---is straightforward, but complicates notation.

\paragraph*{Self-Stabilizing Mutual Exclusion and Distributed Snapshot Algorithms}
There are a variety of mutual exclusion algorithms that could be used to determine locks (\figref{lock},~\lnref{pathmutex}).
For this paper, we require the overall system to be stabilizing and therefore the locking algorithm itself should be stabilizing.
To this end, any of the following algorithms could be adapted to our framework: the token circulation algorithm~\cite{colette1997da}, mutual exclusion~\cite{datta2000dc}, group mutual exclusion~\cite{beauquier2002icpads}, snap-stabilizing propagation of information with feedback (PIF) algorithm~\cite{bui2007dc}, or $k$-out-of-$l$ mutual exclusion~\cite{datta2009ipdps}.
A self-stabilizing distributed snapshot algorithm (see~\cite[Ch. 5]{dolev2000book}) can be used to determine if all $c$ color-shared cells are empty, after having had some entity of color $c$ (\figref{lock},~\lnref{pathempty}).
If all cells are empty, then another round of mutual exclusion commences, excluding color $c$ from the input set.

\paragraph*{General Triangulations and Affine Dynamics}
We assumed in~\secref{model} that the partitions satisfy several geometric assumptions for feasibility of entity transfers.
%, specifically that (a) no triangle in the triangulation has any obtuse angles, and (b) an appropriately defined dilation of each triangle in the triangulation must still form a triangulation.
%
We considered using vector fields generated by a discrete abstraction like those presented in~\cite{belta2005tr,belta2006tac,habets2006tac,kloetzer2008tac}.
The affine vector fields generated on simplices in~\cite{belta2005tr,kloetzer2008tac} can be used to move an entity (with potentially nonholonomic or nonlinear dynamics) through any side of a cell in a triangulation (simplex)~\cite{belta2005tr,habets2006tac} or rectangle~\cite{belta2006tac}.
However, it turns out that it is impossible to maintain our notion of safety for such vector fields without additional collision avoidance mechanisms implemented on each entity.
This is due to a simple geometric observation---moving entities through a shorter side than the side they entered through may require the entities to come closer together.
%
%\begin{lemma}
%
For example, if a cell in the triangulation has an obtuse angle, then the vector field generated by~\cite{belta2005tr} flowing from the longest edge to the shortest edge has negative divergence.
Furthermore, a vector field having negative divergence implies the flow corresponding to any two distinct points starting in that field come closer together, hence safety cannot be maintained.
%
%\end{lemma}
%
The distributed problems using these discrete abstractions~\cite{kloetzer2010tr,ding2011ram,chen2012tr} avoid this by requiring at most one entity in any (triangular) partition at a time.

We also mention a simple condition to ensure that triangulations have the required geometric partition properties (\assreftwo{projectionProperty}{transferFeasibility}).
If all the triangles in the triangulation are non-obtuse, then the triangulation satisfies these assumptions.
We also note that restricting allowable triangulations of an environment to ones without obtuse angles is not restrictive, since any polygon can be efficiently partitioned into a triangulation with non-obtuse~\cite{baker1988,bern1995} or acute~\cite{maehara2002} angles.

\paragraph*{Insufficiency of Disjoint Paths}
Finding disjoint paths, such as by using the algorithms from~\cite{mohanty1986,ogier1993,lee2001icc,marina2001icnp}, could be another approach to solving the multi-color problem, but the locking mechanism used here solves a more general problem.
Even without failures, there are many environments and choices of sources and targets for which there are no disjoint paths between sources and targets.
One such environment is shown in~\figref{exampleSystemSnubSquare}, where for two distinct colors $c$ and $d$, the paths between the respective sources and targets necessarily overlap, so an algorithm for finding disjoint paths cannot be used as there are no disjoint paths between sources and targets.
However, there are disjoint paths in some cases, so no scheduling would be necessary if these are found, but our routing algorithm does not necessarily find these, as the disjoint paths may not be shortest distance.
A self-stabilizing algorithm for finding disjoint paths on planar graphs would be an enhancement to our algorithm, as it would increase throughput in the case that paths need not overlap.

\begin{figure}[t]
	\centering%
	\begin{minipage}[t]{0.525\linewidth}
		\vspace{0pt} % alignment hack
		\centering
		\includegraphics[width=\columnwidth]{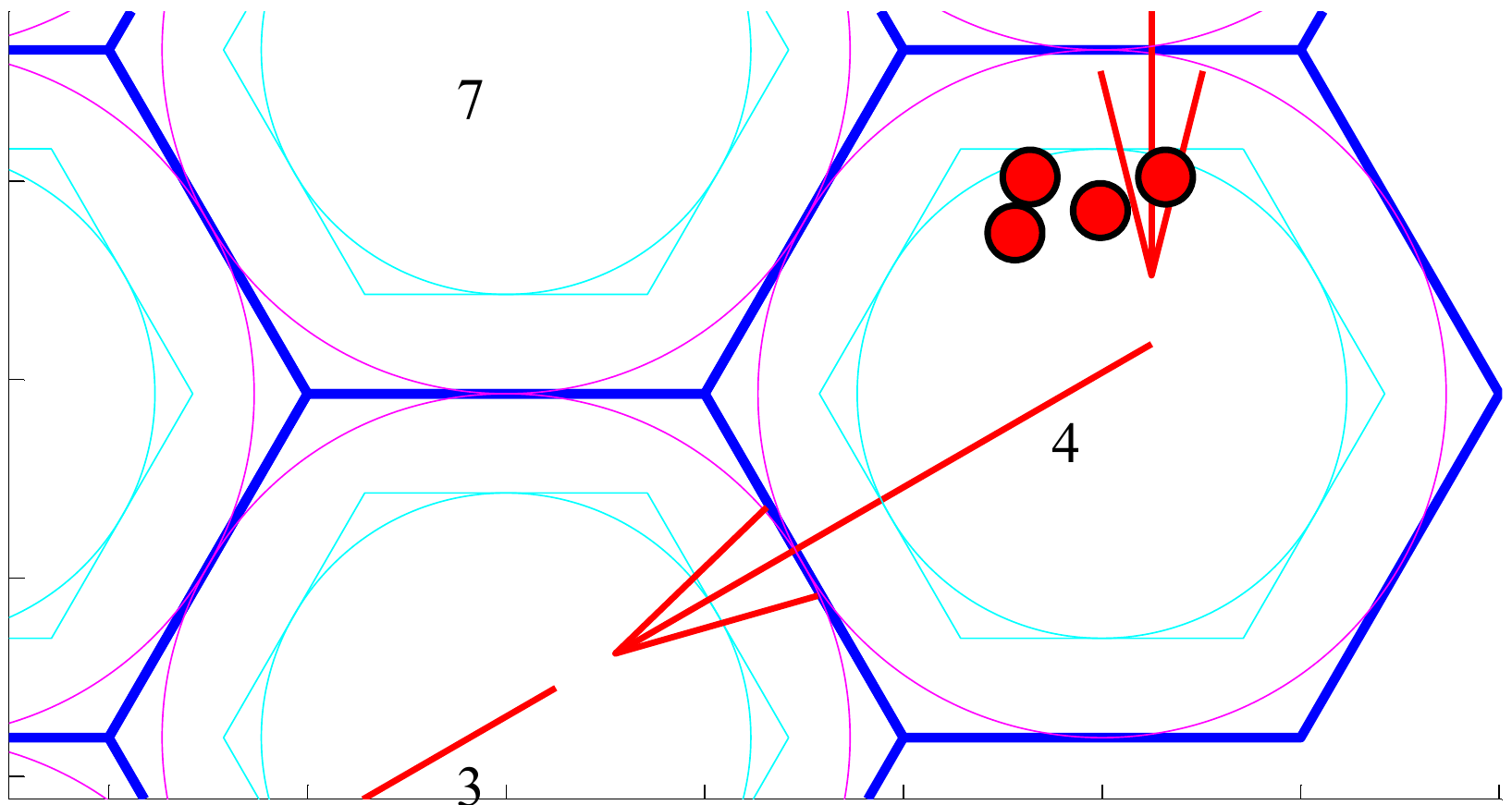}
		\caption{Hexagonal partition that does not satisfy the projection property (\assref{projectionProperty}).  An extension to allow such partitions would require enlarging the transfer region and receiving a signal from all of the potential next neighbors, which would require cells $3$ and $7$ both to signal cell $4$ to move.}
		\figlabel{exampleSystemHexagon}
	\end{minipage}
	\hspace{0.01\linewidth}
	\begin{minipage}[t]{0.425\linewidth}
		\vspace{0pt} % alignment hack
		\centering
		\includegraphics[width=\columnwidth]{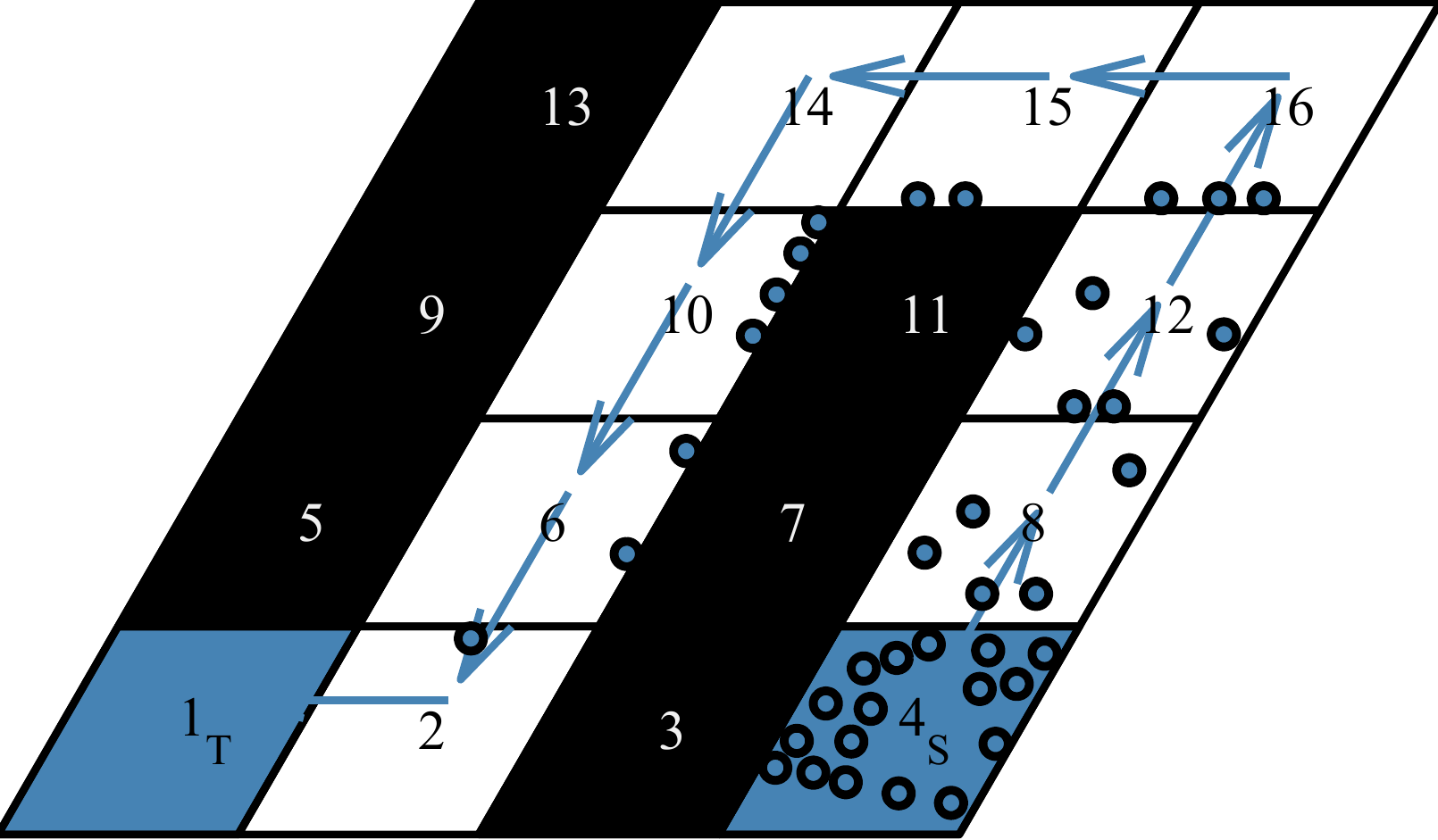}
		\caption{Example system on a parallelogram partition with failed cells in black.  The several turns along the path from the source to the target cause a saturation of entities on cells $6$, $10$, and $14$.  The movement vector $\MoveVectorij$ is defined as the unit vector parallel to the $x$ axis for movement between horizontal neighbors, and the unit vector parallel to the vertical sides of the parallelograms between vertical neighbors.}
		\figlabel{exampleSystemParallelogram}
	\end{minipage}
\end{figure}

%~\cite{mohanty1986} disjoint paths
%~\cite{oglier1993} disjoint paths
%~\cite{lee2001icc,marina2001icnp} disjoint paths
%~\cite{awerbuch1993focs} multi-commodity flow problem
%

%The following lemma states that for any environment $E$, if there is more than a single color, then there exists a color shared route which must overlap.
%%
%This is due to having no assumption on which cells are targets and sources.
%%
%Hence, using a disjoint path algorithm for routing, cannot solve our problem due to the geometric constraints.
%
%%
%\begin{lemma}
%%
%For all $\abs{C} \geq 2$, for any environment $E$, there exist $\TID, \SID \subseteq \ID$ such that the path graphs from $s_c$ to $t_c$ and $s_d$ to $t_d$ overlap.
%%
%\end{lemma}
%%
%\begin{proof}
%%
%Since $\abs{C} \geq 2$, take distinct $c, d \in C$, and suppose $\tid_c \in \SID^d$ and $\tid_d \in \SID^c$, then we immediately have $G_c \cap G_d \neq \emptyset$.
%%
%\end{proof}

%
\paragraph*{Back-Pressure and Wormhole Routing}
%%MSMT scheduling algorithms
%1) Move away from color with the token by setting the distance from all cells in any path to the target of the token to infinity: problem: what if this token path disconnects cells of other colors from their targets (frequently it will).  Now all paths may count to infinity (no longer target connected), which will be safe, but must ensure they do not deadlock.
%2) backpressure routing
%3) wormhole routing
%
Back-pressure routing~\cite{tassiulas1992,awerbuch1993focs} is an algorithm for dynamically routing traffic over an underlying graph using congestion gradients.
If we view the color of each entity as its intended address and consider this problem from the perspective of queuing theory, one might think back-pressure routing could provide a throughput-optimal solution for the problem.
%
%However, under such a formulation, our model has finite size queues because the number of entities in any cell is bounded above by the area of that cell divided by the sum of the areas of the entities that can safely be positioned on that cell.
%
%Additionally, our model does not allow entities with different destinations to reside on a cell (\invref{onecolor}).
%
However, our physical motion model is incompatible with back-pressure routing.
For a given cell, our model does not allow arbitrary choice of the next neighbor for each entity on that cell.
In particular, when one cell moves its entities toward a neighboring cell, all entities sufficiently near the shared side between the two neighbors would transfer.

Wormhole routing~\cite{ni1993} is a flow control policy over a fixed underlying graph for determining when packets move to the node on the graph.
Addresses in wormhole routing are very short and come at the beginning of a packet, so a packet can be subdivided into pieces or \emph{flits} and begin being forwarded after the address is received, yielding a snake-like sequence of flits in transfer.
One could also view the sequence of entities on a path toward the appropriately-colored target (see~\figref{exampleSystemParallelogram}) sequence of flits flowing to a destination in wormhole routing.
While similar deadlock scenarios can arise in our system and wormhole routing, wormhole routing is incompatible with our system due to the motion model just like back-pressure routing.

%\paragraph*{Computing Partitions of Polygons}
%%
%We note that there are efficient algorithms for triangulating a polygonal environment without holes~\cite{chazelle1991} (simple polygons, $O(n)$ where $n$ is the number of vertices) or with holes~\cite{barYehuda1994} (simply connected polygons, $O(n + h (\log h)^{1 + \epsilon})$, where $n$ is the number of vertices, $h$ is the number of holes, and $\epsilon$ is any positive constant).

%%
%
%\paragraph*{Extension to Partitions without the Single-Side Projection Property}
%%
%Most of the system can easily be extended to the case where partitions do not have the projection property.
%%
%For instance, consider the hexagonal tessellation shown in~\figref{exampleSystemHexagon}.
%%
%The safety proof for these partition is straightforward---so long as all the cells where entities could transfer have a safe spacing, then the entity transfer will be safe.
%%
%Similarly, the single-target liveness proof is straightforward.
%

%
\section{Conclusion}
\seclabel{conclusion}
We presented a self-stabilizing distributed traffic control protocol for the partitioned plane, where each partition controls the motion of all entities within that partition.
The algorithm guarantees separation between entities in the face of crash failures of the software controlling a partition.
Once new failures cease occurring, it guarantees progress of all entities that are neither isolated by (a) failed partitions, nor (b) cells with entities of other colors that become deadlocked due to failures, to the respective targets. 
Through simulations, we presented estimates of throughput as a function of velocity, minimum separation, single-target path complexity, failure-recovery rates, and multi-target path complexity.

%
%For practical applications, we need algorithms that tolerate a relaxed coupling between entities and allow them some degree of independent movement while preserving safety and progress.
%
%The goal would then be to design efficient control algorithms for this relaxed setting under different assumptions about the behavior of the free entities.
%
It would be interesting to develop strategies allowing entities of different colors on a single cell.
Our strategy of preventing entities of different colors from residing on a single cell simplified some analysis, but it also complicated some analysis, by making it harder to prove progress because deadlock scenarios may frequently arise.
It would be interesting to develop algorithms allowing mixing and sorting of colors using different types of motion coupling.
It would also be interesting to design algorithms that can allow relaxing the assumption on what failures may occur to ensure liveness.
We believe this would require a more complex routing algorithm to temporarily move entities of some colors off the color shared cells, thus allowing some other color on the color shared cells to make progress.
%

% Acknowledgments
\section{Acknowledgments}
%\begin{acks}
%
The authors thank Zhongdong Zhu for helping develop the current version of the simulator, Karthik Manamcheri for helping develop an earlier version of the simulator, and Nitin Vaidya for helpful feedback.
We also thank the anonymous reviewers who helped improve the earlier version of this paper.
%
%\end{acks}



%% References


\seclabel{references}
\bibliographystyle{elsarticle-num}
\bibliography{IEEEabrv,master}






% \end{thebibliography}

\end{document}